\documentclass{article}

    \PassOptionsToPackage{numbers, compress}{natbib}

     \usepackage[preprint]{neurips_2025}

\usepackage[utf8]{inputenc} 
\usepackage[T1]{fontenc}    
\usepackage{hyperref}       
\usepackage{url}            
\usepackage{booktabs}       
\usepackage{amsfonts}       
\usepackage{nicefrac}       
\usepackage{microtype}      
\usepackage{xcolor}         

\usepackage{comment}
\usepackage{amsthm}

\usepackage{amsmath}
\usepackage{graphicx}
\usepackage{amssymb}

\newtheorem{assumption}{Assumption}
\newtheorem{lemma}{Lemma}

\newtheorem{corollary}{Corollary}
\newtheorem{definition}{Definition}

\newtheorem{theorem}{Theorem}

\title{Risk Bounds For Distributional Regression}

\author{%
  Carlos Misael Madrid Padilla \\
Department of Statistics and Data Science\\
  Washington University in St Louis\\
  St Louis, MO 63130 \\
  \texttt{carlosmisael@wulst.edu} \\
   \And
   Oscar Hernan Madrid Padilla \\
   Department of Statistics\\
   University of California, Los Angeles \\
   Los Angele, CA 90095\\
   \texttt{oscar.madrid@stat.ucla.edu} \\
   \AND
   Sabyasachi Chatterjee\\
   Department of Statistics \\
   University of Illinois at Urbana-Champaign\\
   Champaign, IL 61820\\
   \texttt{sc1706@illinois.edu} \\
}

\begin{document}

\maketitle

\begin{abstract}
  This work examines risk bounds for nonparametric distributional regression estimators. For convex-constrained distributional regression, general upper bounds are established for the continuous ranked probability score (CRPS) and the worst-case mean squared error (MSE) across the domain. These theoretical results are applied to isotonic and trend filtering distributional regression, yielding convergence rates consistent with those for mean estimation. Furthermore, a general upper bound is derived for distributional regression under non-convex constraints, with a specific application to neural network-based estimators. Comprehensive experiments on both simulated and real data validate the theoretical contributions, demonstrating their practical effectiveness.
\end{abstract}

\section{Introduction}
\label{sec:intro}
While regression methods are widely popular across statistics and machine learning, it is well recognized that the conditional mean alone often fails to capture the full relationship between a response variable and a set of covariates. As noted by \citet{shaked2007stochastic}, ``the ultimate goal of regression analysis is to obtain information about the conditional distribution of a response given a set of explanatory variables.'' A common framework is quantile regression, which estimates conditional quantiles to provide a more detailed view of the response distribution \citep{koenker1978regression}. A more direct approach is distributional regression, which estimates the conditional distribution of the response given the covariates.


Distributional regression has found applications in diverse areas, including electricity spot price analysis \citep{klein2024distributional}, understanding income determinants \citep{kneib2023rage}, modeling weather data \citep{umlauf2018primer}, and improving precipitation forecasts \citep{henzi2021isotonic,schlosser2019distributional}.

The estimation of distributions of random variables under structural constraints is a fundamental problem in many statistical and machine learning tasks, including nonparametric regression, density estimation, and probabilistic forecasting \citep{hastie2009elements, tibshirani2014adaptive, guntuboyina2020adaptive}. Consider the sequence distributional model, in which we observe independent random variables \( y_1, \ldots, y_n \in \mathbb{R} \), each drawn from an unknown distribution \( F_1^*, \ldots, F_n^* \), respectively. The objective of this paper is to estimate the vector \( F^*(t) = (F_1^*(t), \ldots, F_n^*(t))^\top \) for $t\in \mathbb{R}$. Here, \( F_i^*(t) = \mathbb{P}(y_i \leq t) \) represents the cumulative distribution function (CDF) at a specified $t$ for each observation $y_i$. 

In this paper, we will explore different structural constraints to estimate $F^*(t)$. These constraints not only ensure interpretable and robust estimators but also prevent overfitting, making them essential in domains such as signal processing, medical diagnostics, and probabilistic weather forecasting. For example, in survival analysis, monotonicity reflects the cumulative nature of survival probabilities, while in genomics, smoothness helps capture gradual trends in gene expression data \cite{barlow1972statistical,tibshirani2014adaptive}. 

To rigorously evaluate the quality of the estimators, we employ the continuous ranked probability Score (CRPS), a widely used metric for assessing the accuracy of probabilistic forecasts \cite{gneiting2007strictly}. By quantifying the distance between the estimated and true CDFs, CRPS provides an interpretable and robust framework for comparing estimators under various structural constraints.

\subsection{Summary of Results}

We now provide a brief summary of the contributions in this paper.

\textbf{Unified Framework for Estimation:} We study a unified framework for estimating $F^*(t)$ by minimizing a quadratic loss over a convex set $K_t \subset \mathbb{R}$.  The convex set 
$K_t \subset \mathbb{R}$ enforces structural constraints on the parameter $F^*(t)$. For the resulting estimator, we provide rigorous theoretical guarantees, including non-asymptotic bounds on both the mean squared error (MSE) and the CRPS.

\textbf{Applications to Monotonicity and Bounded Variation:}  We demonstrate the applicability of our framework to convex constraints arising in monotonicity \cite{brunk1969estimation, barlow1972isotonic, chatterjee2015risk} and bounded total variation \cite{mammen1997locally, tibshirani2014adaptive, guntuboyina2020adaptive}. These examples illustrate the flexibility and practical utility of the proposed approach for structured parameter estimation.

\textbf{General Theory for Distributional Regression:} We establish a general theory for distributional regression with constraints encoded by arbitrary sets. The main result provides a uniform bound on the empirical  $\ell_2$ error of $F^*(t)$ across $t$, with the upper bound explicitly dependent on the approximation error and the complexity of the sets $K_t$.

\textbf{Convergence Rates for Neural Networks:} Exploiting the general results for arbitrary sets, we derive convergence rates for distributional regression using dense neural networks. This extends the framework of \citet{kohler2021rate} to the context of distributional regression.

\subsection{Other Related Work}

Distributional regression involves modeling the cumulative distribution function (CDF) of a random variable, whereas quantile regression focuses on estimating the inverse CDF. \citet{koenker2013distributional} provides a comprehensive review of the distinctions between these two approaches. For instance, if the outcome variable is income and the covariate is a binary variable representing educational attainment, distributional regression would model the probability that income falls below a certain threshold for each educational group. In contrast, quantile regression would estimate income differences between individuals ranked at the same quantile within the two groups. For further discussion on these differences, see also \citet{peracchi2002estimating}.
The approach we adopt in this paper is based on modeling the mean of the random variables  of the form $1_{\{y_i\leq t\}}$, which represent the indicator of the events $\{y_i\leq t\}$ for $i=1,\ldots,n$. This idea was first introduced in \cite{foresi1995conditional} and has since been explored in various contexts, including \citet{firpo2023decomposition}, \citet{rothe2012partial}, \citet{rothe2013misspecification}, and \citet{chernozhukov2013inference}.
More recently, distributional regression has been studied in diverse settings, such as isotonic regression \cite{henzi2021isotonic}, random forests \cite{cevid2022distributional}, and neural networks \cite{shen2024engression, imani2018improving}. In particular, \citet{henzi2023consistent} develop consistent estimators of conditional CDFs under increasing concave and convex stochastic orders, providing a flexible framework that accommodates heterogeneous variance structures and distributional crossings—situations where traditional stochastic dominance assumptions may fail. These works highlight the versatility and applicability of distributional regression across a range of methodologies and problem domains.
Finally, other nonparametric approaches to distributional regression include \citet{dunson2007bayesian} and \citet{hall1999methods}.

\section{Notation}
\label{CRPS-equation-def}
 Througout, for a vector $v \in \mathbb{R}^n$, we denote by $\|v\|$ and $\|v\|_1$ the $\ell_2$ and $\ell_1$ norms, respectively. Thus, $\|v\| = \sqrt{\sum_{i=1}^n  v_i^2}$ and $\|v\|_1 \,=\, \sum_{i=1}^n \vert v_i\vert$. Furthermore, given two  functions  $G,H \,:\,\mathbb{R} \rightarrow \mathbb{R}$, we define 
\(
    \mathrm{CRPS}(G,H) = \int_{-\infty}^{\infty} (G(t) - H(t))^2 dt.
\)
Also, for  for  $\eta  >0$ and   $v \in \mathbb{R}^n$, we write
\(
B_{\eta}(v) \,:=\, \left\{    u\in  \mathbb{R}^n \,:\,\|v-u\|\leq \eta  \right\}.
\)
For a metric space $(\mathcal{X},d)$, let $K$ be a subset of $\mathcal{X}$, and $r>0$ be a positive number. Let $B_r(x,d)$ be the ball of radius $r$ with center $x\in \mathcal{X}$. We say that a subset $C \subset \mathcal{X}$ is an $r$-external covering of $K$ if 
\(
K\,\subset\, \cup_{x \in C} B_r(x,d).
\)
Then  external covering number of $K$, written as $N(r,K,d)$, is defined as the minimum cardinality of any $r$-external covering of $K$. Furthermore, for a function $f \,:\, \mathcal{X}\rightarrow \mathbb{R}$, we define its $\ell_{\infty}$ norm as $\|f\|_{\infty} \,:=\, \sup_{x \in \mathcal{X}} \vert f(x) \vert$. Also, for two sequences $a_n$ and $b_n$, we write $a_n \lesssim b_n$ if $a_n \leq c b_n$ for a positive constant $c$, and if $a_n \lesssim b_n$ and $b_n\lesssim a_n$, then we write $a_n \asymp b_n$. The indicator function of a set $A$ is denoted as $1_{A}(t)$, which takes value $1$ if $t\in A$ and $0$ otherwise. Finally, the indicator of an event $A$ is $1_A$ which takes value $1$ if $A$ holds and $0$ otherwise.

\subsection{Outline}

The paper is organized as follows: Section \ref{Sec-General-Const} establishes a unified framework for distributional regression under structural constraints. In particular, Section~\ref{sec:th_constrained_estimators} introduces the general methodology. Section~\ref{SectionConvexRisk} derives statistical risk guarantees for the convex case. 
Sections \ref{iso-reg-sec} and \ref{TredFil-sec} provide concrete examples of convex estimators, focusing on isotonic regression and trend filtering, respectively. Section \ref{Nonconvex-sec} extends the framework to the non-convex setting. Specifically, Section \ref{Nonconvex-sec-GR} develops the general theory for non-convex estimators, and Section \ref{DRN-section} explores an estimator based on deep neural networks. Section \ref{simu-data} and \ref{RealD-sec} present simulation studies and real-data experiments, respectively, comparing the proposed methods to state-of-the-art competitors. 
Finally, Section \ref{sec:conclusion} concludes with a discussion of future research directions, and the Appendix provides additional theoretical results and experimental findings, including two additional real-data applications.

\section{Theory}
\label{Sec-General-Const}

\subsection{General Result for Constrained Estimators}
\label{sec:th_constrained_estimators}
We begin this section by addressing general problems in distributional regression under structural constraints. Using the notation introduced in Section~\ref{sec:intro}, our goal is to estimate the vector of distribution functions evaluations \( F^*(t) \). To achieve this, we adopt an empirical risk minimization framework based on the continuous ranked probability score (CRPS), a widely used tool for evaluating distributional forecasts \cite{matheson1976scoring,gneiting2007strictly}. Specifically, we consider the estimator
\begin{equation}
\label{eq:crps_estimator}
\widehat{F} := \underset{\{F_i\}_{i=1}^n \,:\, F(t) \in K \text{ for all } t}{\arg\min} \sum_{i=1}^{n} \mathrm{CRPS}(F_i,  1_{\{y_i \leq  \cdot\}} ),
\end{equation}
where \( F(t) = (F_1(t), \ldots, F_n(t))^\top \in \mathbb{R}^n \), and \( K \subset \mathbb{R}^n \) is a  set encoding the structural constraint. Related CRPS-based formulations have been proposed in the context of isotonic distributional regression \cite{henzi2021isotonic}.

This formulation can be viewed as a form of M-estimation over function-valued parameters, where the loss is defined via a proper scoring rule. Specifically, the estimator \( \widehat{F} \) minimizes the empirical CRPS loss, computed relative to the observed point masses \( 1_{\{y_i \leq \cdot\}} \). Because CRPS is strictly proper \citep{gneiting2007strictly}, the population version of this loss—where expectations are taken over the distribution of each \( y_i \)—is minimized at the true distribution function \( F^*_i \). Hence, the estimator in \eqref{eq:crps_estimator} can be interpreted as an empirical risk minimizer, and is statistically and decision-theoretically justified. The constraint \( F(t) \in K \) for all \( t \) imposes additional structure on the estimator, promoting regularity and interpretability.


We now show that the solution to the empirical CRPS minimization problem can be obtained via a simple projection estimator. For each \( t \in \mathbb{R} \), define
\(
w(t) \,=\, (  1_{\{y_1\leq t\}},\ldots, 1_{\{y_n\leq t\}} )^{\top} \in \mathbb{R}^n,
\)
and consider the projection
\begin{equation}
\label{eqn:estimatoraux}
\widehat{F}(t) := \underset{\theta \in K}{\arg\min} \left\{ \| w(t) - \theta \|^2 \right\}.
\end{equation}

\begin{lemma}
\label{lemma-1}
For $K \subset \mathbb{R}^n$, the function-valued estimator defined by \eqref{eqn:estimatoraux}, with \( \widehat{F}_i(t) := [\widehat{F}(t)]_i \), solves the empirical CRPS minimization problem in \eqref{eq:crps_estimator}.
\end{lemma}

Thus for each \( t \in \mathbb{R} \), the solution \( \widehat{F}(t) \) in problem \eqref{eq:crps_estimator}  is obtained by projecting the empirical vector \( w(t) \) onto the convex set \( K \), making the estimation process both computationally efficient and conceptually transparent. In the more general case where the structural constraint may vary with \( t \), i.e., \( K_t \subset \mathbb{R}^n \), we extend the estimator to
\begin{equation}
\label{eqn:estimator}
\widehat{F}(t) := \underset{\theta \in K_t}{\arg\min} \left\{ \| w(t) - \theta \|^2 \right\}.
\end{equation}
This projection-based formulation allows us to define flexible estimators tailored to time-varying structural assumptions.

\subsection{Risk bounds for convex case}
\label{SectionConvexRisk}

In this subsection, we focus on the special case where each constraint set \( K_t \subset \mathbb{R}^n \) is convex and satisfies \( F^*(t) \in K_t \). As described in \eqref{eqn:estimator}, the corresponding estimator \( \widehat{F}(t) \) is obtained by projecting the empirical vector \( w(t) \) onto \( K_t \), which amounts to minimizing the empirical \( \ell_2 \)-loss under convex constraints.

Our next result provides an upper bound on  the expected value of CRPS error. This is a consequence of a modified version of  Theorem A.1 in \citet{guntuboyina2020adaptive}, see Theorem \ref{thm1} in the Appendix.

\begin{theorem}
	\label{thm2}
    Suppose that $(0,\ldots,0),(1,\ldots,1) \in K_t$,  and 
    \begin{equation}
        \label{eqn:support}
        \mathbb{P}(y_i \in \Omega)=1,   \,\,\,\text{for all }\,\, i, 
    \end{equation}
    and some fixed compact set $\Omega\subset \mathbb{R}$. Then,  
    there exists a  constant  $C >0$ such that
	\begin{equation}
	    \label{eqn:upper0}
           \begin{array}{l}
       \displaystyle 	\mathbb{E}\left(  \frac{1}{n} \sum_{i=1}^{n} \mathrm{CRPS}( \widehat{F}_i,F_i^*  )   \right)	      \,\leq \, \frac{C  \eta^2}{n} 
    \end{array}
	\end{equation}
    for every  $\eta>1$ satisfying 
	\begin{equation}
	    \label{eqn:lc_gaussian}
     \underset{t \in \mathbb{R}}{\sup}\,    \mathbb{E}\left[  \underset{ \theta \in K_t   \,:\,  \| \theta  - F^*(t)\|\leq  \eta }{\sup}   \, g^{\top} (\theta - F^*(t))     \right]\leq  \frac{\eta^2}{L}
	\end{equation}
	where  $g \sim  N(0,  I_n)$, for a universal constant $L>0$.
\end{theorem}
Thus, Theorem \ref{thm2} shows that bounding the expected CRPS error can be reduced to analyzing the local Gaussian complexity, as given on the left-hand side of Equation (\ref{eqn:lc_gaussian}).

As a direct consequence of Theorem \ref{thm2}, we can derive the same upper bound as in (\ref{eqn:upper0}) for a  rearrangement (see e.g. \citet{lorentz1953inequality,bennett1988interpolation})
of a truncated version of  $\widehat{F}_i$,  which is non-decreasing by construction. This result is formalized in the following theorem.


\begin{corollary}
	\label{cor1}
          Suppose that $(0,\ldots,0),(1,\ldots,1) \in K_t$ for all $t \in \mathbb{R}$, and that $F^*_i(\cdot)$ is continuous for all $i$.  Given $i \in \{1,\ldots,n\}$, let $\widehat{F}_i^{+}(t) \,=\, \max\{0,  \widehat{F}_i(t)  \} $.     Let 
$y_{(1) } \leq\ldots \leq y_{(n)} $ be  the order statistics of $y$. Define  $a_{i,j} = \widehat{F}_i^{+}(y_{ (j)  }) $ for $j=1,\ldots,n-1$  and   sort the vector $a_{i,\cdot}$ as
\(
a_{i,j_{1}  } \geq  \ldots \geq  a_{i,j_{n-1}},
\)
and let
$\widetilde{F}_i$ be defined  as 
{\small{\(
\widetilde{F}_i(t )  \,=\, \begin{cases}
	0 &\text{if}\,\,\,    t <   y_{(1)}\\  
\displaystyle \sum_{l=1}^{n-1} a_{i,j_l, }1_{  [  v_{l}, v_{l-1}) }(t)    &\text{if}\,\,   y_{(1)}   \leq t < y_{(n)}\\
				1  &\text{if}\,\,\,     y_{(n)} \leq t, \\  
\end{cases}
 \)}}
 where $v_0 = y_{(n)}$ and 
 \(
v_l \,=\,  y_{(n)}  -  \sum_{k=1}^l (y_{ ( j_k+1)  }\,-\,y_{ (j_k)  } )
 \)
 for $l=1,\ldots,n-1$. Then, with the notation from Theorem \ref{thm2}, we have that 
	\(
\begin{array}{l}
\displaystyle     \mathbb{E}\left(  \frac{1}{n} \sum_{i=1}^{n} \mathrm{CRPS}( \widetilde{F}_i,F_i^*  )   \right)  \,\leq \, \frac{C \eta^2 }{n}.
\end{array}
\)
\end{corollary}


The function \( \widetilde{F}_i(t) \) defined above is non-decreasing and is constructed by modifying the original estimator 
\( \widehat{F}_i(t) \) through a rearrangement of its values. Specifically, this is achieved by applying a change of variable to 
\( \widehat{F}_i(t) \), followed by the Hardy–Littlewood decreasing rearrangement. This classical construction \citep{hardy1928some} preserves the level set measures and minimizes the 
\( L^2 \) distance among all decreasing equimeasurable functions. An example of \( \widetilde{F}_i(t) \)  and \( \widehat{F}_i(t) \) is given in Figure \ref{fig10} in the Appendix. 


We now present the final result from this section.
\begin{theorem}
	\label{thm3}
	Suppose that $K_t \subset K$ for all $t$. 	For any $\eta >0$ it holds that 
{\small{
    \begin{equation}
            \arraycolsep=1.4pt\def\arraystretch{1.8}
	    \label{eqn:upperbound2}
        	\begin{array}{l}
\displaystyle  	 \mathbb{P}\left(   \underset{t \in \mathbb{R}}{\sup} \,    \sum_{i=1}^{n}   \left(  \widehat{F}_i(t)    - F^*_i(t)  \right)^2   > 2\eta^2   \right)      \leq  \frac{C}{\eta^2} \int_{0}^{\eta/4}    \sqrt{\log N(\varepsilon,(K-K)\cap   B_{\varepsilon}(0),\|\cdot\|   )} d\varepsilon     +  \frac{ C  \sqrt{\log n}}{\eta},
	\end{array}
	\end{equation}}}
	for a positive constant $C$.
\end{theorem}

Theorem~\ref{thm3} provides a high-probability concentration bound for the MSE in estimating \( F^*(t) \), holding uniformly over all \( t \in \mathbb{R} \). This sets it apart from standard sub-Gaussian bounds, which typically yield control at a fixed evaluation point \( t \), see for example \cite{bellec2018sharp, chatterjee2015matrix}. To ensure that the bound in Theorem~\ref{thm3} is small, it suffices to upper bound the local entropy of \( K - K \), where \( K \) is an upper set that contains the sets \( K_t \).

\subsubsection{Isotonic Regression}
\label{iso-reg-sec}

In this subsection, we present the first application of our general theory from Section \ref{Sec-General-Const}, focusing on distributional isotonic regression—a topic that has garnered significant attention in the literature \cite{davidov2012estimating, el2005inferences, hogg1965models, jimenez2003estimation}. The most relevant works to our results are \citet{mosching2020monotone}, which examined distributional isotonic regression under smoothness constraints, and \citet{henzi2021isotonic}, which proposed an interpolation method equivalent to the formulation in (\ref{eqn:estimator}) with $K_t$  enforcing a monotonicity constraint.

Setting 
\begin{equation}
    \label{eqn:isotonic_k}
       K_t  \,=\,  K \,:=\,  \left\{  \theta   \in \mathbb{R}^n  \,:\,   \theta_1 \leq  \theta_2  \leq  \ldots  \leq  \theta_n  \right\},
\end{equation}
     we now consider the case of isotonic distributional regression assuming that  $F^*(t) \in K_t$, which is  equivalent to  $\mathbb{P}(Y_1\leq t) \leq \ldots \leq \mathbb{P}(Y_n\leq t)  $.
	With the constraint sets as in (\ref{eqn:isotonic_k}),	the resulting estimator in (\ref{eqn:estimator}) can be found with the  pool adjacent violators algorithm from \citet{robertson1988order}.

\begin{corollary}
    \label{thm_isotonic}
    Consider the estimators  $\{  \widehat{F}_t \}_{t \in \mathbb{R} }$ defined in (\ref{eqn:estimator}) with $K_t$ defined as (\ref{eqn:isotonic_k}). If $F^*(t) \in K$ and (\ref{eqn:support}) holds for some fixed compact set $\Omega$, then 
    \begin{equation}
    		    \label{eqn:e11}
    	\begin{array}{l}
    		\displaystyle 	\mathbb{E}\left(  \frac{1}{n} \sum_{i=1}^{n} \mathrm{CRPS}( \widehat{F}_i,F_i^*  )   \right) \,\leq \,    C n^{-2/3},
    	\end{array}
    	\end{equation}
    	for some positive constant $C>0$. Moreover, (\ref{eqn:e11})  holds replacing $\widehat{F}$ with the corresponding $\widetilde{F}$ as defined in Corollary \ref{cor1}, provided that each function $F_i^*$ is continuous. Finally, 
       \begin{equation}
       	\label{eqn:e12}
       	 \underset{t \in \mathbb{R}}{\sup}   \sum_{i=1}^{n}   \frac{1}{n}\left(  \widehat{F}_i(t)    - F^*_i(t)  \right)^2  = O_{\mathbb{P}}\left(\frac{1}{n^{2/3}} + \frac{\log n}{n} \right),
       \end{equation}
       where (\ref{eqn:e12}) holds without requiring (\ref{eqn:support}) nor continuity of the $F_i^*$'s. 

\end{corollary}


The result in Corollary~\ref{thm_isotonic} establishes that distributional isotonic regression achieves an estimation rate of \( n^{-2/3} \) for both the expected average CRPS and the worst-case MSE, as shown in~(\ref{eqn:e12}). This result improves upon Theorem 3 in \citet{henzi2021isotonic} in the univariate case, which only demonstrated convergence in probability for isotonic distributional regression. However, we emphasize that \citet{henzi2021isotonic} study the more general setting of multivariate covariates, while our analysis is restricted to the univariate case (\( d = 1 \)).

To further contextualize our theoretical guarantees, we compare them with Theorem 3.3 in \citet{mosching2020monotone}, which establishes uniform consistency for estimating the conditional distribution function. Considering the fixed design case in their formulation and assuming without loss of generality that in their notation $X_1<...<X_n \in  \mathbb{R}$, the goal is to estimate an unknown family of distributions \( (F_x)_{x \in \mathbb{R}} \), where for each fixed \( t \in \mathbb{R} \), the map \( x \mapsto F_x(t) \) is assumed to be non-decreasing and \(\alpha\)-Hölder continuous with constant \( C > 0 \) that is the same across \( t \). Additionally, the design is assumed to be asymptotically dense—i.e., the covariate values \( X_1 < \cdots < X_n \) sufficiently cover the domain. Translating their setup into our notation,  their target \( F_{X_i}(t) := \mathbb{P}(Y_i \leq t \mid X_i) \) corresponds to \( F_i^*(t):= \mathbb{P}(y_i \leq t) \) in our sequence model.  Their assumption that \( x \mapsto F_x(t) \) is non-decreasing for each \( t \) aligns with our isotonic regression framework, where we impose monotonicity of the sequence \( F_1^*(t) \leq \cdots \leq F_n^*(t) \). However, in contrast to their setting, we do not require any smoothness assumptions such as Hölder continuity on the sequence \( \{F_i^*(t)\}_{i=1}^n \), nor do we require the covariates to be dense. Despite that, our method achieves a faster convergence rate of order \( n^{-2/3} \) for both the average CRPS risk and the worst-case MSE, compared to the rate \( n^{-\alpha/(2\alpha + 1)} \) (up to logarithmic factors) established under their stronger regularity assumptions.

We also show in Appendix~\ref{fast-rates-sec} that faster rates are achievable under additional structural assumptions on \( F^*(t) \). Specifically, if \( F^*(t) \) has few strict increases—e.g., if it is piecewise constant with a small number of jumps—then the estimator can attain nearly parametric risk rates up to logarithmic factors.








\subsubsection{Trend Filtering}
\label{TredFil-sec}

In this subsection, we apply the theory from Section \ref{Sec-General-Const} to distributional regression under a total variation constraint. Total variation-based methods were independently introduced by \citet{rudin1992nonlinear}, \citet{mammen1997locally}, and \citet{tibshirani2005sparsity}. These methods have been extensively studied in various contexts within the statistics literature, including univariate settings \cite{tibshirani2014adaptive, lin2017sharp, guntuboyina2020adaptive, madrid2022risk, padilla2023temporal}, grid graphs \cite{hutter2016optimal, chatterjee2021new}, and general graphs \cite{wang2016trend, padilla2018dfs}.

Before establishing our proposed total variation estimators, we introduce some additional  notation. For a vector $\theta \in  \mathbb{R}^n,$ define $D^{(0)}(\theta) = \theta, D^{(1)}(\theta) = (\theta_2 - \theta_1,\dots,\theta_n - \theta_{n - 1})^{\top}$ and $D^{(r)}(\theta)$, for $r \geq 2$, is recursively defined as $D^{(r)}(\theta) = D^{(1)}(D^{(r - 1)}(\theta))$, where $D^{(r)}(\theta) \in  \mathbb{R}^{n - r}$.  With this notation, for $r\geq 1$,  the $r$th order total variation of a vector $\theta$  is given as 
\begin{equation}
	\mathrm{TV}^{(r)}(\theta) = n^{r - 1} \|D^{(r)}(\theta)\|_{1}.
\end{equation}

The concept of the $r$th  total variation can be understood as follows. Consider $\theta$ as the evaluations of an $r$ times differentiable function $f:[0,1] \rightarrow  \mathbb{R}$  on the grid $(1/n,2/n,\dots,n/n)$. In this case, a Riemann approximation of the integral $\int_{[0,1]}\vert  f^{(r)}(t)\vert dt$ corresponds precisely to $\mathrm{TV}^{(r)}(\theta)$, where $f^{(r)}$  denotes the $r$th derivative of $f$.  Therefore, for natural instances of $\theta$, it is reasonable to expect that $\mathrm{TV}^{(r)}(\theta) = O(1).$ The above discussion motivates us to define the sets 
\begin{equation}
	\label{eqn:tv_set}
	K_t \,:=\,   \left\{  \theta   \in \mathbb{R}^n  \,:\,    	\mathrm{TV}^{(r)}(\theta) \leq V_t  \right\},
\end{equation}
for some $V_t>0$, and consider the corresponding estimator in (\ref{eqn:estimator}). The intuition here is that if $F^*(t) \in K_t$ then the probabilities $F_1^*(t),\ldots, F_n^*(t)$ change smoothly over $i$ in the sense that $F^*(t)$ has bounded $r$th total variation. The resulting set in (\ref{eqn:tv_set}) allows us to define the trend filtering distributional regression estimator subject of our next corollary which follows from the results in Section \ref{sec:th_constrained_estimators}.

Refined risk bounds under additional sparsity assumptions are presented in Appendix~\ref{fast-rates-sec}, where we show that trend filtering estimators can achieve near-parametric rates when the signal is both smooth and piecewise sparse. These results extend recent adaptive risk bounds in trend filtering; see, for example, \citet{guntuboyina2020adaptive}.

\begin{corollary}
	\label{thm5}
	Consider the estimator in (\ref{eqn:estimator}) with $K_t$ as in (\ref{eqn:tv_set}) for an integer $r$ satisfying $r\geq 1$. If  $F^*(t ) \in K_t$, (\ref{eqn:support}) holds for some fixed compact set $\Omega$, and $\sup_t V_t \leq  V$, then  
  \begin{equation}
      \label{eqn:e22}
        \begin{array}{l}
\displaystyle     \mathbb{E}\left(  \frac{1}{n} \sum_{i=1}^{n} \mathrm{CRPS}( \widehat{F}_i,F_i^*  )   \right) \,\leq \, C \left[\frac{  V^{ \frac{2}{2r+1}  }  }{ n^{ \frac{2r}{2r+1}  } } + \frac{\log n}{n}  \right]\\
\end{array}
  \end{equation}
  for a positive constant $C$. Moreover, the upper bound in (\ref{eqn:e22}) also holds for the corresponding sorted estimators $ \widetilde{F}_i$ as defined in Corollary \ref{cor1}, if in addition each function $F_i^*$ is continuous. Finally, 
         \begin{equation}
       	\label{eqn:e23}
       	 \underset{t \in \mathbb{R}}{\sup} \,    \sum_{i=1}^{n}   \frac{1}{n}\left(  \widehat{F}_i(t)    - F^*_i(t)  \right)^2  = O_{\mathbb{P}}\left(\frac{V^{ \frac{2}{2r+1}  }  }{ n^{ \frac{2r}{2r+1}  } } + \frac{\log n}{n}  \right),
       \end{equation}
       where (\ref{eqn:e23}) holds without requiring (\ref{eqn:support}), nor continuity of the $F_i^*$'s. 
\end{corollary}

Corollary \ref{thm5} establishes that the constrained version of trend filtering for distributional regression achieves the rate  $ V^{1/(2r+1)} n^{-2r/(2r+1)}$,  ignoring logarithmic factors, for both the CRPS and the worst-case MSE. This result aligns with the convergence rate of trend filtering in one-dimensional regression, where the same rate is attained when the regression function has 
$r$th-order total variation \cite{mammen1997locally, tibshirani2014adaptive, guntuboyina2020adaptive}. Additionally, per Corollary \ref{cor3} in Appendix \ref{pen-conv}, the penalized version of trend filtering for distributional regression achieves the same rate in terms of the worst-case MSE, further reinforcing its consistency with classical trend filtering results.

\subsection{Risk Bounds for the General Case}
\label{Nonconvex-sec}

\subsubsection{General Result}
\label{Nonconvex-sec-GR}

This subsection aims to present our main result on constrained distributional regression in scenarios where the constraint sets 
$K_t$ are arbitrary, not necessarily convex, and potentially misspecified for  $F^*(t)$.

\begin{theorem}
	\label{thm9}
	Let $\widehat{F}(t)$ be the estimator defined in (\ref{eqn:estimator})  for all $t\in \mathbb{R}$ but with $K_t $  not necessarily convex and with $F^*(t)$ not necessarily in $K_t$. Suppose that 
	\(
	\underset{t\in \mathbb{R} }{\sup}  \underset{F(t) \in K_t}{\sup}  \,\| F(t)\|_{\infty} \,\leq\, B
	\)
	for some constant $B\geq 1$,
	and $K_t \subset  K$ for all $t$ and some set $K$.	Let
	  $G(t)$ be defined as
	\(
	G(t) \,\in \,\underset{F(t)  \in K_t }{\arg \min}\,\|F(t)- F^*(t)\|_{\infty}.
	\) 
	Then, for $\eta>1$, with $K({\varepsilon}) = (K-K)\cap   B_{\varepsilon}(0)$, we have that 
    \[
    \begin{array}{l}
    		\mathbb{P}\left(      \underset{t\in \mathbb{R}}{\sup}\,   \|  \widehat{F}(t) - F^*(t) \|   >\eta   +   \underset{t\in \mathbb{R}}{\sup} \sqrt{n}\| F^*(t)- G(t)\|_{\infty}  \right)  \\
    \displaystyle 		 \leq   \frac{C  }{\eta^2}\sum_{j=1}^{J}    \frac{1}{2^{j-2}} \int_{0}^{  2^{j/2} \eta/4 }\sqrt{\log N(\varepsilon,K({\varepsilon}),\|\cdot\|   )} d\varepsilon +\frac{C \sqrt{\log n} }{\eta} \,+\, \frac{C\sqrt{n}}{\eta}\underset{t\in \mathbb{R}}{\sup}\,  \| G(t)-F^*(t) \|_{\infty} \\ 
    \end{array}
    \]
    for some positive constant $C>0$, and where 
    \(
      J \,=\,  \left\lceil \frac{\log(2nB/\eta^2)}{\log 2} \right\rceil.
    \)
\end{theorem}

The intuition behind Theorem~\ref{thm9} is that \( \eta \) captures the estimation error, which depends on the local covering complexity of the sets \( K(\varepsilon) \). The second term, \( \sup_{t \in \mathbb{R}} \sqrt{n} \|F^*(t) - G(t)\|_{\infty} \), corresponds to the approximation error, measuring how well the true  \( F^*(t) \) can be approximated within the model class \( K_t \). While such a decomposition into estimation and approximation error is standard in nonparametric theory, our result provides a uniform guarantee over all \( t \in \mathbb{R} \), in contrast to classical bounds that control error only at a fixed evaluation point, as in \cite{padilla2024confidence}. This distinction is particularly relevant for distributional regression problems, where the goal is to control the entire CDF path. The proof of Theorem~\ref{thm9}, provided in Appendix~\ref{sec:proof_thm9}, relies on a peeling argument and extends techniques originally developed for the convex case.

\subsubsection{Dense ReLU Networks}
\label{DRN-section}

We now turn to the application of Theorem \ref{thm9} to the problem of distributional regression using dense neural networks. The results in this section add to the literature on statistical theory for rectified linear unit (ReLU) networks as in \citet{bauer2019deep,schmidt2020nonparametric,kohler2021rate,padilla2022quantile,ma2022theoretical,zhang2024dense,padilla2024confidence}.

Before presenting our main result, we first introduce some notation. Suppose   we are given i.i.d. data $\{(x_i,y_i)\}_{i=1}^n   \subset  [0,1]^{d_0}  \times \mathbb{R}$  and let
\begin{equation}
    \label{eqn:signal}
  G^*(x,t)   \,:=\, \mathbb{P}(y_i \leq  t | x_i =x)    \,=\,   f^*_t(x)
\end{equation}
for unctions $f^*_t:  [0,1]^{d_0}\,\rightarrow\,\mathbb{R}$ for all $t$. We set the conditional cdf $F_i^*(t) = G^*(x_i,t)$.

To define the constraint set \( K_t \), we follow \citet{kohler2021rate} and assume all hidden layers have the same width. Let \( \mathcal{F}(L, \nu) \) denote the set of neural networks with depth \( L \), width \( \nu \), and ReLU activation, restricting the functions in  \( \mathcal{F}(L, \nu) \)  to satisfy \( \|f\|_\infty \leq 1 \). The structure of these networks is described in Appendix~\ref{Appendix-DRN}. Then, for all \( t \in \mathbb{R} \), we define
\begin{equation}
    \label{eqn:set_k}
    K_t = K := \left\{ \theta \in \mathbb{R}^n : \theta_i = f(x_i), \,\, i=1,\ldots,n, \text{ for some } f \in \mathcal{F}(L, \nu) \right\}.
\end{equation}


We now present our main result concerning distributional regression with  ReLU neural networks.

\begin{corollary}
    \label{thm11}
    Let $\widehat{F}(t)$ be the estimator from (\ref{eqn:estimator}) with the set  $K_t$ as in (\ref{eqn:set_k}) for all $t\in \mathbb{R}$ with $F^*(t)$ not necessarily in $K_t$. Suppose that Assumption \ref{as1}, described in Appendix~\ref{Appendix-DRN}, holds. Let
    \(
    \phi_{n} = \max_{(p, M) \in \mathcal P } n^{\frac{-2p}{ (2p+M)}}. 
\)
   Under the choices of $L$ and $\nu$ specified in equations (\ref{eqn:choice1}) or (\ref{eqn:choice2}) in Appendix~\ref{Proofofthm11}, the estimator satisfies
    \begin{equation}
       	\label{eqn:e30}
       	 \underset{t \in \mathbb{R}}{\sup} \,    \sum_{i=1}^{n}   \frac{1}{n}\left(  \widehat{F}_i(t)    - F^*_i(t)  \right)^2  = O_{\mathbb{P}}\left(  \frac{\log n}{n} \,+\, \phi_n \log^4 n   \right).
       \end{equation}
 \end{corollary}


Corollary~\ref{thm11} demonstrates that dense ReLU neural network estimators for distributional regression uniformly achieve the rate \( \phi_n \), up to logarithmic factors, in terms of the worst-case MSE for estimating the true parameters \( \{F^*(t)\}_{t \in \mathbb{R}} \), provided these parameters belong to a hierarchical composition class. Importantly, while this rate matches that of \citet{kohler2021rate} for mean regression under sub-Gaussian error assumptions, our result strengthens the guarantee by holding uniformly over all thresholds \( t \in \mathbb{R} \), rather than for a fixed target. This extension is essential for distributional learning tasks where uniform control is required, such as CRPS-based risk bounds.

\section{Simulated data analysis}\label{simu-data}

We evaluate the performance of the proposed methods against state-of-the-art approaches across diverse simulation settings that reflect various practical challenges and structural assumptions. Specifically, six distinct scenarios are considered to evaluate different aspects of the distributional regression problem.
We refer to our proposed approach as {\bf UnifDR} which adapts different methods based on the scenario. 
In the first two scenarios, {\bf UnifDR} applies the isotonic regression method from Section~\ref{iso-reg-sec}; in the next two, it uses the trend filtering approach from Section~\ref{TredFil-sec}; and in the final two scenarios, it employs the Dense ReLU Networks method described in Section~\ref{DRN-section}.

 {\bf Scenario 1 (S1)}. We generate data \(y_i \sim \text{Normal}(\mu_i, 1)\) where \(\mu_i = 1 - i/n\) for \(i = 1, \ldots, n\). 

 {\bf Scenario 2 (S2)}. We consider \(y_i\sim \text{Unif}(a_i, b_i)\), where \(a_i = (n - i)/n\) and \(b_i = a_i + 1\).

 {\bf Scenario 3 (S3)}. The true CDFs are modeled as \(F_i^*(t) = \text{Exp}(\mu_i)\), where \(\mu_i = 1 + 0.5 \sin(2\pi i/n)\). 
    
 {\bf Scenario 4 (S4)}. Consider \(F_i^*(t) = \text{Gamma}(\text{shape} = 0.7, \text{scale} = \mu_i)\), where 
    \(
    \mu_i = 6 \cdot 1_{\{i \leq n/4\}} + 2 \cdot 1_{\{n/4 < i \leq n/2\}} + 8 \cdot 1_{\{n/2 < i \leq 3n/4\}} + 4 \cdot 1_{\{i > 3n/4\}}.
    \)

 {\bf Scenario 5 (S5)}. Let \(\mathbf{x}_i \sim \text{Unif}([0, 1]^5)\). 
    The true CDFs are given by \(F_i^*(t) = \Phi((t - h(\mathbf{x}_i))/0.5)\), where 
    \(
    h(\mathbf{x}_i) = -3x_{i}^{(1)} + 2\log(1 + x_{i}^{(2)}) + x_{i}^{(3)} + 5x_{i}^{(4)} + (x_{i}^{(5)})^2,
    \) with $x_i^{(j)}$ denoting the $j$ th coordinate of $x_i$.
    The function \(\Phi\) represents the standard normal cumulative distribution function.

 {\bf Scenario 6 (S6)}. 
 Let \(\mathbf{x}_i\sim\text{Unif}([0, 1]^{10})\) and
\(
y_i \sim \chi^2\left(h\left(\mathbf{x}_i\right)\right).
\)
Here,
\(
h(\mathbf{x}_i) = \log \left(\Big\vert-0.5 \cdot \sum_{j=1}^3 \sin(\pi x_i^{(j)}) - 0.5 \sum_{j=4}^9 x_i^{(j)} + 0.5 \cos(x_i^{(10)})\Big\vert + 2\right),
\) with $x_i^{(j)}$ denoting the $j$ th coordinate of $x_i$.

To implement our proposed approach \textbf{UnifDR} we proceed as follows. The isotonic method introduced in Section~\ref{iso-reg-sec} is implemented in \texttt{R} using the pool adjacent violators algorithm (PAVA) from \citet{robertson1988order}. For this approach there are no direct competitors in the distributional regression problem. The Trend Filtering estimator in Section~\ref{TredFil-sec} is implemented using the \texttt{trendfilter} function from the \texttt{glmgen} package in \texttt{R}, and we compare it with additive smoothing splines (AddSS) via the \texttt{smooth.spline} function in \texttt{R}. For the Dense ReLU Networks method in Section~\ref{DRN-section}, we use a fully connected feedforward architecture with an input layer, two hidden layers (64 units each), and an output layer. The network is implemented in \texttt{Python} and trained using the Adam optimizer with a learning rate of 0.001. In this case, we compare the proposed \textbf{UnifDR} method with five benchmark methods. First, we consider Classification and Regression Trees (CART) \citep{breiman1984richard}, implemented in \texttt{R} via the \texttt{rpart} package, with the complexity parameter used for tuning. Second, we evaluate Multivariate Adaptive Regression Splines (MARS) \citep{friedman1991multivariate}, available in the \texttt{earth} package, where the penalty parameter serves as the tuning parameter. Third, we assess Random Forests (RF) \citep{breiman2001random}, implemented in \texttt{R} via the \texttt{randomForest} package, using 500 trees and tuning the minimum terminal node size. Additionally, we consider two recent methods. Distributional Random Forests (DRF) \citep{cevid2022distributional} is implemented via the \texttt{drf} package in \texttt{Python}. DRF employs tree-based ensemble 
with tuning parameters including the splitting rule and the number of trees. Lastly, Engression (EnG) \citep{shen2024engression} is implemented using the \texttt{engression} package in \texttt{Python}. EnG utilizes hierarchical structured neural networks, 
and we adopt the same training hyperparameters as our deep learning approach to ensure optimization consistency. EnG also requires a sampling procedure, with the number of samples set to 1000 for accurate distribution estimation.

\begin{figure}[]
\vskip -0.4in
\centering
\begin{tabular}{cc}
    \includegraphics[width=0.47\columnwidth]{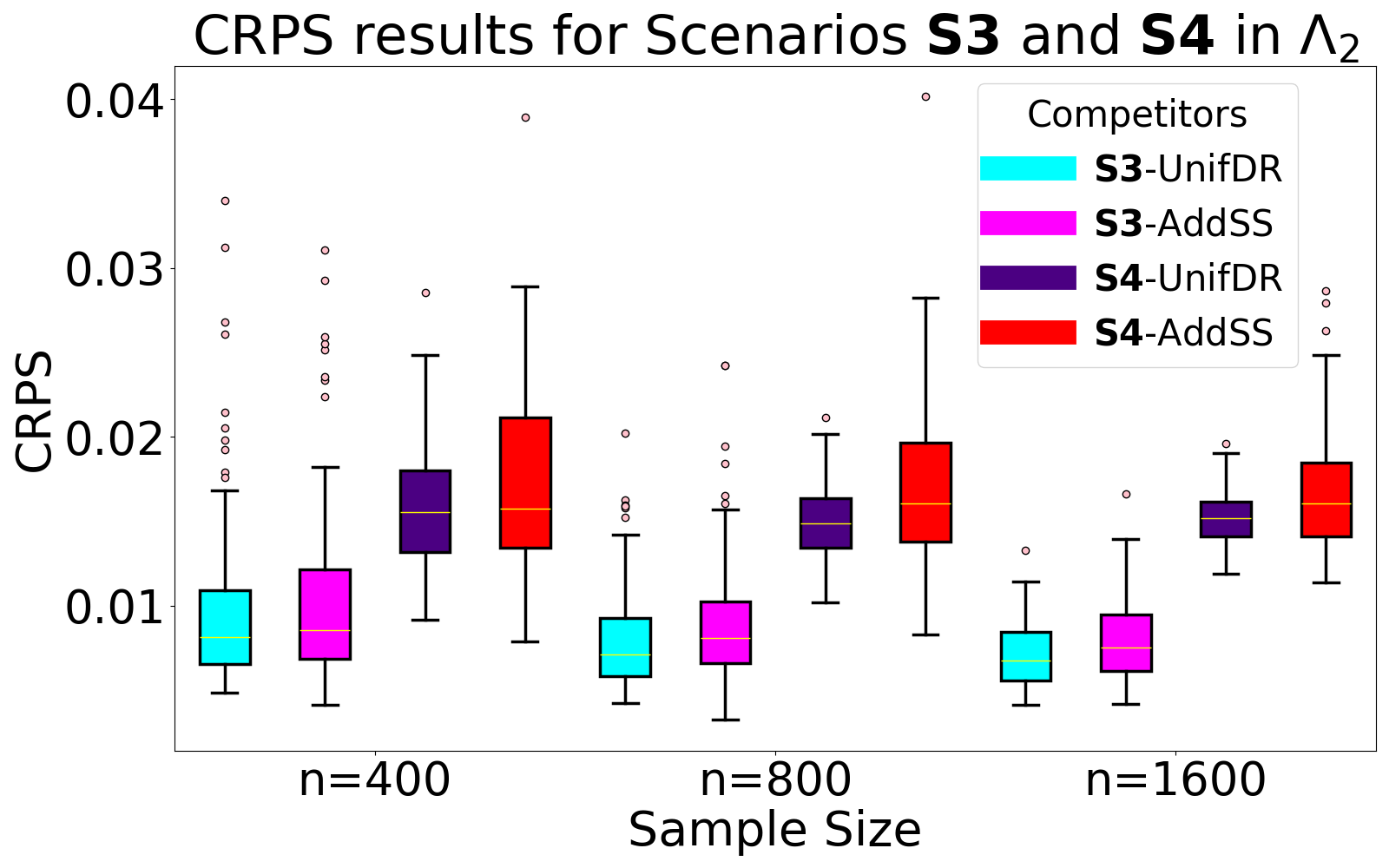} &
        \hspace{-1.1em}
\includegraphics[width=0.47\columnwidth]{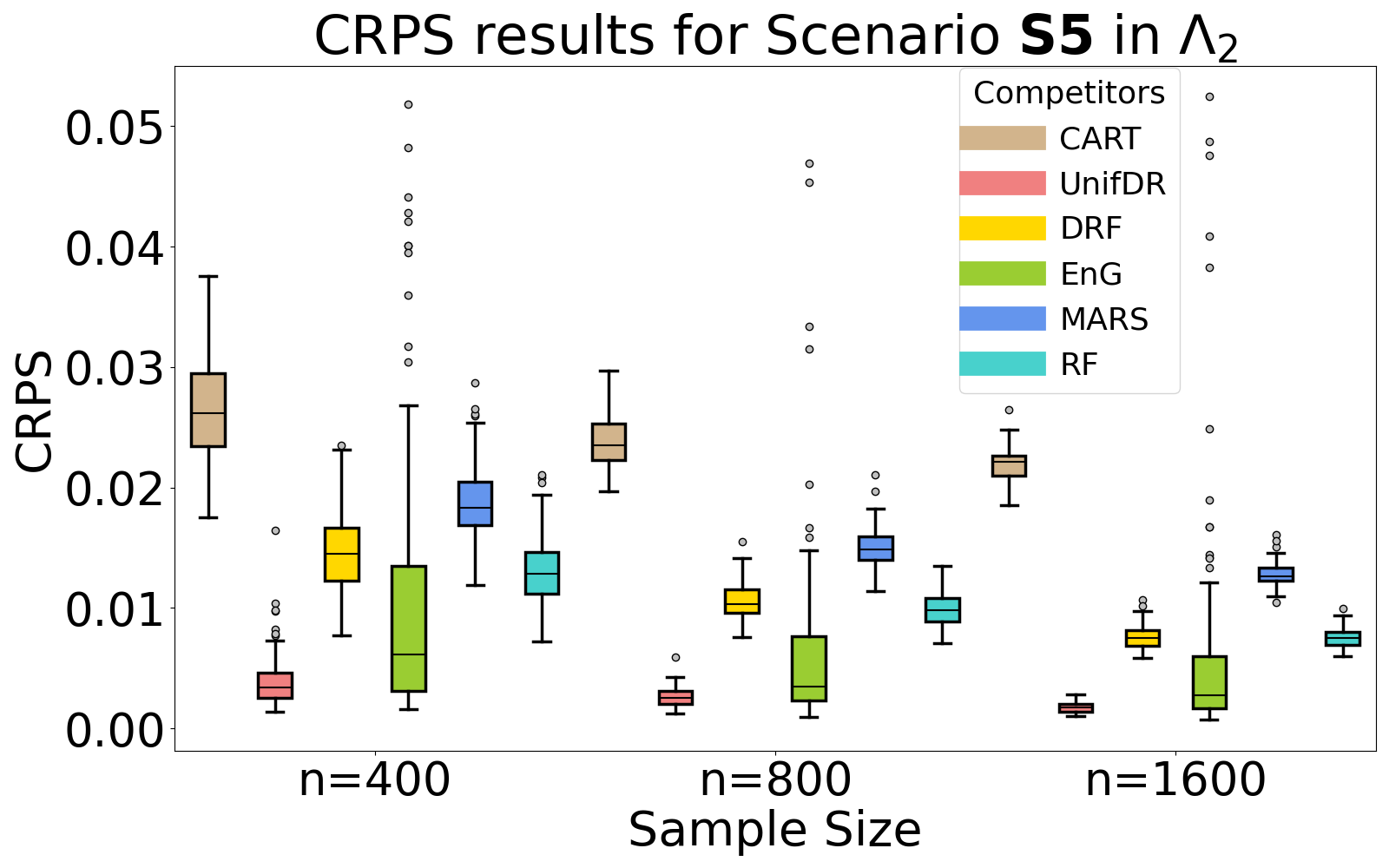} \\
\end{tabular}
\caption{Box plots of CRPS results in $\Lambda_2$. The left plot corresponds to {\bf S3} and {\bf S4}, while the right plot displays the results for {\bf S5}.}
\label{fig:2}
\vskip -0.2in
\end{figure}

\begin{table}[]
\vskip 0.07in
\caption{Evaluation metrics for {\bf{UnifDR}} (Trend Filtering approach) and its competitor AddSS on the 2015 Chicago crime dataset.}
\label{tab:RD1}
\begin{center}
\begin{small}
\begin{sc}
\begin{tabular}{lcc}
\toprule
Method & CRPS (Mean $\pm$ Std) & MSD (Mean $\pm$ Std) \\
\midrule
UnifDR & \textbf{0.0976$\pm$0.0017} & \textbf{0.2509$\pm$0.0025} \\
AddSS    & 0.1223$\pm$0.0078 & 0.2850$\pm$0.0391 \\
\bottomrule
\end{tabular}
\end{sc}
\end{small}
\end{center}
\vskip -0.3in
\end{table}

{\bf{Performance Evaluation:}}
For each scenario, datasets with sample sizes \(n \in \{400, 800, 1600\}\) are generated, with each experiment repeated 100 times using Monte Carlo simulations. Evaluations are conducted at 100 evenly spaced points \(t\) from three fixed intervals: \(\Lambda_1=[-1,0.4]\), \(\Lambda_2=[-2,2]\), and \(\Lambda_3=[0.8,10]\). Each dataset is randomly split into 75\% training and 25\% test sets. Competing models undergo 5-fold cross-validation on the training data for hyperparameter tuning, with performance assessed on the test set. For the isotonic regression method, test set predictions are obtained via naive nearest neighbor interpolation. The accuracy of the estimated CDFs \(\widehat{F}_i(t)\) relative to the true CDFs \(F_i^*(t)\) is evaluated using the following performance metrics, averaged over 100 Monte Carlo repetitions.
\textbf{CRPS:}  
    CRPS evaluates the overall fit of \(\widehat{F}_i(t)\) to \(F_i^*(t)\) across all evaluation points in \(\mathbb{R}\), see Section \ref{CRPS-equation-def}. Since the evaluations in experiments are performed over a finite set of 100 values in \(\Lambda\), CRPS is approximated via a Riemann sum:
    $
    \text{CRPS} = \frac{1}{|\text{Test}|} \sum_{i \in \text{Test}} \frac{1}{100} \sum_{t \in \Lambda} \Big( \widehat{F}_i(t) - F_i^*(t) \Big)^2,
    $
    where \(|\text{Test}|=\frac{n}{4}\) is the size of the test set.
 \textbf{Maximum Squared Difference (MSD):}  
    MSD captures the worst-case discrepancy between \(\widehat{F}_i(t)\) and \(F_i^*(t)\), and is approximated as:
    $
    \text{MSD} = \max_{t \in \Lambda} \frac{1}{|\text{Test}|} \sum_{i \in \text{Test}} \Big( \widehat{F}_i(t) - F_i^*(t) \Big)^2.
    $


The results below focus on the CRPS metric for Scenarios \textbf{S3}, \textbf{S4}, and \textbf{S5}, where the CDFs are evaluated at the points in \(\Lambda_2\). Additional results, including MSD performance and CRPS evaluations at \(\Lambda_1\) and \(\Lambda_3\) across all scenarios, as well as the \(\Lambda_1\)-based CRPS for Scenarios \textbf{S1}, \textbf{S2}, and \textbf{S6}, are provided in Appendix~\ref{ad-sim}. Figure \ref{fig:2} presents the performance of {\bf UnifDR} in {\bf S3} and {\bf S4}, where we compare the Trend Filtering approach with AddSS. The same figure also includes results for {\bf S5}, where {\bf UnifDR} utilizes the Dense ReLU Networks method against five state-of-the-art competitors: CART, MARS, RF, DRF, and EnG. In all scenarios, {\bf UnifDR} consistently outperforms competing methods, with its performance superiority becoming more pronounced as the sample size increases. This dominance is further confirmed by the extended evaluations in Appendix \ref{ad-sim}, reinforcing the robustness of {\bf UnifDR} across various conditions.


\section{Real data application}
\label{RealD-sec}

In this section, we evaluate the performance of the proposed \textbf{UnifDR} method using both the Trend Filtering and Dense ReLU network procedures on real-world datasets. 

\subsection{Chicago crime data} 

We analyze the 2015 Chicago crime dataset, available at \url{https://data.gov/open-gov/}, which records reported crimes in Chicago throughout the year. Following \citet{tansey2018maximum}, the spatial domain is discretized into a $100 \times 100$ grid, where each grid cell aggregates crime counts within its spatial boundary. The response variable is defined as the log-transformed total crime counts per grid cell. Grid cells with zero observed crimes are excluded, yielding a final dataset of 3,844 grid cells. The {\bf UnifDR'} Trend Filtering procedure does not use covariates. Instead, the spatial grid is treated as an ordered sequence, assuming a smooth spatial trend. The grid cells are ordered lexicographically. The dataset is randomly partitioned into 100 train (75\%)–test (25\%) splits, and evaluation is conducted at evenly spaced points \(\Lambda\) in the interval [-1,6]. Performance is assessed using the Continuous Ranked Probability Score (CRPS) and Maximum Squared Difference (MSD) metrics, comparing estimated CDFs \( \widehat{F}_i(t) \) against empirical indicators \( w_i(t) \), where $t\in \Lambda$.
AddSS is used as a benchmark for the Trend Filtering approach on the same dataset. Table \ref{tab:RD1} presents CRPS and MSD metrics, demonstrating {\bf UnifDR} superior performance. Furthermore, Figure \ref{fig:RD1} in Appendix \ref{Chi-Cri-ex-DRN} provides a
visualization of $\widehat{F}_i(t)$ at $t = 3$ for both competing methods. The same data set
is analyzed using {\bf UnifDR} with the Dense ReLU network
framework, with further details in Appendix \ref{other-rd}.

\subsection{Other real data examples}

Beyond the Chicago crime dataset, \textbf{UnifDR} is further evaluated on California housing prices and daily Ozone measurements. 
Detailed descriptions, pre-processing steps, and results are provided in the Appendix \ref{other-rd}.

\section{Conclusion}
\label{sec:conclusion} 
This paper introduced a unified framework for nonparametric distributional regression under convex and non-convex structural constraints. We established theoretical risk bounds for the estimation of cumulative distribution functions (CDFs) in various settings, including isotonic regression, trend filtering, and deep neural networks. Our analysis leveraged continuous ranked probability scores (CRPS) and worst-case mean squared error (MSE) to quantify estimation accuracy, demonstrating that structured constraints such as monotonicity, bounded total variation, and hierarchical function composition lead to improved estimation accuracy. Our theoretical results establish explicit convergence guarantees for isotonic, strengthening prior findings. Moreover, the trend filtering estimator achieves rates consistent with classical one-dimensional regression results. For deep neural networks, we show that dense ReLU-based estimators achieve comparable rates under hierarchical composition constraints, aligning with existing results in structured regression. Experiments on simulated and real datasets further validated the theoretical guarantees, with our proposed methods consistently outperforming alternative approaches. In particular, {\bf{UnifDR}}, the distributional regression framework we study, demonstrated superior performance across all considered settings.

An important avenue for future research is extending our theory to dependent data settings. Many real-world applications involve time-series data, spatial data, or network-structured data, where dependencies among observations must be accounted for.

\newpage
\bibliographystyle{plainnat}  
\bibliography{references}     

\newpage
\appendix

\section{Fast Rates}
\label{fast-rates-sec}

This appendix develops refined risk bounds for distributional regression estimators under additional structural assumptions on the true signal. While our general theory establishes minimax-optimal rates under convex constraints such as monotonicity or bounded total variation, certain low-complexity signal classes allow for significantly faster convergence.

We provide two canonical examples illustrating this phenomenon. The first focuses on isotonic regression, where we show that if the true distribution functions have only a small number of strict increases, the estimator achieves a nearly parametric rate. The second concerns trend filtering, where we demonstrate that sparsity in the higher-order differences of the signal—combined with a minimum segment length condition—leads to similarly fast convergence.

These results extend analogous adaptive risk bounds from classical point estimation \citep{chatterjee2015risk, guntuboyina2020adaptive} to the distributional regression setting.

\subsection{Isotonic Regression}

We begin by revisiting the isotonic case. In the main text, Corollary~\ref{thm_isotonic} established a general risk bound of order \( n^{-2/3} \) under minimal monotonicity assumptions. Here, we refine this analysis by showing that much faster rates, nearly parametric, are achievable when the true distribution function \( F^*(t) \) exhibits low complexity, in the sense of having few strict increases.

To formalize this, consider the isotonic constraint set
\begin{equation}
    \label{eqn:isotonic_k2}
    K := \left\{ \theta \in \mathbb{R}^n : \theta_1 \leq \theta_2 \leq \cdots \leq \theta_n \right\},
\end{equation}
and assume \( F^*(t) \in K \) for all \( t \in \mathbb{R} \). Define
\[
k(t) := \left| \left\{ i \in \{1, \dots, n - 1\} : F_i^*(t) < F_{i+1}^*(t) \right\} \right|.
\]
This quantity counts the number of strict increases in \( F^*(t) \), and hence controls the complexity of the signal.

\begin{corollary}
\label{thm_isotonic-fr}
Let \( \widehat{F}_t \) be the estimator defined in \eqref{eqn:estimator} with \( K_t = K \) as in \eqref{eqn:isotonic_k2}. Assume \( F^*(t) \in K \) and that (\ref{eqn:support}) holds for a compact set \( \Omega \). Then
\[
\mathbb{E} \left( \frac{1}{n} \sum_{i=1}^n \mathrm{CRPS}(\widehat{F}_i, F_i^*) \right)
\leq C \underset{t \in \mathbb{R}}{\sup} \cdot \frac{1 + k(t)}{n} \log \left( \frac{e n}{1 + k(t)} \right)
\]
for some constant \( C > 0 \). Moreover, this bound also applies to the sorted estimator \( \widetilde{F} \) from Corollary~\ref{cor1}, provided each \( F_i^* \) is continuous.
\end{corollary}

This result improves on the general \( n^{-2/3} \) bound when the number of strict increases \( k(t) \) is small. For example, if \( F^*(t) \) is piecewise constant with at most \( s \) jumps, then \( k(t) \leq s \), and the bound becomes \( \frac{1 + s}{n} \log \left( \frac{e n}{1 + s} \right) \), nearly achieving the parametric rate when \( s = O(1) \). The proof of Corollary~\ref{thm_isotonic-fr} is provided in Appendix~\ref{proofthm_isotonic-fr}.

\subsection{Trend Filtering}

We now turn to the trend filtering setting. In the main text, Corollary~\ref{thm5} established a general risk bound of order \( n^{-2r/(2r+1)} \) under a total variation constraint of order \( r \), matching the minimax-optimal rate for function estimation under such constraints \citep{mammen1997locally, tibshirani2014adaptive}.

Here, we show that significantly faster rates are achievable in the distributional regression setting when the signal exhibits additional structure. In particular, when the \( r \)th-order differences of the vector \( F^*(t) \) are sparse and well-separated, the estimator can attain nearly parametric accuracy up to logarithmic factors.

Define
\begin{equation}
	\label{eqn:tv_set-fr}
	K_t \,:=\,   \left\{  \theta   \in \mathbb{R}^n  \,:\,    	\mathrm{TV}^{(r)}(\theta)\le V_t^*  \right\},
\end{equation}
where $V_t^*=\mathrm{TV}^{(r)}(F^*(t)).$ 
\begin{corollary}
\label{Cor-fast-rat-trend}
Consider the estimator in (\ref{eqn:estimator}) with $K_t$ as in (\ref{eqn:tv_set-fr}) for an integer $r$ satisfying $r\geq 1$. Suppose $s=\left\|D^{(r)} F^*(t)\right\|_0$ and $S=\left\{j:\left(D^{(r)} F^*(t)\right)_j \neq 0\right\}$ for all $t$. Let $j_0<j_1<\ldots<j_{s+1}$ be such that $j_0=1, j_{s+1}=n-r$ and $j_1, \ldots, j_s$ are the elements of $S$. With this notation define $\eta_{j_0}=\eta_{j_{s+1}}=0$. Then for $j \in S$ define $\eta_j$ to be 1 if $\left(D^{(r-1)} F^*(t)\right)_j<\left(D^{(r-1)} F^*(t)\right)_{j+1}$, otherwise set $\eta_j=-1$. Suppose that $F^*(t)$ satisfies the following minimum length assumption
$$
\min _{l \in[s], \eta_{j_l} \neq \eta_{j_l+1}}\left(j_{l+1}-j_l\right) \geq \frac{c n}{s+1}
$$
for some constant $c$ satisfying $0 \leq c \leq 1$. Then, for $\sup_t V_t^* \leq  V^*,$
\begin{equation}
      \label{eqn:e22-fr}
        \begin{array}{l}
\displaystyle     \mathbb{E}\left(  \frac{1}{n} \sum_{i=1}^{n} \mathrm{CRPS}( \widehat{F}_i,F_i^*  )   \right) \,\leq \, C \left[\max \left\{\frac{V^*}{n^{r-1}}, 1\right\} \frac{(s+1)}{n} \log \left(\frac{e n}{s+1}\right)  \right]
\end{array}
\end{equation}
for a positive constant $C$. Moreover, the upper bound in (\ref{eqn:e22-fr}) also holds for the corresponding sorted estimators $ \widetilde{F}_i$ as defined in Corollary \ref{cor1}, provided that each function $F_i^*$ is continuous and $V_t^* \leq  V^*$ for all $t$. 
\end{corollary}

This result shows that, under a sparse difference structure and a minimum segment length condition, trend filtering estimators for distributional regression can achieve nearly parametric accuracy (up to logarithmic factors). The bound in (\ref{eqn:e22-fr}) mirrors the adaptive rates derived for standard trend filtering in point estimation problems; see, for example, \citet{guntuboyina2020adaptive}. Our extension demonstrates that the same refined rate behavior persists in the more general context of distributional regression, where the goal is to estimate full conditional distributions rather than scalar means.

The proof of Corollary~\ref{Cor-fast-rat-trend} is provided in Appendix~\ref{Corfastrattrendproof}.

\section{General Result for Penalized Estimators}
\label{pen-conv}

A natural alternative to shape-constrained estimators is the use of penalized estimators, where the penalty term promotes a desired behavior in the signal being estimated. Motivated by this approach, we present a general result for distributional regression using penalized estimators in this subsection. Specifically, consider estimators of the form
\begin{equation}
	\label{eqn:penalized}
	\hat{F}(t) \,:=\,   \underset{\theta \in \mathbb{R}^n}{\arg \min} \left\{   \frac{1}{2}\|  w(t)  - \theta \|^2    \,+\,   \lambda_t  \mathrm{pen}_t(\theta) \right \},
\end{equation}
where  $\lambda_t >0$ is a tuning parameter and $\mathrm{pen} \,:\, \mathbb{R}^n \rightarrow \mathbb{R}$  is a penalty function. We now present our main result for the penalized estimator defined in (\ref{eqn:penalized}).

\begin{theorem}
	\label{thm6}
	Suppose that $\mathrm{pen}_t(\cdot)$
 is convex for all $t$ and it is a semi-norm.
In addition, assume that 
 \(
    \underset{t \in \mathbb{R}}{\sup}  \,\,\mathrm{pen}_t (F^*(t)  ) \,\leq \,  V.
 \)
 Let  $K \,:=\,    \{  \theta   \in \mathbb{R}^n\,:\,  \mathrm{pen}_t(\theta )   \,\leq \, 6 V\} $.  Then for any $\eta> 0$, it holds that 
 \begin{equation}
     \label{eqn:penalized_est}
      	\begin{array}{l}
\displaystyle  	 \mathbb{P}\left(   \underset{t \in \mathbb{R}}{\sup}     \sum_{i=1}^{n}   \left(  \widehat{F}_i(t)    - F^*_i(t)  \right)^2   > 2\eta^2   \right)
        \leq    \frac{C}{\eta^2}  \int_{0}^{\eta/4}    \sqrt{\log N(\varepsilon,K\cap   B_{\varepsilon}(0),\|\cdot\|   )} d\varepsilon  +  \frac{ C  \sqrt{\log n}}{\eta},
 	\end{array}
 \end{equation}
 for some constant $C>0$, proivide that we set $\lambda_t \,=\, \eta^2/  4   \mathrm{pen}(  F^*(t)).$
\end{theorem}

Theorem \ref{thm6} demonstrates that achieving a uniform upper bound on the MSE can be accomplished by controlling the covering number of sets of the form $K\cap   B_{\varepsilon}(0)$, as outlined on the right-hand side of (\ref{eqn:penalized_est}).

We now turn to a statistical guarantee for the penalized version of trend filtering in distributional regression. This result follows directly from Theorem \ref{thm6} and involves a calculation analogous to the proof of (\ref{eqn:e23}).

\begin{corollary}
    \label{cor3}
	Consider the estimator in (\ref{eqn:penalized}) with  $\mathrm{pen}_t(\theta ) := \mathrm{TV}^{(r)}(\theta) $ and $\lambda_t$ chosen as in Theorem \ref{thm6}. Then
         \begin{equation}
       	\label{eqn:e24}
       	 \underset{t \in \mathbb{R}}{\sup} \,    \sum_{i=1}^{n}   \frac{1}{n}\left(  \widehat{F}_i(t)    - F^*_i(t)  \right)^2  = O_{\mathbb{P}}\left(\frac{V^{ \frac{2}{2r+1}  }  }{ n^{ \frac{2r}{2r+1}  } } + \frac{\log n}{n}  \right),
       \end{equation}
       where $V\,:=\, \underset{t \in \mathbb{R}}{\sup}  \,\,\mathrm{pen}_t (F^*(t)  ) $.
\end{corollary}

\newpage

\section{Dense ReLU Networks: assumption and definitions}
\label{Appendix-DRN}

In this appendix, we provide additional details for Section~\ref{DRN-section}. Before outlining our assumptions on the functions \( F^*(t) \), we introduce notation related to dense ReLU networks. To that end, we describe a dense neural network with architecture $(L,k)$ employing the ReLU activation function given as $\rho(s) =\max\{0,s\}$ for any $s\in \mathbb{R}$. Such a network is represented as a real-valued function $f: \mathbb{R}^d \rightarrow \mathbb{R}$  satisfying the following properties:
\begin{align}
f(x)=\sum_{i=1}^{k_L} c_{1, i}^{(L)} f_i^{(L)}(x)+c_{1,0}^{(L)}  \label{eq:form of approximation function 0}
\end{align}
for weights $c_{1,0}^{(L)}, \ldots, c_{1, k_L}^{(L)} \in \mathbb{R}$ and for $f_i^{(L)}$ 's recursively defined by
\begin{align}
f_i^{(s)}(x)=\rho\left(\sum_{j=1}^{k_{s-1}} c_{i, j}^{(s-1)} f_j^{(s-1)}(x)+c_{i, 0}^{(s-1)}\right) \label{eq:form of approximation function L}
\end{align} for some $c_{i,0}^{(s-1)}, \dots, c_{i, k_{s-1}}^{(s-1)} \in \mathbb{R}$,
$s \in \{2, \dots, L\}$,
and $f_i^{(1)}(\mathbf{x}) = \rho \left(\sum_{j=1}^d c_{i,j}^{(0)} x^{(j)} + c_{i,0}^{(0)} \right)$
\\
with  $c_{i,0}^{(0)}, \dots, c_{i,d}^{(0)} \in \mathbb{R}$.

Next we provide some notation necessary for define the class of signals where the  $F^*(t)$'s
 belong. 
 


\begin{definition}[$(p, C)$-smoothness]\label{definition: p,c smoothness}
Let $p=q+s$ for some $q \in \mathbb{N}=\mathbb{Z}^{+}\cup\left\{0\right\}$ and $0<s \leq 1$. We say that a function $g: \mathbb{R}^d \rightarrow \mathbb{R}$ is  $(p, C)$-smooth, if for every $\alpha=\left(\alpha_1, \ldots, \alpha_d\right) \in \mathbb{N}^d$, with $d \in \mathbb{Z}^{+}$ , where  $\sum_{j=1}^d \alpha_j=q$, the partial derivative $\partial^q g /\left(\partial u_1^{\alpha_1} \ldots \partial u_d^{\alpha_d}\right)$ exists and 
$$
\left|\frac{\partial^q g}{\partial u_1^{\alpha_1} \ldots \partial u_d^{\alpha_d}}\left(u\right)-\frac{\partial^q g}{\partial u_1^{\alpha_1} \ldots \partial u_d^{\alpha_d}}\left(v\right)\right| \leq C\| u- v\|^s
$$
for all $u, v \in \mathbb{R}^d$.
\end{definition}


Let us now define the generalized hierarchical interaction models $\mathcal{H}(l, \mathcal{P})$.

\begin{definition} [Space of Hierarchical Composition Models, \cite{kohler2021rate}]\label{definition: hierarchical composition model}
For $l=1$ and  smoothness constraint $\mathcal{P} \subseteq(0, \infty) \times \mathbb{N}$, the  space of hierarchical composition models is defined as
$$
\begin{array}{l}
         \mathcal{H}(1, \mathcal{P})    :=
      \left\{h: \mathbb{R}^d \rightarrow \mathbb{R}: h(a )=m\left(a_{(\pi(1))}, \ldots, a_{(\pi(M))}\right),\right. \text { where } \\
         \,\,\, m: \mathbb{R}^M \rightarrow \mathbb{R} \text { is }(p, C) \text {-smooth for some }(p, M) \in \mathcal{P}  
 \text { and } \pi:\{1, \ldots, M\} \rightarrow\{1, \ldots, d\}\} .
    \end{array}
    $$
For $l>1$, we set 
$$
\begin{array}{l}
 \mathcal{H}(l, \mathcal{P})    := 
 \left\{h: \mathbb{R}^d \rightarrow \mathbb{R}: h(\mathbf{x})=m\left(f_1(a), \ldots, f_M(a)\right),\right.  \text { where }\\
\,\,\,  m: \mathbb{R}^M \rightarrow \mathbb{R} \text { is }(p, C) \text {-smooth for some }(p, M) \in \mathcal{P}
 \text { and } \left.f_i \in \mathcal{H}(l-1, \mathcal{P})\right\}.
\end{array}
$$	
\end{definition}

With the notation above, we are ready to state our assumption on the true signals in the spirit of \cite{kohler2021rate}.

\begin{assumption}
    \label{as1}
Suppose that for all $t$ the function  $G^*(\cdot,t)$ is in the class $\mathcal{H}(l,\mathcal{P})$ as in Definition  \ref{definition: hierarchical composition model}. In addition, assume that each function $g^{t}$ in the definition of $G^*(\cdot,t)$ can have different smoothness $p_{g^t} =  q_{g^t} +s_{g^t}$, for $q_{g^t} \in \mathbb{N}$, $s_{g^t} \in (0,1]$,  and of potentially different input dimension $M_{g^t} $, so that $(p_{g^t}, M_{g^t} ) \in \mathcal{P}$. Let $M_{\max}$ be the largest input dimension and $p_{\max}$ the largest smoothness of any of the functions $g^t$ for all $t$. Suppose that for each $g^{t}$ all the partial derivatives of order less than or equal to $q_{g^t}$ are uniformly bounded by constant $C_{\mathrm{Smooth} }$, and each function $g^t$ is Lipschitz continuous
with Lipschitz constant $C_{\mathrm{Lip} } \geq1 $. Also, assume that $\max\{p_{\max},M_{\max} \} =O(1)$.

\end{assumption}

\newpage

\section{Rearrangement of estimates}

\begin{figure}[]
\begin{center}
\vskip 0.1in
\includegraphics[width=2in]{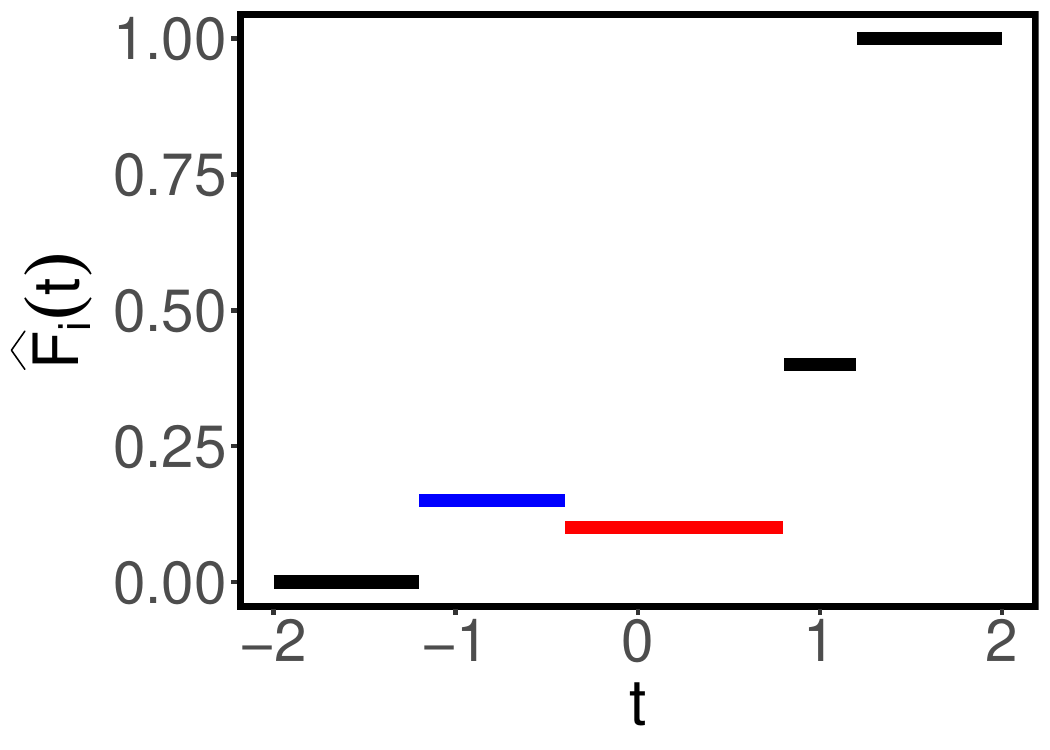}
\includegraphics[width=2in]{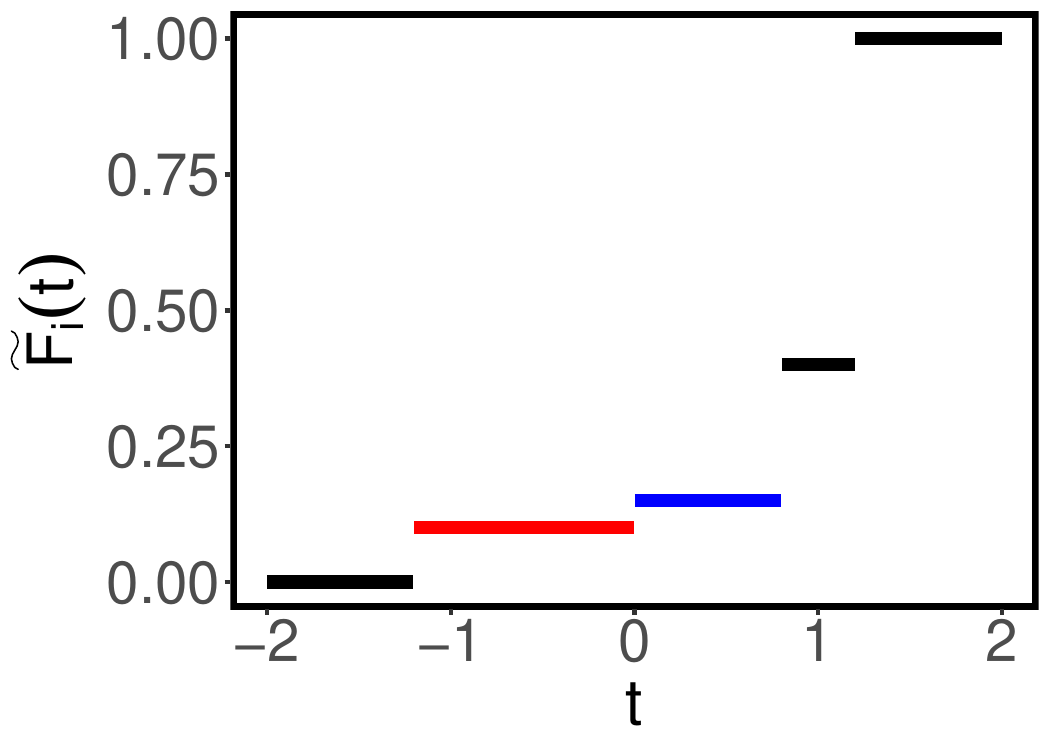}
\caption{The plot in the left shows a display of an example of a function $\widehat{F}_i$ and the right panel shows the corresponding rearrangement $\widetilde{F}_i$ as described in Corollary \ref{cor1}.}
\label{fig10}
\end{center}
\vskip -0.2in
\end{figure}

Figure \ref{fig10} shows an example of the original $\widehat{F}_i$ and its rearrangement $\widetilde{F}_i$ as in Corollary \ref{cor1}, which is guaranteed to be non-decreasing.

\newpage

\section{Additional Numerical Results}
\label{ad-sim}

This appendix provides an extensive evaluation of the proposed method, {\bf UnifDR}, including additional results and analyses omitted from Sections \ref{simu-data} and \ref{RealD-sec}. These supplementary results further demonstrate the effectiveness and robustness of our methods across diverse settings. The appendix presents:

\begin{itemize}
    \item Additional real data applications to illustrate the practical utility of the proposed methods (Appendix \ref{other-rd}).
    \item Comprehensive evaluations on alternative evaluation sets (\(\Lambda_1\) and \(\Lambda_3\)) supplementing the results for \(\Lambda_2\) in Section \ref{simu-data}. This section also includes missing evaluation results for \(\Lambda_2\) related to the CRPS metric (Appendix \ref{lambda1-lambda3}).
    
    \item Performance results based on the Maximum Squared Difference (MSD) metric across all scenarios, which were not included in Section \ref{simu-data} (Appendix \ref{msd-results}).
\end{itemize}

\subsection{Additional Real Data Applications}
\label{other-rd}

This appendix presents two additional real-world data applications to further demonstrate the effectiveness of the proposed methods. These examples span different domains, illustrating the versatility and robustness of our approach. Each case study includes a description of the dataset, the experimental setup, and a comparative performance analysis.

For each dataset, the data is randomly divided into a training subset (75\%) and a testing (25\%) subset. Model performance is evaluated using the empirical cumulative distribution function (CDF), \( w_i(t) \), computed over $100$ evenly spaced points in a predefined set $\Lambda$. The performance of each competitor is assessed using the Continuous Ranked Probability Score (CRPS) and the Maximum Squared Difference (MSD) metrics, defined as follows:
\[
\mathrm{CRPS}=\frac{1}{|\text{Test}|} \sum_{i \in \text{Test}} \frac{1}{100} \sum_{t \in \Lambda}\left(\widehat{F}_i(t)-w_i(t)\right)^2,
\]
and
\[
\mathrm{MSD}=\max _{t \in \Lambda} \frac{1}{|\text{Test}|} \sum_{i \in \text{Test}}\left(\widehat{F}_i(t)-w_i(t)\right)^2.
\]
 The proposed {\bf UnifDR} method is implemented in two variants: Trend Filtering, which captures smooth variations through total variation regularization, and Dense ReLU Networks, which leverages a deep neural network to incorporate covariate information.

\subsubsection{Chicago Crime Data with Dense ReLU Networks Approach}
\label{Chi-Cri-ex-DRN}

We revisit the \texttt{2015 Chicago crime dataset}, previously analyzed in the main text using the Trend Filtering approach. We remember that the dataset contains reported crimes in Chicago throughout 2015. As before, the spatial domain is discretized into a $100 \times 100$ grid, where each grid cell represents an aggregated crime count. The response variable remains the log-transformed total crime counts per grid cell, and grid cells with zero observed crimes are excluded, yielding a final dataset of 3,844 grid cells.  

Unlike the Trend Filtering approach which assumes a smooth
index trend, the Dense ReLU network approach leverages covariate information
for modeling crime intensity. The following covariates are included. Latitude and Longitude Bins, encapsulating spatial crime patterns. Day of the Week, represented using dummy variables for each weekday (Monday through Sunday). Beat, a categorical identifier for Chicago’s policing districts.  Arrest Indicator, a binary variable denoting whether an arrest was made (1) or not (0).

 The dataset is randomly split into 100 train (75\%) – test (25\%) partitions, with evaluation conducted at evenly spaced points \( \Lambda \) ranging from -1 to 6.  The Dense ReLU
Networks approach employs a fully connected feedforward architecture with five hidden layers of
64 neurons each, using ReLU activations. The model is trained using the Adam optimizer with a
learning rate of 0.001 over 1,000 epochs, minimizing the Binary Cross-Entropy (BCE) loss function
for improved CDF estimation.

Performance is assessed using the Continuous Ranked Probability Score (CRPS) and Maximum Squared Difference (MSD) metrics, comparing estimated CDFs $\widehat{F}_i(t)$ against empirical indicators $w_i(t)$, where $t \in \Lambda$. The methods CART, MARS, RF, DRF, and EnG serve as competitors for the Dense ReLU Networks approach. Table \ref{tab:RD1-DNN} summarizes the results demonstrating the superior performance of the Dense ReLU network relative to classical nonparametric regression methods. Additionally, Figure \ref{fig:RD1-DNN} visualizes \( \widehat{F}_i(t) \) for \( t = 3 \) across test grid cells for all competitors.

\begin{figure}[]
\begin{center}
\centerline{\includegraphics[width=\columnwidth]{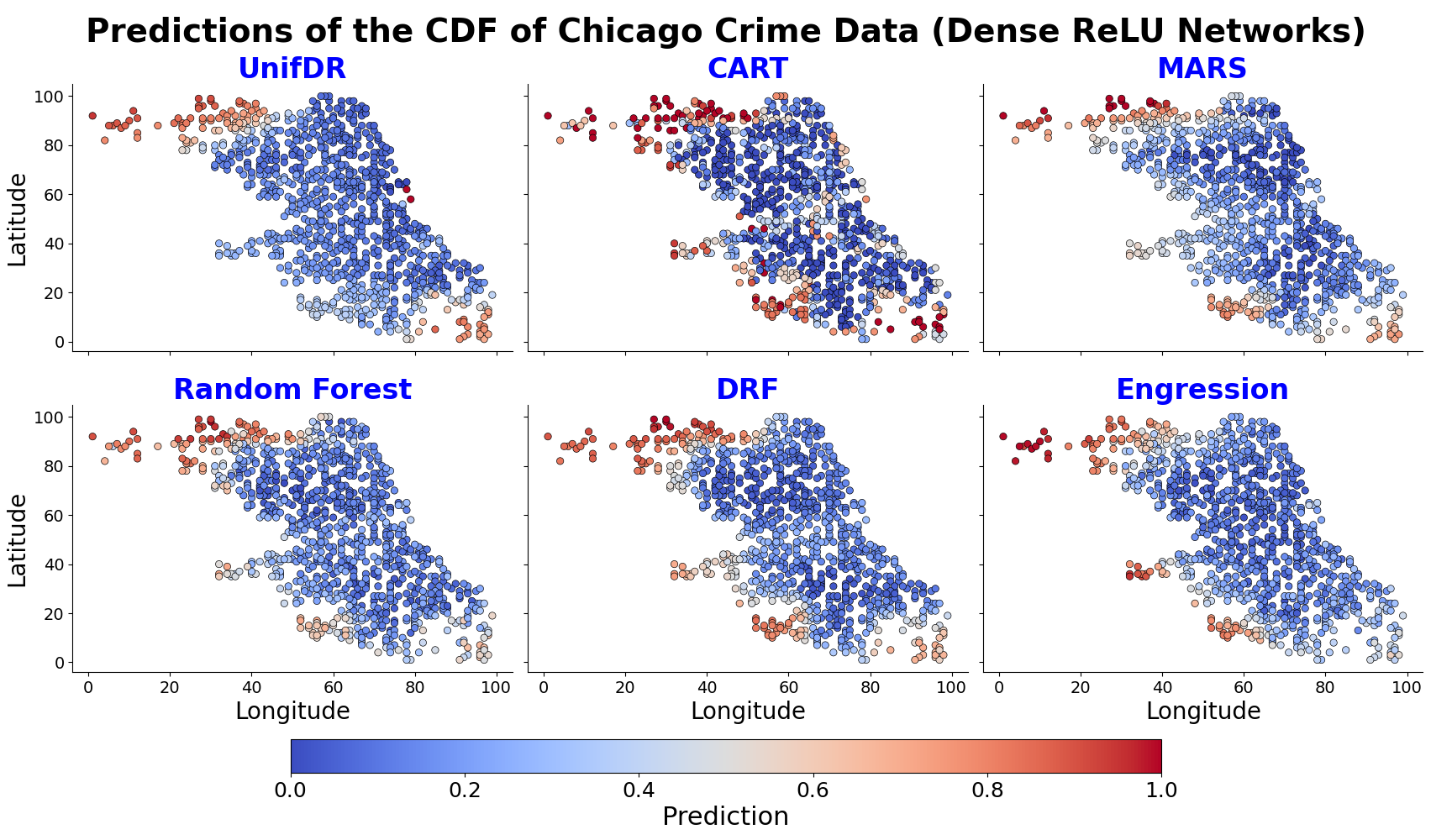}}
\caption{$\widehat{F}_i(t)$ for $t=3$ and all $i\in \text{Test}$, for all competitors.}
\label{fig:RD1-DNN}
\end{center}
\vskip -0.2in
\end{figure}

\begin{figure}[]
\begin{center}
\centerline{\includegraphics[width=\columnwidth]{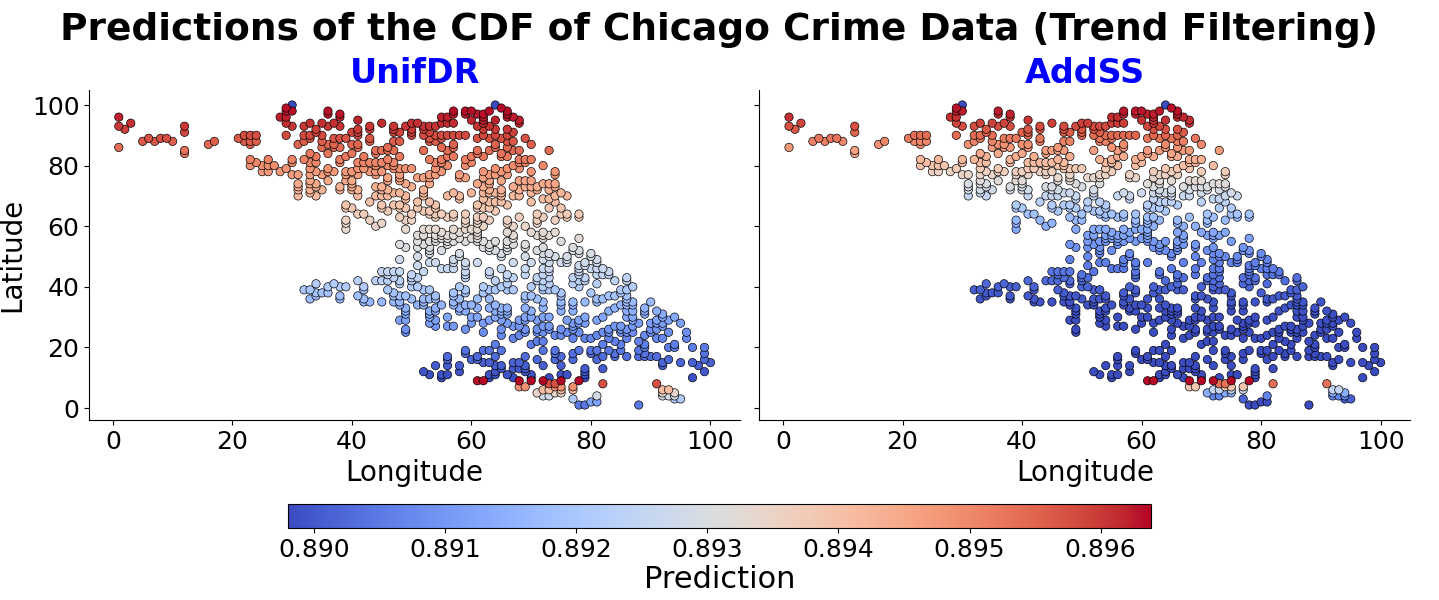}}
\caption{$\widehat{F}_i(t)$ for $t=3$ and all $i\in \text{Test}$, for all competitors for the example in Section \ref{RealD-sec}.}
\label{fig:RD1}
\end{center}
\vskip -0.3in
\end{figure}

\subsubsection{California Housing Prices}

We evaluate the effectiveness of the proposed {\bf UnifDR} method by analyzing the 1990 California housing dataset, which contains demographic and economic information from various neighborhoods across California. Originally introduced in \citet{pace1997sparse}, the dataset comprises 20,640 observations and is publicly available via the Carnegie Mellon StatLib repository at \url{http://lib.stat.cmu.edu/datasets/}, as well as the \url{https://www.dcc.fc.up.pt/~ltorgo/Regression/cal_housing.html} portal.
 
Following the approach of \citet{ye2021non}, the geographic area is discretized into a $200 \times 200$ spatial grid based on latitude and longitude coordinates. The response variable is derived by applying a logarithmic transformation to the median house values within each grid cell to enhance numerical stability and interpretability. Grid cells lacking valid data are excluded from further analysis, resulting in a final sample size of 3,165 grid cells.
 
In the Trend Filtering approach no covariates are included. Instead, the spatial grid is treated as an ordered sequence, allowing total variation regularization to capture smooth spatial variations in housing prices.  In contrast, the Dense ReLU Networks method integrates spatial features such as latitude and longitude with socioeconomic attributes like {\texttt{median$\_$income}} and {\texttt{average$\_$occupancy}}, which is computed as population divided by households. This approach enables the model to capture complex relationships influencing housing prices. The Dense ReLU neural network consists of three hidden layers, each containing 30 neurons followed by a ReLU activation function.
  
Evaluations are performed over 100 evenly spaced points within the range $\Lambda=[5,15]$.

\begin{table}[]
\caption{Evaluation metrics for UnifDR (Dense ReLU networks) and its competitors on the Chicago crime dataset.}
\label{tab:RD1-DNN}
\vskip 0.15in
\begin{center}
\begin{small}
\begin{sc}
\begin{tabular}{lcc}
\toprule
Method & CRPS (Mean $\pm$ Std) & MSD (Mean $\pm$ Std) \\
\midrule
UnifDR & \textbf{0.0811$\pm$ 0.0018} & \textbf{0.2133 $\pm$ 0.0031} \\
CART    & $0.0951\pm 0.0017$ & $0.2622 \pm 0.0071$ \\
DRF    & $0.0906\pm 0.0018$ & $0.2477 \pm 0.0042$ \\
ENG    & $0.1014\pm 0.0028$ & $0.2652 \pm 0.0045$ \\
MARS    & $0.0974\pm 0.0019$ & $0.2732 \pm 0.0047$ \\
RF    & $0.0934\pm 0.0020$ & $0.2581 \pm 0.0031$ \\
\bottomrule
\end{tabular}
\end{sc}
\end{small}
\end{center}
\vskip -0.1in
\end{table}
Figure \ref{fig:Housing-TF} presents the estimated cumulative distribution functions \( \widehat{F}_i(t) \) at \( t=12 \), comparing the performance of {\bf UnifDR} (Trend Filtering) against its competitor, AddSS. Similarly, Figure \ref{fig:RD2-DNN} illustrates the estimated cumulative distribution functions at the same evaluation point, this time comparing {\bf UnifDR} (Dense ReLU networks) against CART, MARS, RF, DRF, and EnG. The corresponding evaluation metrics, summarized in Tables \ref{tab:Housing-Results} and \ref{tab:RD2-DNN}, demonstrate that {\bf UnifDR} consistently outperforms all competitors in terms of both CRPS and MSD.

\begin{figure}[ht]
\begin{center}
\centerline{\includegraphics[width=\columnwidth]{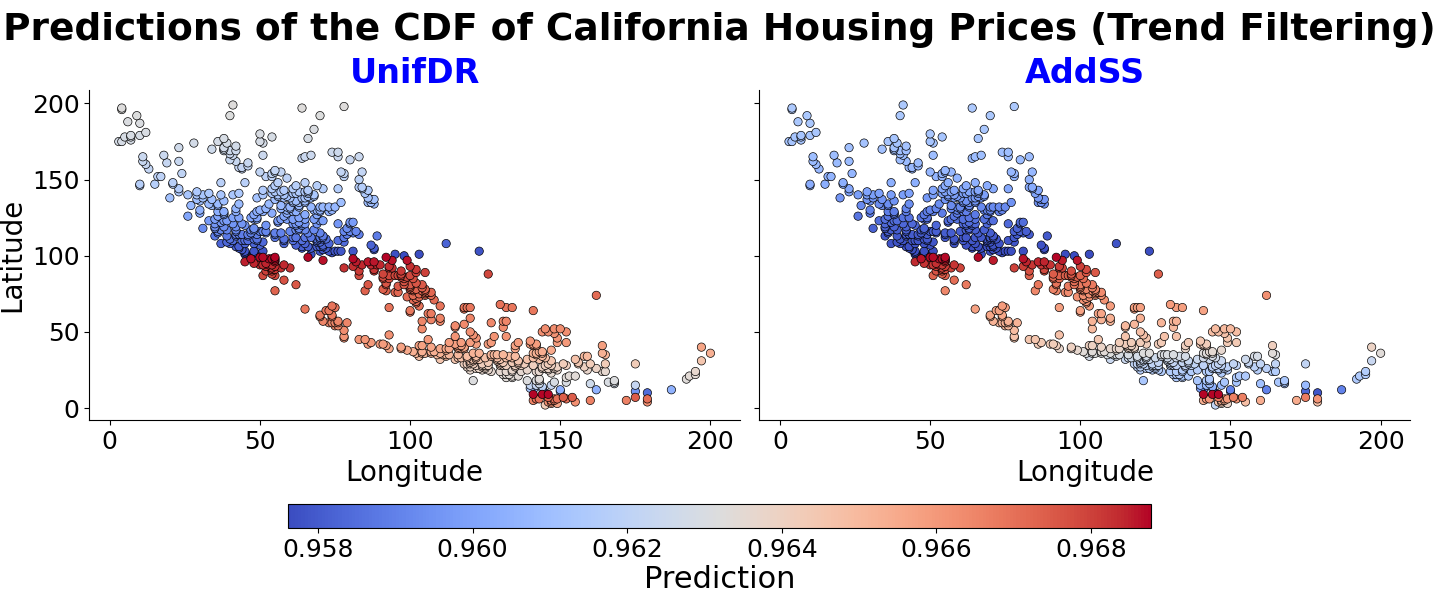}}
\caption{Estimated distribution function \( \widehat{F}_i(t) \) at \( t=12 \) for all grid cells in the test set, comparing UnifDR (Trend Filtering) and the AddSS competitor.}
\label{fig:Housing-TF}
\end{center}
\vskip -0.2in
\end{figure}

\begin{table}[t]
\caption{Evaluation metrics for UnifDR (Trend Filtering) and its competitor AddSS on the California housing dataset.}
\label{tab:Housing-Results}
\begin{center}
\begin{small}
\begin{sc}
\begin{tabular}{lcc}
\toprule
Method & CRPS (Mean $\pm$ Std) & MSD (Mean $\pm$ Std) \\
\midrule
UnifDR & {\bf{0.0343$\pm$ 0.0013}} & {\bf{0.2505 $\pm$ 0.0097}} \\
AddSS    & $0.0357\pm 0.0028$ & $0.2652 \pm 0.0407$ \\
\bottomrule
\end{tabular}
\end{sc}
\end{small}
\end{center}
\vskip -0.1in
\end{table}

\begin{figure}[]
\begin{center}
\vskip 0.1in
\centerline{\includegraphics[width=\columnwidth]{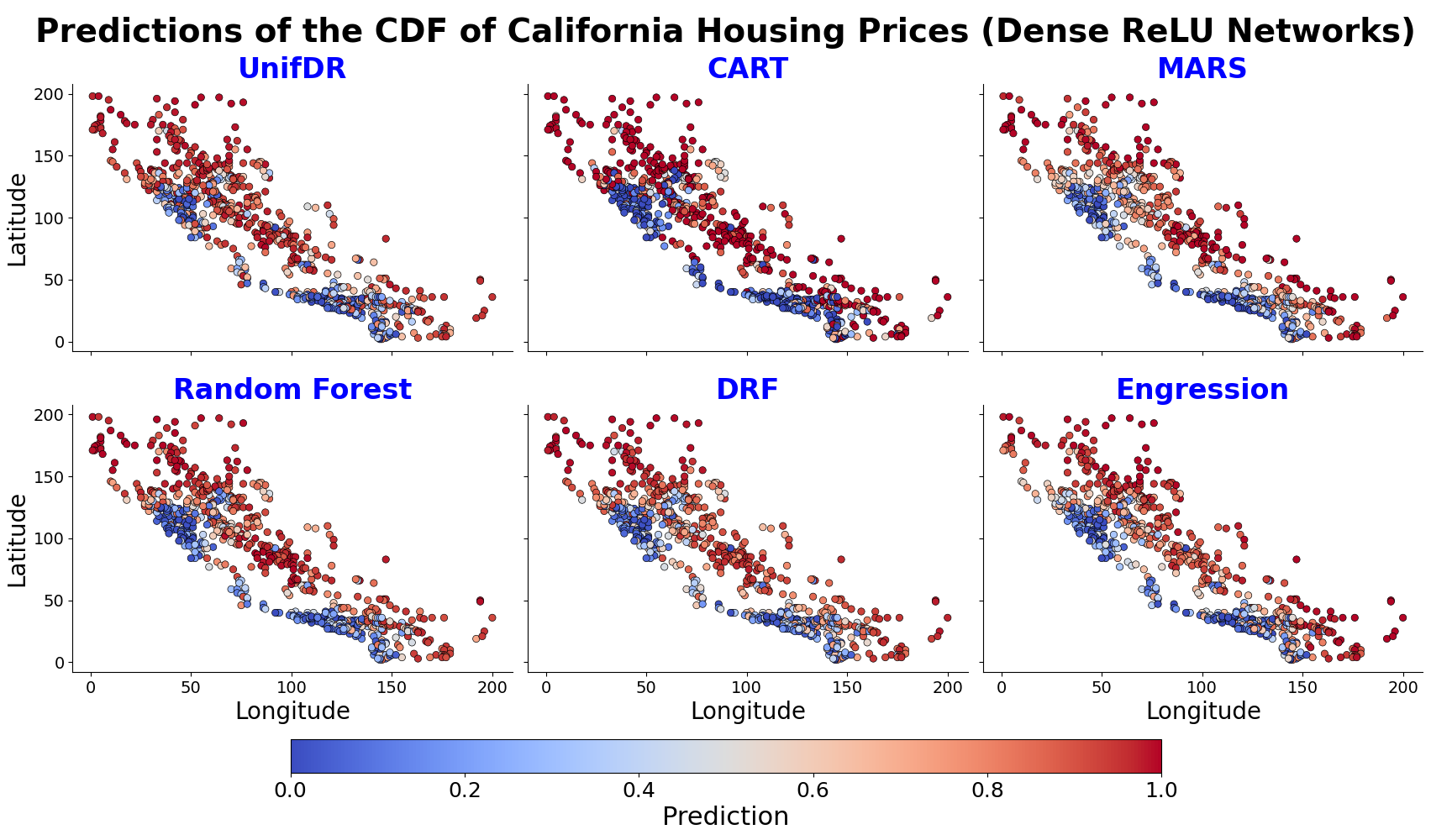}}
\caption{$\widehat{F}_i(t)$ for $t=12$ and all $i\in \text{Test}$, for all competitors.}
\label{fig:RD2-DNN}
\end{center}
\vskip -0.2in
\end{figure}

\begin{table}[]
\caption{Evaluation metrics for UnifDR (Dense ReLU networks) and its competitors on the California housing price dataset.}
\label{tab:RD2-DNN}
\vskip 0.15in
\begin{center}
\begin{small}
\begin{sc}
\begin{tabular}{lcc}
\toprule
Method & CRPS (Mean $\pm$ Std) & MSD (Mean $\pm$ Std) \\
\midrule
UnifDR & \textbf{0.0209$\pm$ 0.00055} & \textbf{0.1450 $\pm$ 0.0047} \\
CART    & $0.0258\pm  0.00053$ & $0.1984 \pm 0.0067$ \\
DRF    & $0.0229\pm 0.00051$ & $0.1755 \pm 0.0045$ \\
ENG    & $0.0244\pm 0.00059$ & $0.1924 \pm 0.0057$ \\
MARS    & $0.0233\pm 0.00057$ & $0.1798 \pm 0.0099$ \\
RF    & $0.0221\pm 0.00046$ & $0.1650 \pm 0.0047$ \\
\bottomrule
\end{tabular}
\end{sc}
\end{small}
\end{center}
\vskip -0.1in
\end{table}

\subsubsection{Ozone Data Analysis}

We further evaluate the effectiveness of the proposed {\bf UnifDR} method by analyzing ozone concentration data collected from the Environmental Protection Agency (EPA) Regional dataset. This dataset consists of daily ozone measurements collected across various monitoring stations in different regions of the United States for the year 2024. The dataset includes measurements of ozone concentration along with associated variables such as Air Quality Index (AQI), wind speed, temperature, latitude, and longitude. The available variables include: State Code, County Code, Site Number, Latitude, Longitude, Date (Year, Month, Day), Ozone concentration (in parts per million), AQI, Wind Speed (in miles per hour), Temperature (in Fahrenheit), Observation Percentage (percentage of valid observations for a given day), First Maximum Value (highest ozone level recorded in a day), First Maximum Hour (time at which the maximum value was recorded), and Observation Count (number of valid ozone measurements per day). The data are publicly accessible via the EPA's AirData portal at \url{https://aqs.epa.gov/aqsweb/airdata/download_files.html}, where historical records spanning multiple years can be downloaded.

Following the approach of previous studies, the geographic region of interest is discretized based on latitude and longitude coordinates, specifically within the range \( 30^\circ N \) to \( 50^\circ N \) and \( -153^\circ W \) to \( -70^\circ W \). In the Trend Filtering setup, each monitoring site within these bounds is uniquely identified using a lexicographic ordering of its latitude and longitude values. The response variable \( y_i \) for each site \( i \) is defined as the mean of the daily recorded ozone levels, ensuring robustness against short-term fluctuations and missing data. This approach results in a total of 1,189 unique monitoring sites, which serve as the basis for spatial trend estimation.

 In contrast, the Dense ReLU network method method allows for a flexible spatial fit by incorporating spatial location data as covariates, capturing complex relationships that influence ozone concentration.  Specifically, the model uses latitude, longitude, mean AQI, mean percentage of valid observations, mean 1st maximum ozone value, mean 1st maximum hour, and mean observation count as input features. The Dense ReLU network consists of two hidden layers, each containing 100 neurons followed by a ReLU activation function.

Evaluations are performed over 100 evenly spaced points within the range $\Lambda=[0,1]$. Figures \ref{fig:Ozone-TF} and \ref{fig:RD3-DNN} present the estimated cumulative distribution functions \( \widehat{F}_i(t) \) at \( t=0.03\), comparing the performance of {\bf UnifDR} using Trend Filtering and Dense ReLU network against their competitors. Tables \ref{tab:Ozone-Results} and \ref{tab:RD3-DNN} summarize the evaluation metrics, demonstrating that {\bf UnifDR} achieves superior performance in terms of both CRPS and MSD.

\begin{figure}[]
\begin{center}
\centerline{\includegraphics[width=\columnwidth]{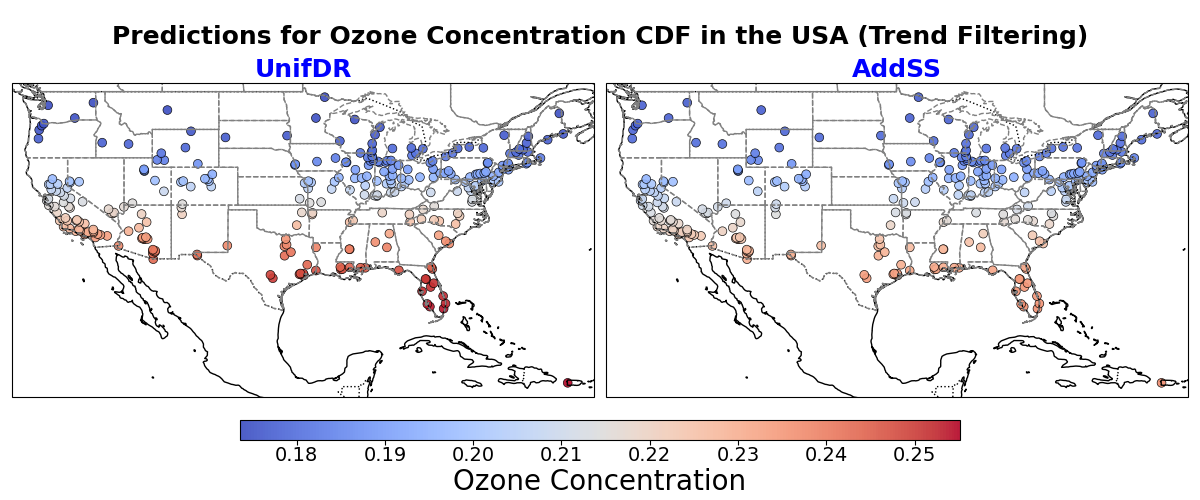}}
\caption{Estimated distribution function \( \widehat{F}_i(t) \) for all monitoring sites in the test set, comparing UnifDR (Trend Filtering) and the AddSS competitor.}
\label{fig:Ozone-TF}
\end{center}
\vskip -0.2in
\end{figure}

\begin{table}[]
\caption{Evaluation metrics for UnifDR (Trend Filtering) and its competitor AddSS on the ozone concentration dataset.}
\label{tab:Ozone-Results}
\begin{center}
\begin{small}
\begin{sc}
\begin{tabular}{lcc}
\toprule
Method & CRPS (Mean $\pm$ Std) & MSD (Mean $\pm$ Std) \\
\midrule
UnifDR & {\bf{0.0027$\pm$ 0.0002}} & {\bf{0.1581 $\pm$ 0.0124}} \\
AddSS    & $0.0035\pm 0.0002$ & $0.1987 \pm 0.0127$ \\
\bottomrule
\end{tabular}
\end{sc}
\end{small}
\end{center}
\vskip -0.1in
\end{table}

\begin{figure}[]
\begin{center}
\vskip 0.1in
\centerline{\includegraphics[width=\columnwidth]{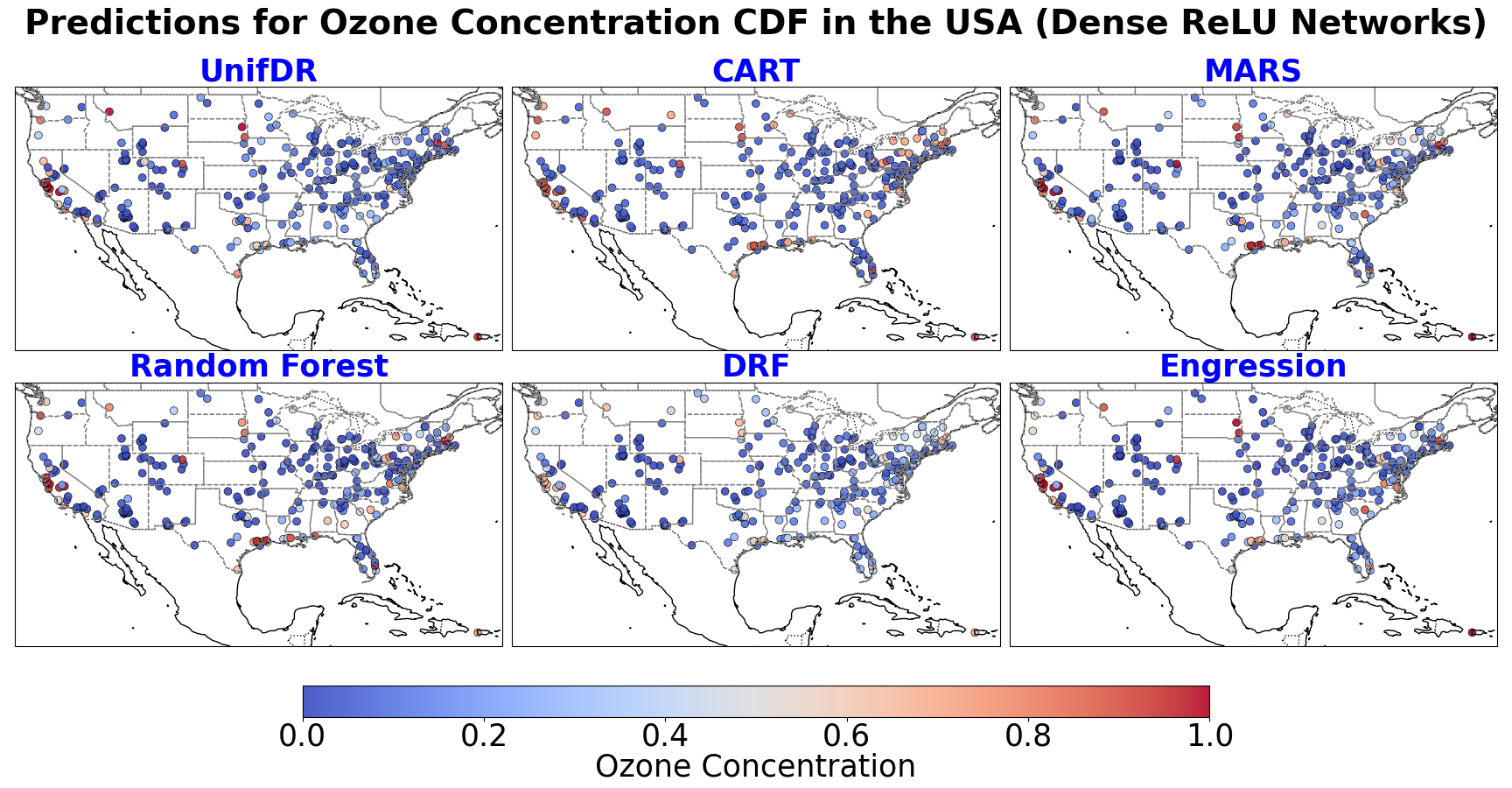}}
\caption{$\widehat{F}_i(t)$ for $t=0.03$ and all $i\in \text{Test}$, for all competitors.}
\label{fig:RD3-DNN}
\end{center}
\vskip -0.2in
\end{figure}

\begin{table}[]
\caption{Evaluation metrics for UnifDR (Dense ReLU networks) and its competitors on the Ozone dataset.}
\label{tab:RD3-DNN}
\vskip 0.15in
\begin{center}
\begin{small}
\begin{sc}
\begin{tabular}{lcc}
\toprule
Method & CRPS (Mean $\pm$ Std) & MSD (Mean $\pm$ Std) \\
\midrule
UnifDR & \textbf{0.0016$\pm$ 0.00021} & \textbf{0.0959 $\pm$ 0.0136} \\
CART    & $0.0026\pm  0.00028$ & $0.1529 \pm 0.0209$ \\
DRF    & $0.0031\pm  0.00016$ & $0.1593 \pm 0.0118$ \\
ENG    & $0.0022\pm 0.00099$ & $0.1097 \pm 0.0603$ \\
MARS    & $0.0022\pm 0.00039$ & $0.1288 \pm 0.0376$ \\
RF    & $0.0021\pm 0.00023$ & $0.1012 \pm 0.0157$ \\
\bottomrule
\end{tabular}
\end{sc}
\end{small}
\end{center}
\vskip -0.1in
\end{table}

\subsection{CRPS results on evaluation sets \texorpdfstring{\(\Lambda_1,\) and \(\Lambda_3\)}{Lambda1, and Lambda3}, and missing results for {\bf{S1}}, {\bf{S1}} and {\bf{S6}} in \texorpdfstring{\(\Lambda_2\)}{Lambda2}}
\label{lambda1-lambda3}

Section \ref{simu-data} presented results for the evaluation set \(\Lambda_2\) using the Continuous Ranked Probability Score (CRPS). However the results for scenarios \textbf{S1}, \textbf{S2} and \textbf{S6} were omitted due to space constraints. Moreover,
it is worth noting that the evaluation set \(\Lambda_2\) represents a balanced range of values centered around zero. To provide a more comprehensive analysis, this appendix includes the omitted CRPS results
for $\Lambda_2$, and the CRPS results for the alternative evaluation sets \(\Lambda_1\) and \(\Lambda_3\), which emphasize distinct distributional regions.

The new evaluation sets are defined as follows:
\begin{itemize}
    \item \(\Lambda_1\): 100 points evenly spaced between \(-1\) and \(0.4\), focusing on the lower and middle ranges of the distribution.
    \item \(\Lambda_3\): 100 points evenly spaced between \(0.8\) and \(10\), capturing the upper tail of the distribution.
\end{itemize}
This extended analysis provides deeper insights
into the robustness of the proposed methods in varying distributional regimes, including regions
with lower densities and heavier tails. As described in Section \ref{simu-data}, for each evaluation set (\(\Lambda_1\), \(\Lambda_2\) and \(\Lambda_3\)), the CRPS is computed and averaged over 100 Monte Carlo repetitions. Figures \ref{fig:lambda1-lambda3-S1-S2} through \ref{fig:lambda3-S6-CRPS} summarize the CRPS results across different scenarios and evaluation regions. Specifically, Figure \ref{fig:lambda1-lambda3-S1-S2} correspond to Scenarios \textbf{S1} and \textbf{S2} in \(\Lambda_1\), $\Lambda_2$, and \(\Lambda_3\).
 Figure \ref{fig:lambda1-lambda3-S3-S4} presents Scenarios \textbf{S3} and \textbf{S4} in \(\Lambda_1\) and \(\Lambda_3\).
Figure \ref{fig:lambda1-lambda3-S5} focuses on Scenario \textbf{S5} in \(\Lambda_1\) and \(\Lambda_3\).
 Figures \ref{fig:lambda1-S6-CRPS}, \ref{fig:lambda2-S6-CRPS}. and \ref{fig:lambda3-S6-CRPS} consider Scenario \textbf{S6} in \(\Lambda_1\), \(\Lambda_2\), and \(\Lambda_3\), respectively.

Across all scenarios, {\bf UnifDR} is implemented using different estimation methods: the isotonic estimator (Section \ref{iso-reg-sec}) for \textbf{S1} and \textbf{S2}, the trend filtering estimator (Section \ref{TredFil-sec}) for \textbf{S3} and \textbf{S4}, and the Dense ReLU Networks method (Section \ref{DRN-section}) for \textbf{S5} and \textbf{S6}. Our proposed framework {\bf UnifDR} outperforms competing approaches across all scenarios, evaluation regions, and sample sizes, regardless of the specific estimation method employed.  These findings highlight the versatility and robustness of {\bf UnifDR} in  adapting to diverse structural patterns within the data. Overall, performance trends remain consistent across \(\Lambda_1\), \(\Lambda_2\), and \(\Lambda_3\), with minor variations due to differences in the underlying data distributions. This extended analysis further reinforces the effectiveness of the proposed methods under varying evaluation conditions.

\begin{figure}[]
\begin{center}
\centerline{
\begin{tabular}{cc}   
\includegraphics[width=0.47\columnwidth]{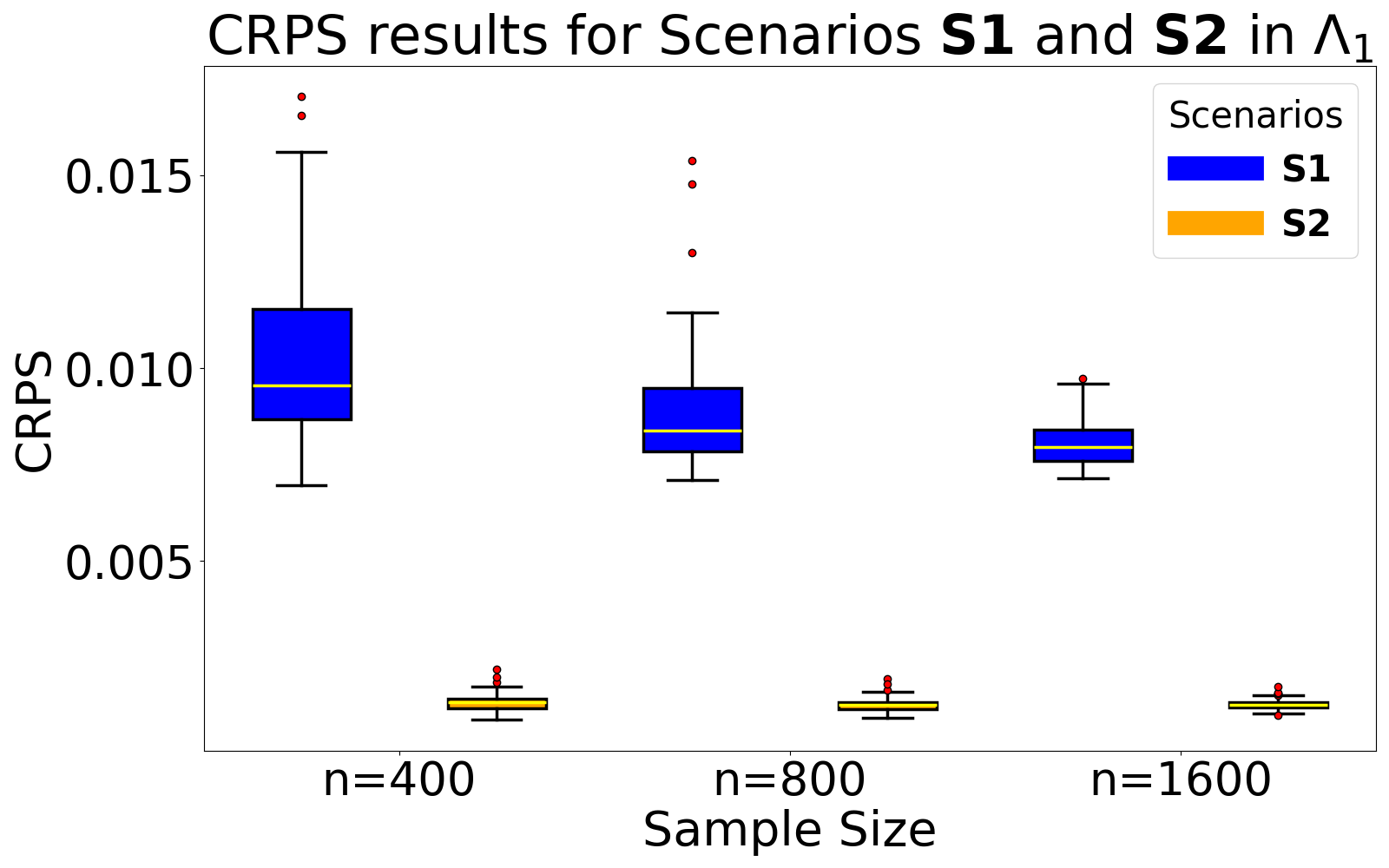} &
        \hspace{-1.1em}
\includegraphics[width=0.47\columnwidth]{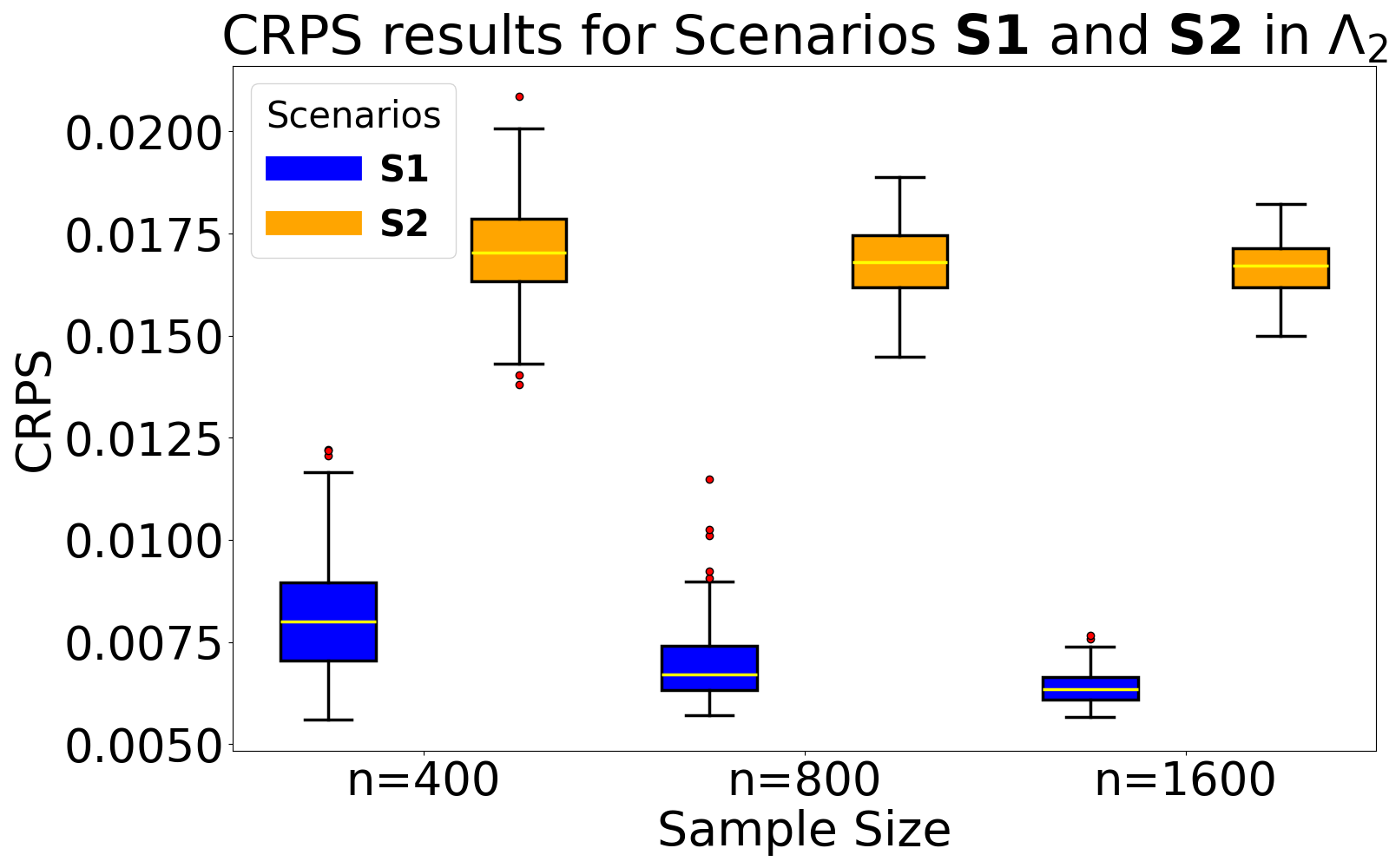} \\
    \multicolumn{2}{c}{\includegraphics[width=0.47\columnwidth]{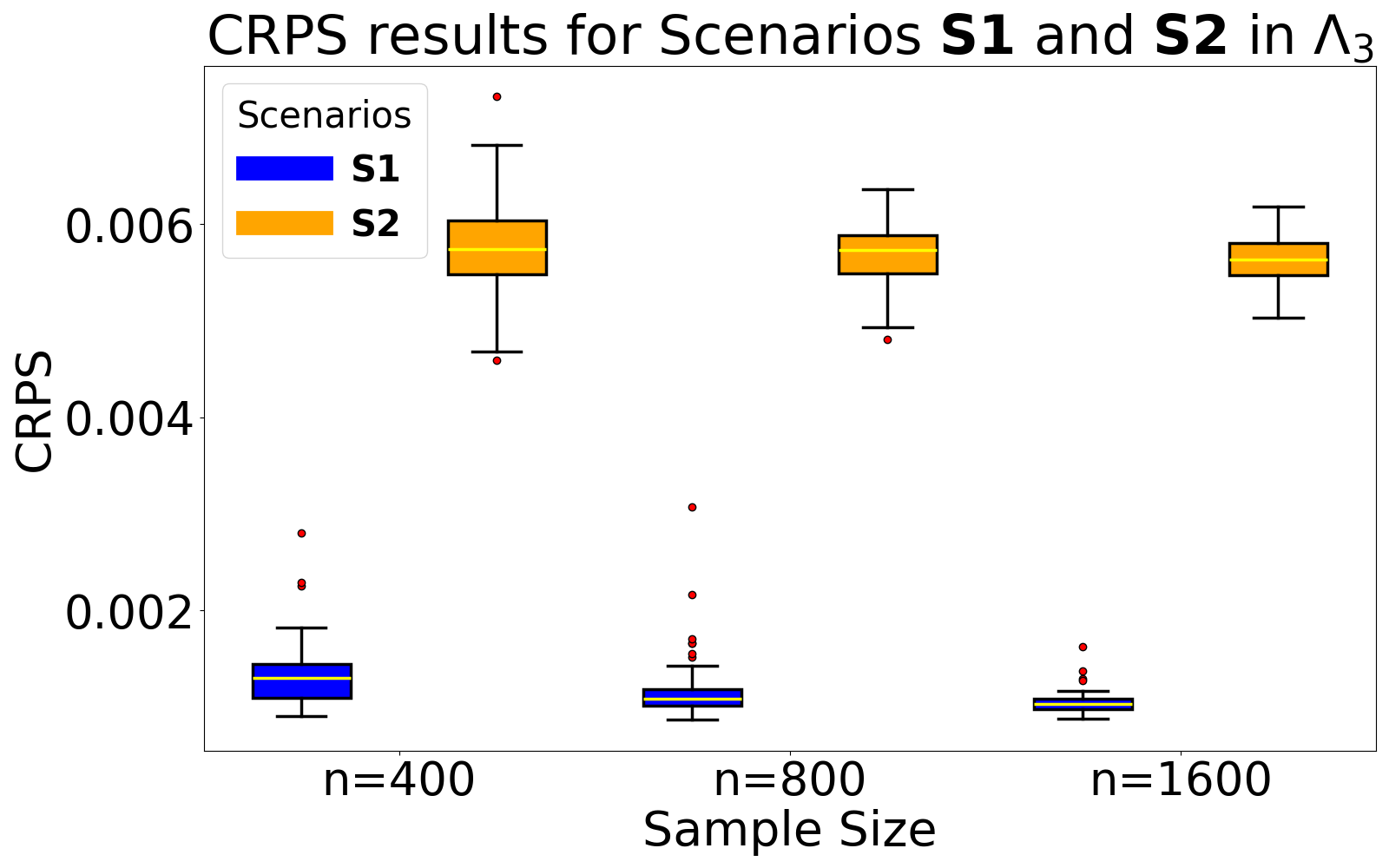}}
\end{tabular}}
\caption{Box plots for CRPS results in Scenarios \textbf{S1} and \textbf{S2}. The top row shows results for \(\Lambda_1\) (left) and \(\Lambda_2\) (right), while the bottom row displays results for \(\Lambda_3\).}
\label{fig:lambda1-lambda3-S1-S2}
\end{center}
\vskip -0.2in
\end{figure}

\begin{figure}[]
\begin{center}
\centerline{
\includegraphics[width=0.49\textwidth]{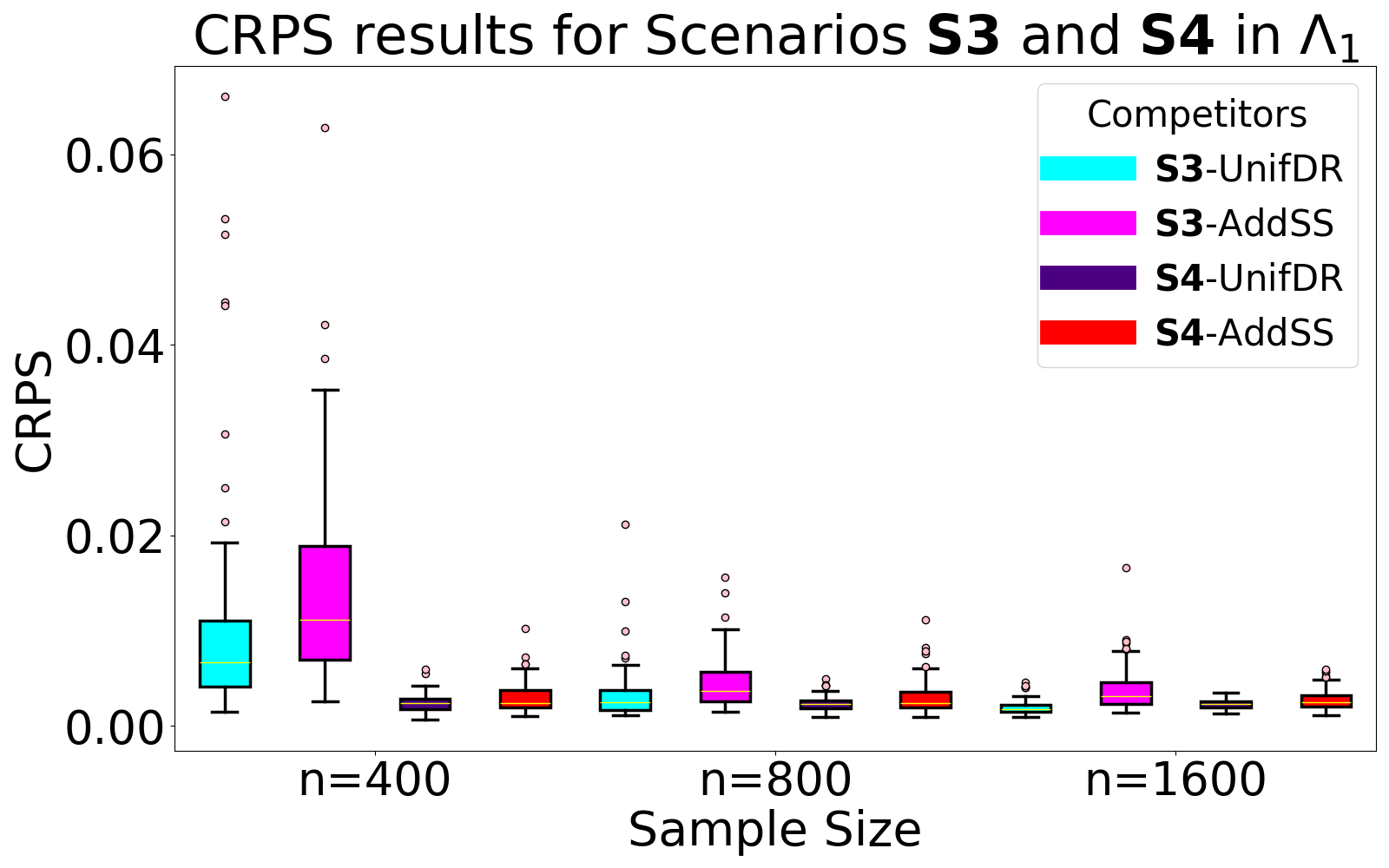}
\hspace{-0.5em}
\includegraphics[width=0.49\textwidth]{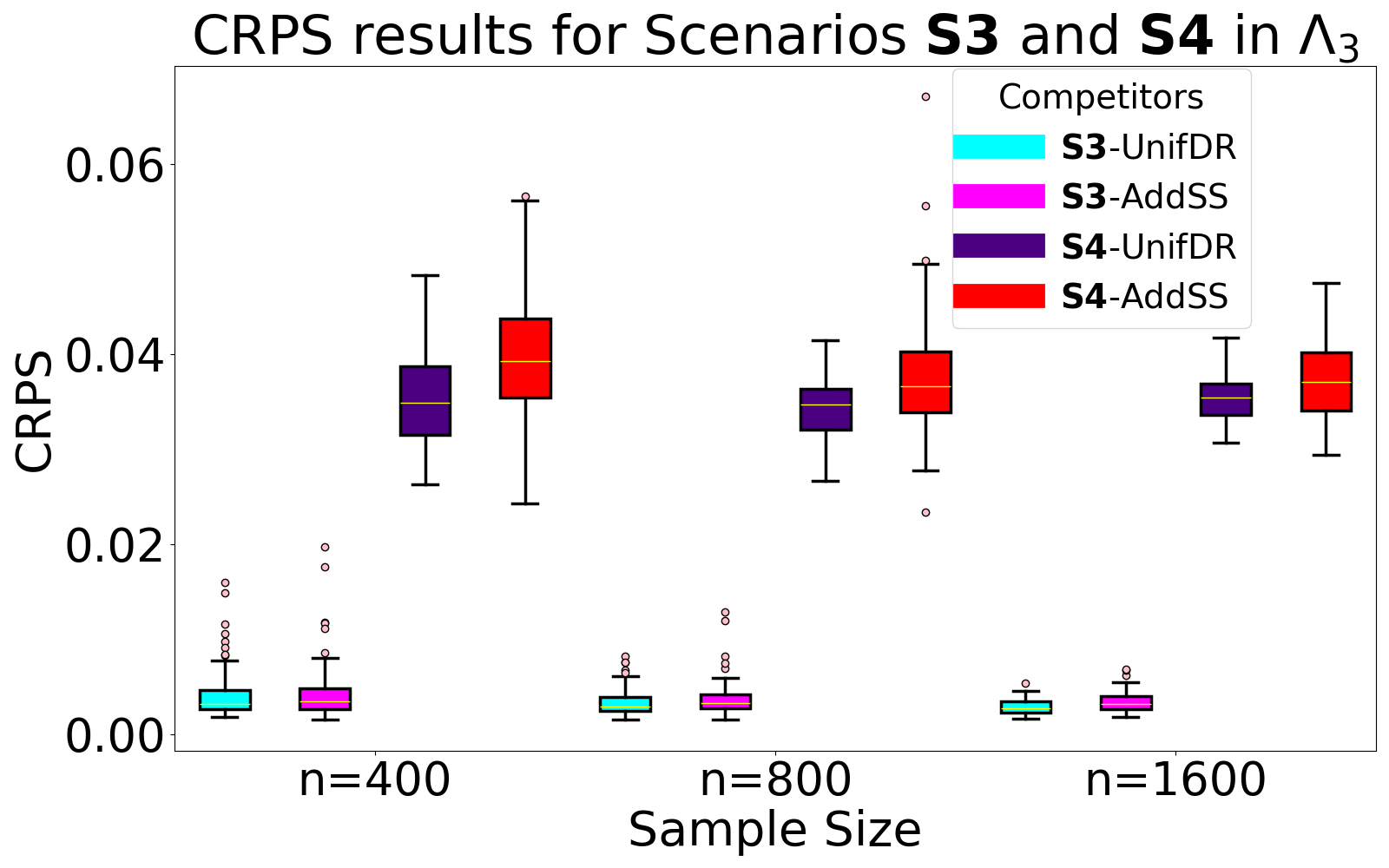}
}
\caption{Box plots for simulation results of {\bf S3-S4} for the CRPS metric. The row shows results for \(\Lambda_1\) (left) and \(\Lambda_2\) (right).}
\label{fig:lambda1-lambda3-S3-S4}
\end{center}
\vskip -0.2in
\end{figure}

\begin{figure}[]
\begin{center}
\centerline{
\includegraphics[width=0.5\textwidth]{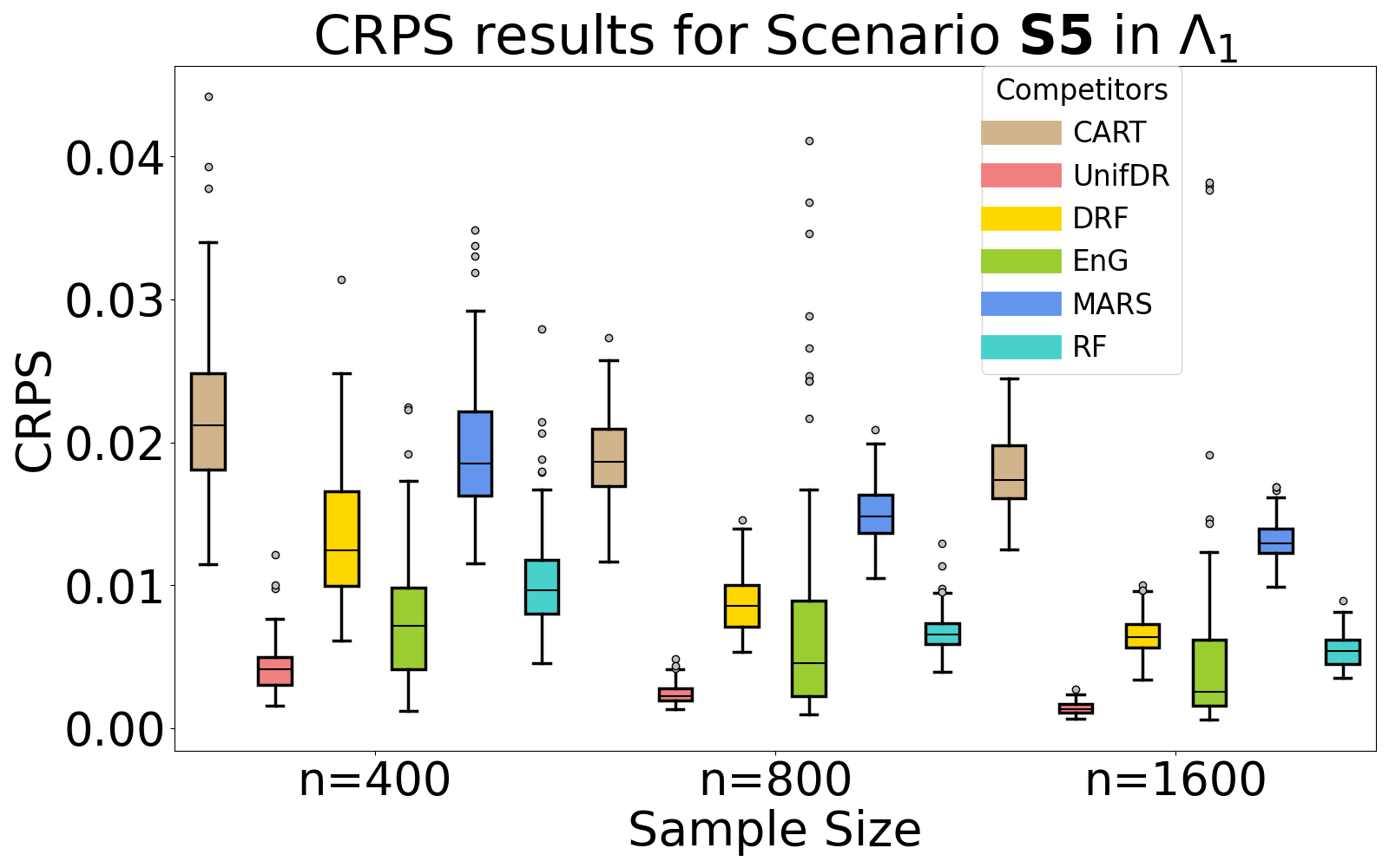}
\hspace{-0.5em}
\includegraphics[width=0.5\textwidth]{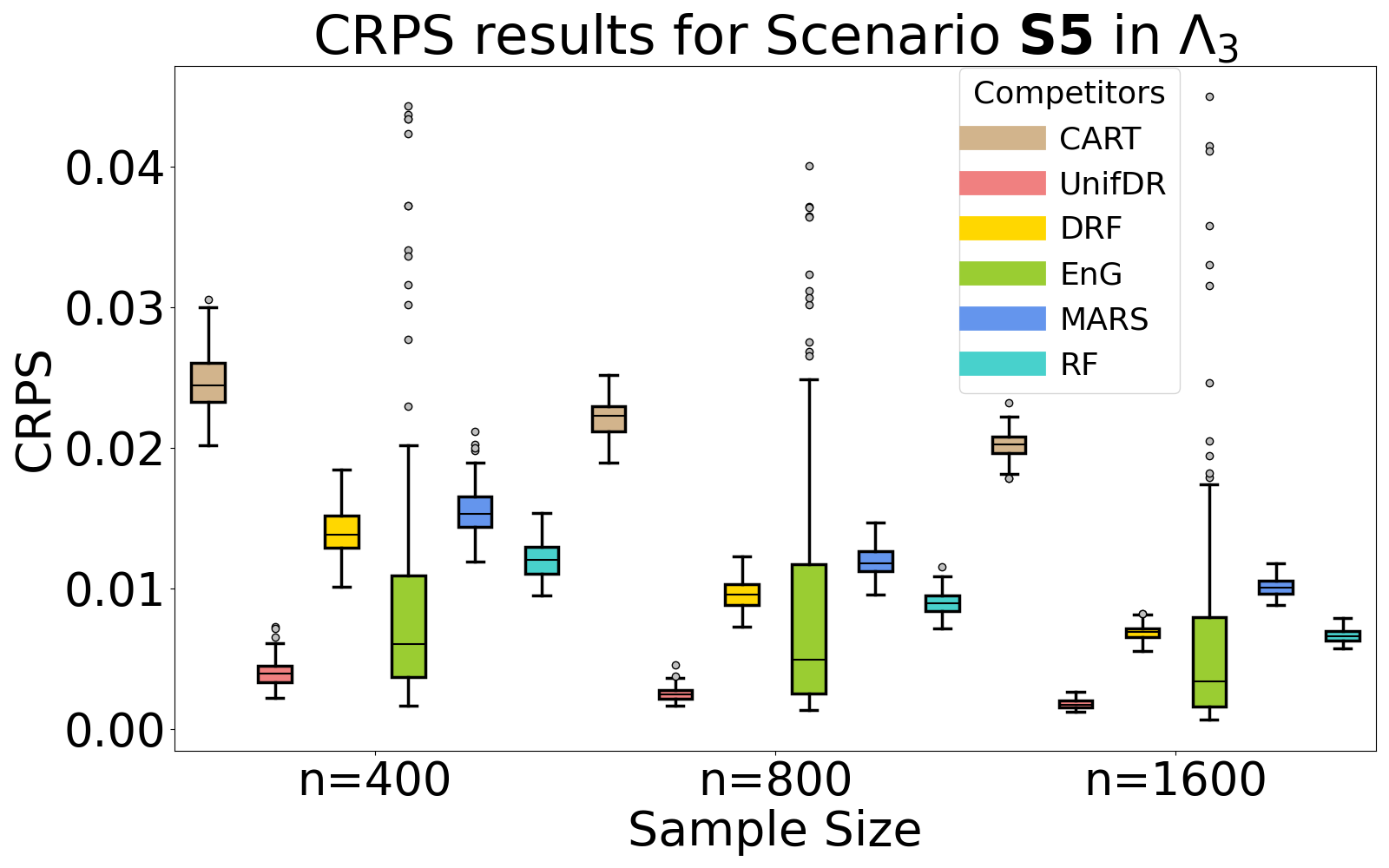}
}
\caption{Box plots for simulation results of {\bf S5} for the CRPS metric. The row shows results for \(\Lambda_1\) (left) and \(\Lambda_2\) (right).}
\label{fig:lambda1-lambda3-S5}
\end{center}
\vskip -0.2in
\end{figure}

\begin{figure}[]
\centering
\begin{tabular}{cc}
    \includegraphics[width=0.49\columnwidth]{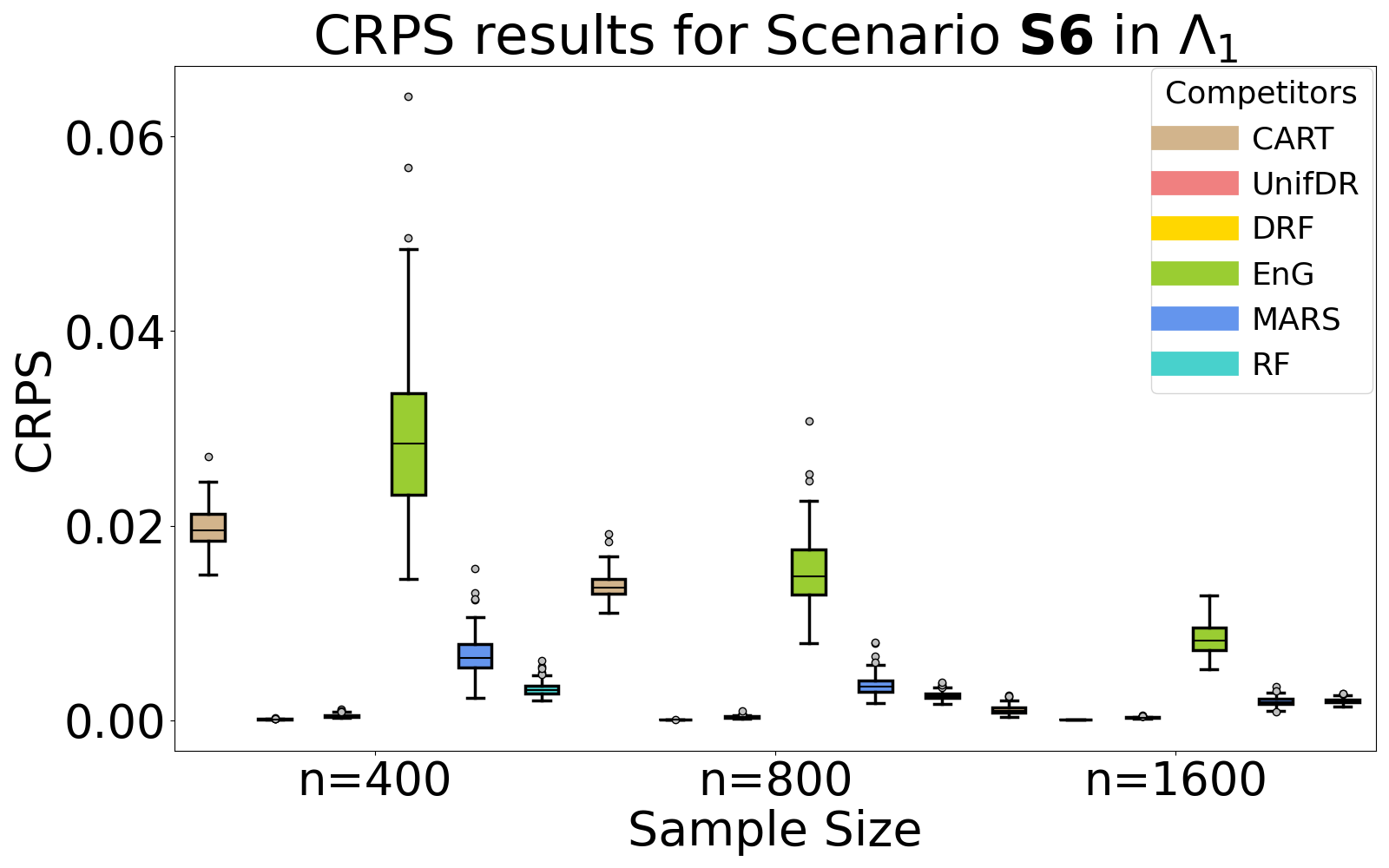} &
    \hspace{-1.1em}
    \includegraphics[width=0.49\columnwidth]{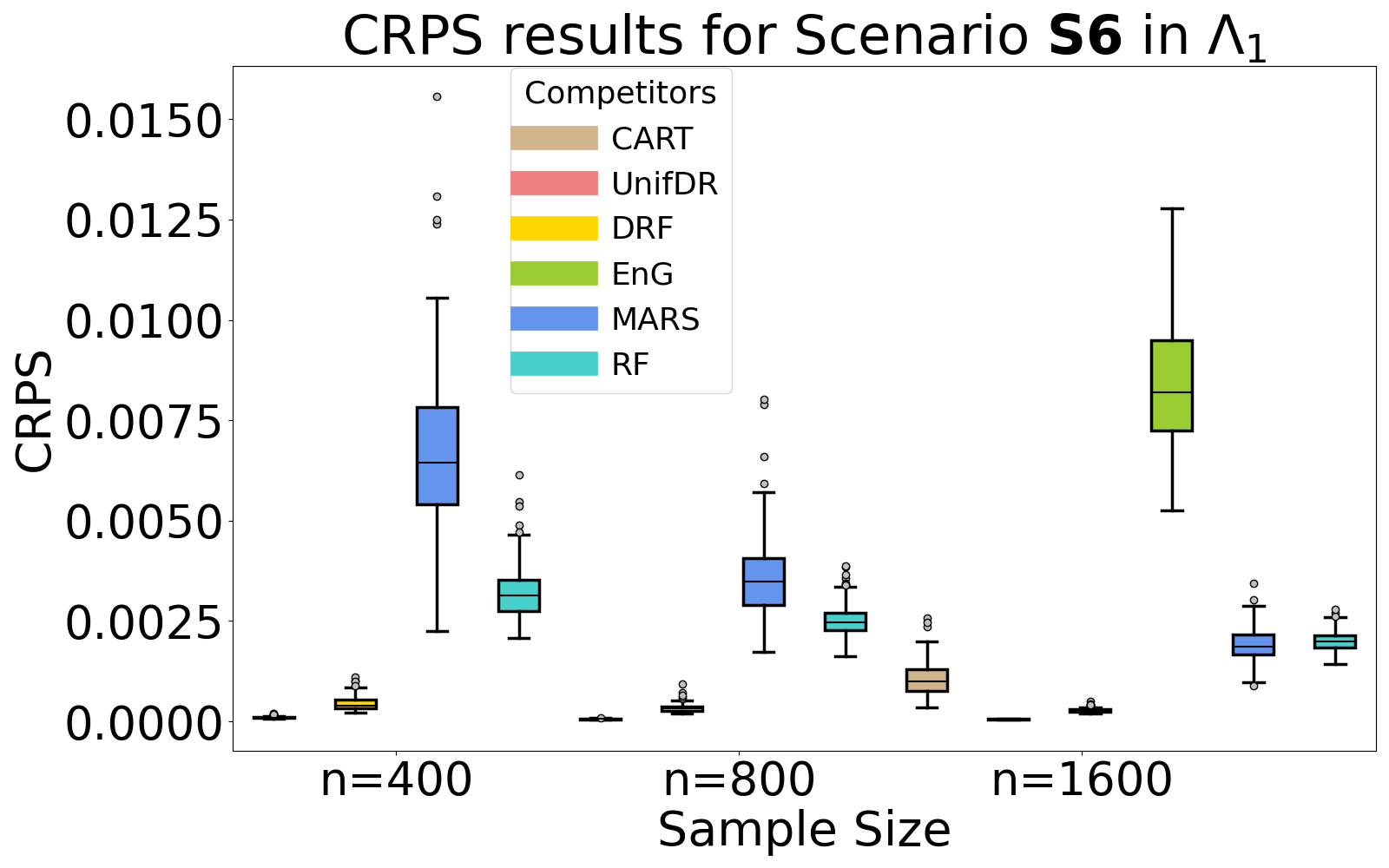} \\
    \multicolumn{2}{c}{\includegraphics[width=0.49\columnwidth]{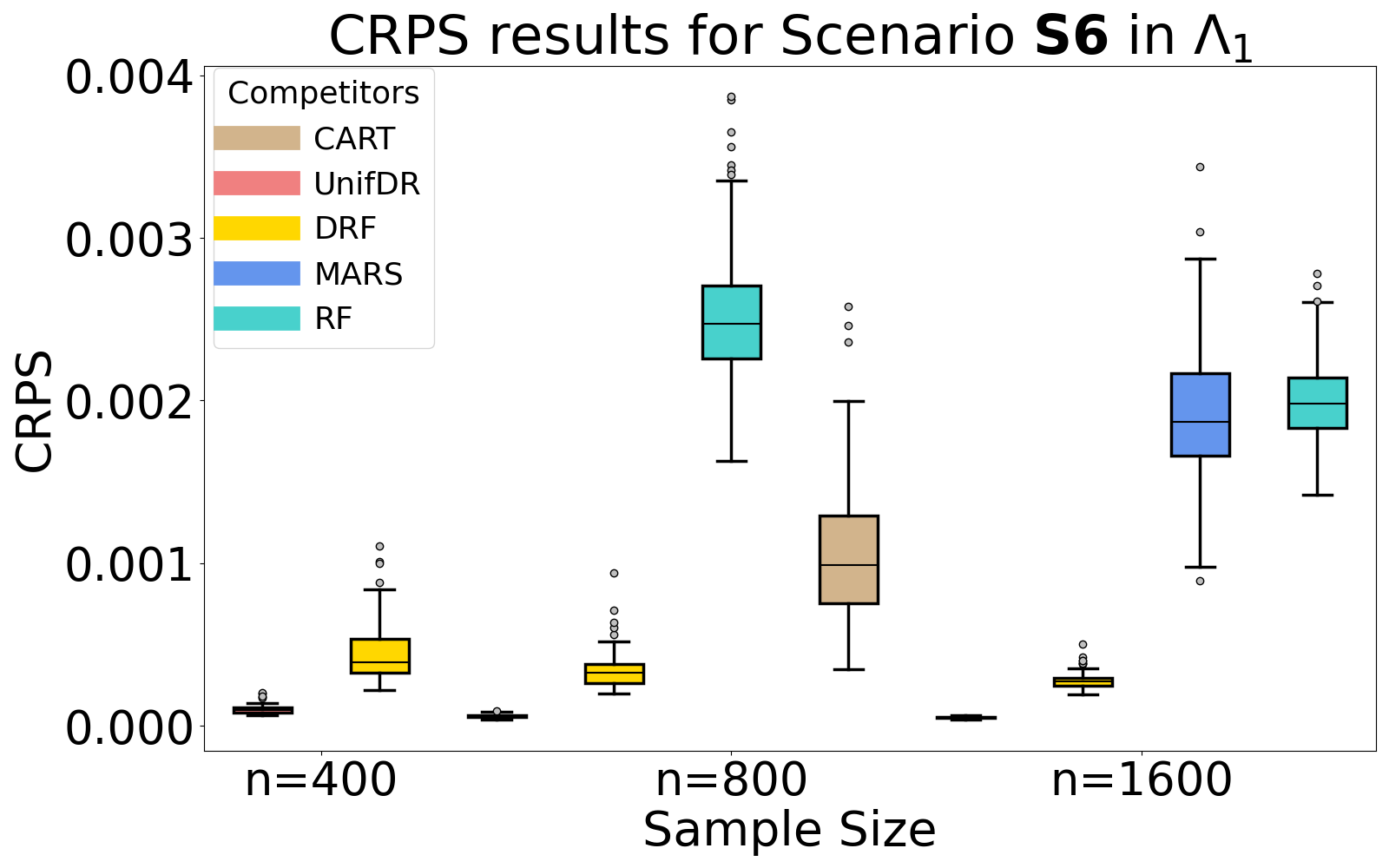}
    
    \includegraphics[width=0.49\columnwidth]{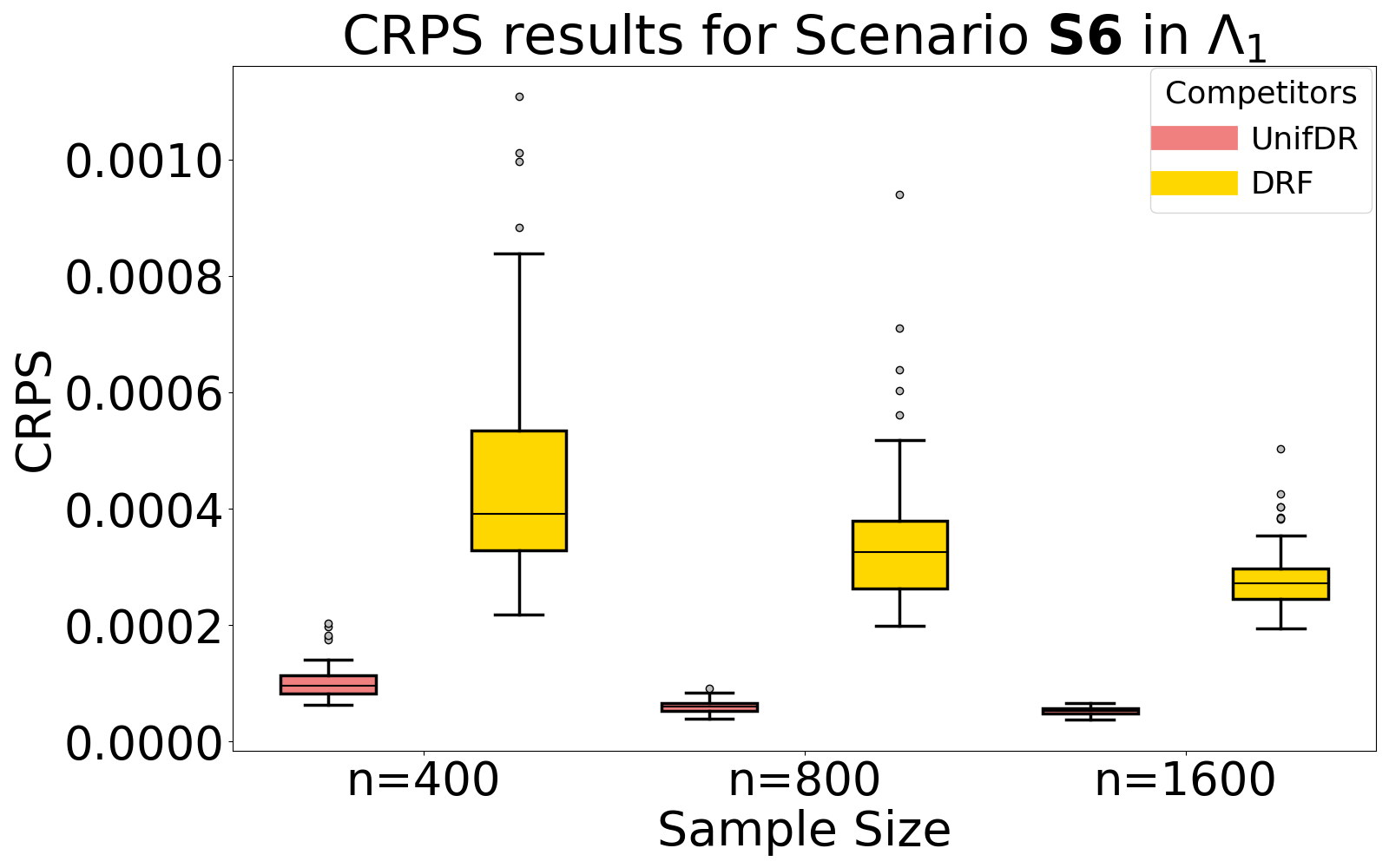}}
\end{tabular}
\caption{Box plots for simulation results of {\bf S6} for the CRPS metric for the set \(\Lambda_1\). The top row shows results for the all the competitors (left) and, competitors with median below $0.01$ (right). The bottom row displays results for competitors with median below \(0.0025\) (left), and best two competitors (right).}
\label{fig:lambda1-S6-CRPS}
\vskip -0.2in
\end{figure}

\begin{figure}[]
\centering
\begin{tabular}{cc}
    \includegraphics[width=0.47\columnwidth]{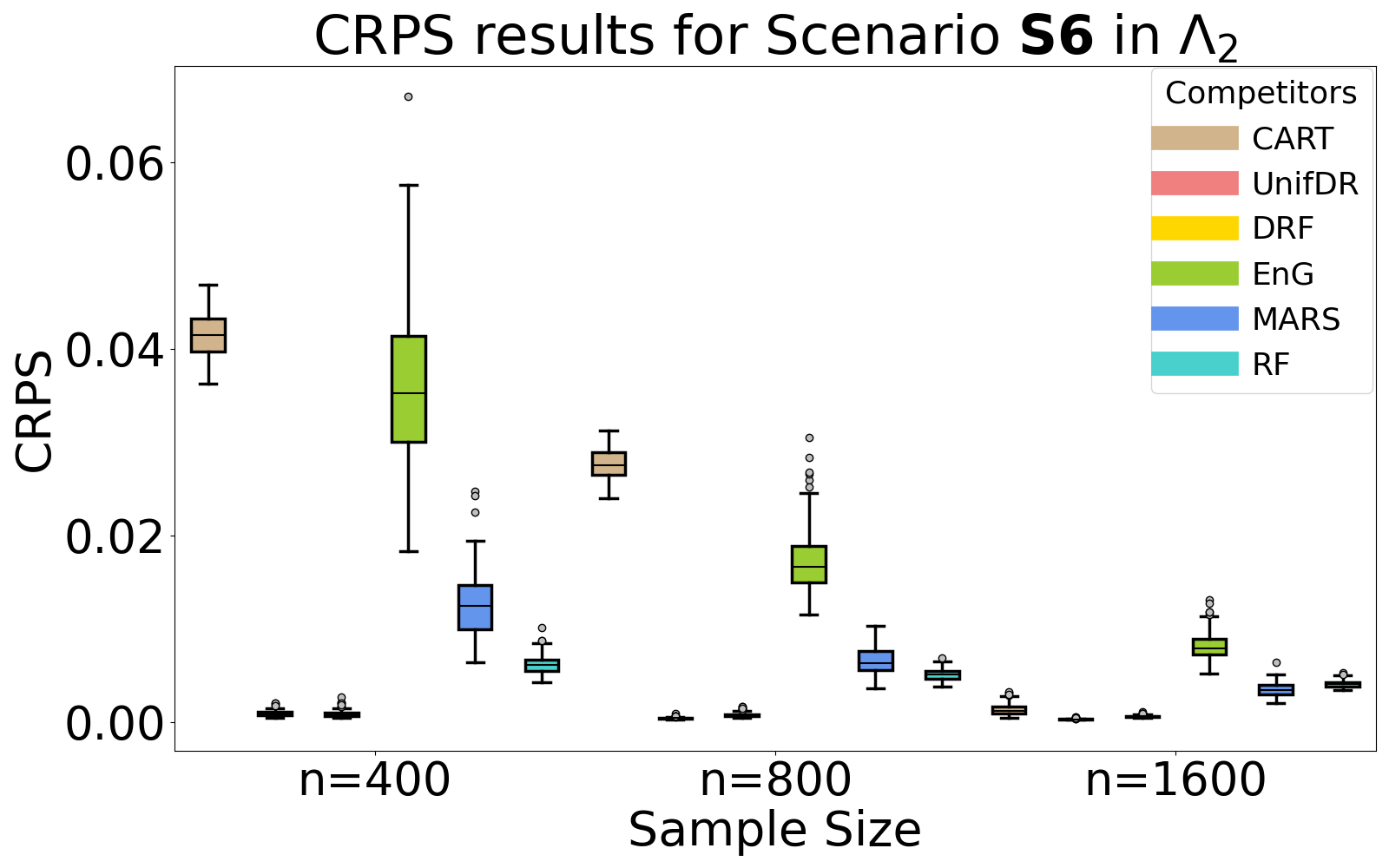} &
    \includegraphics[width=0.47\columnwidth]{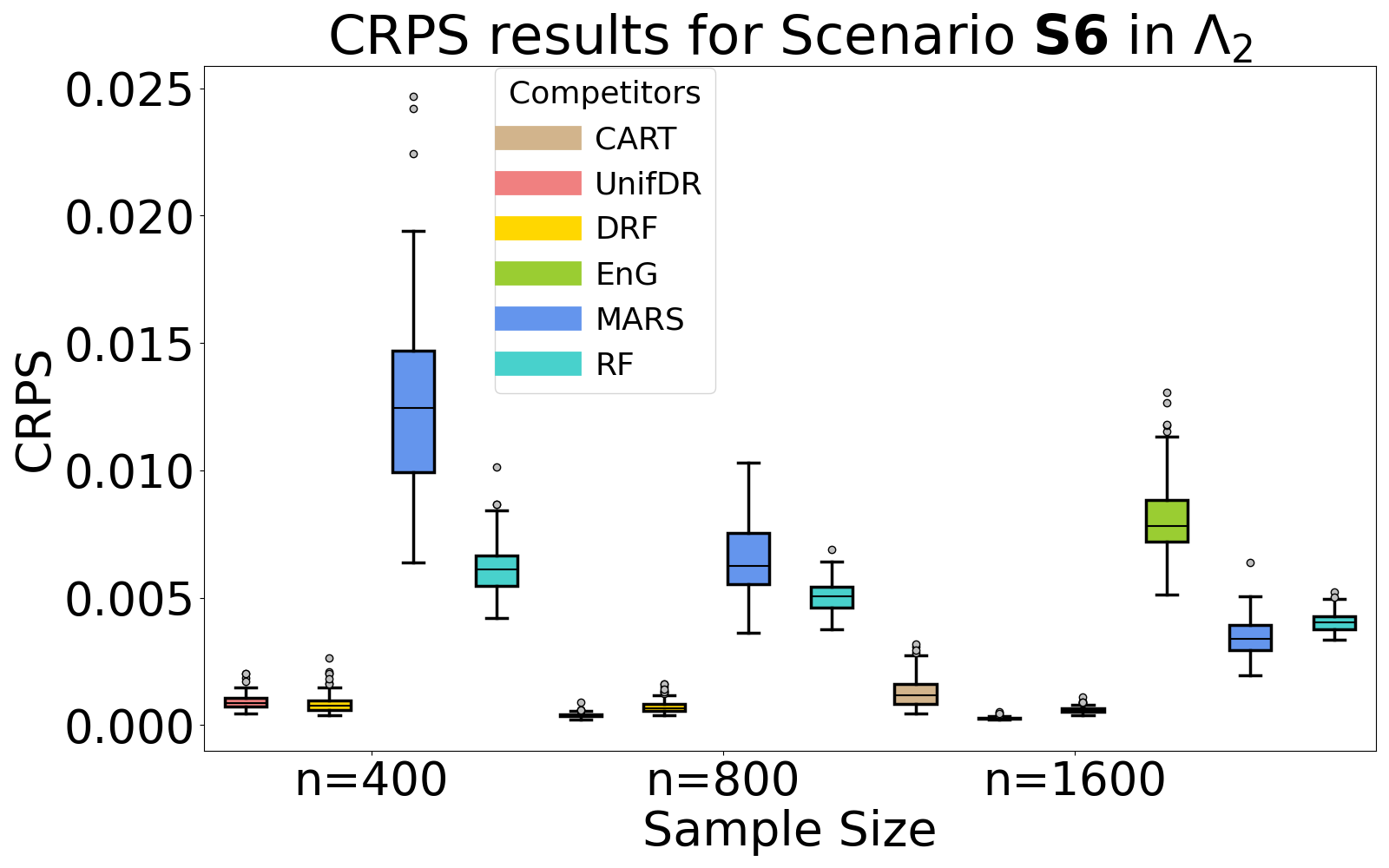} \\
    \multicolumn{2}{c}{\includegraphics[width=0.47\columnwidth]{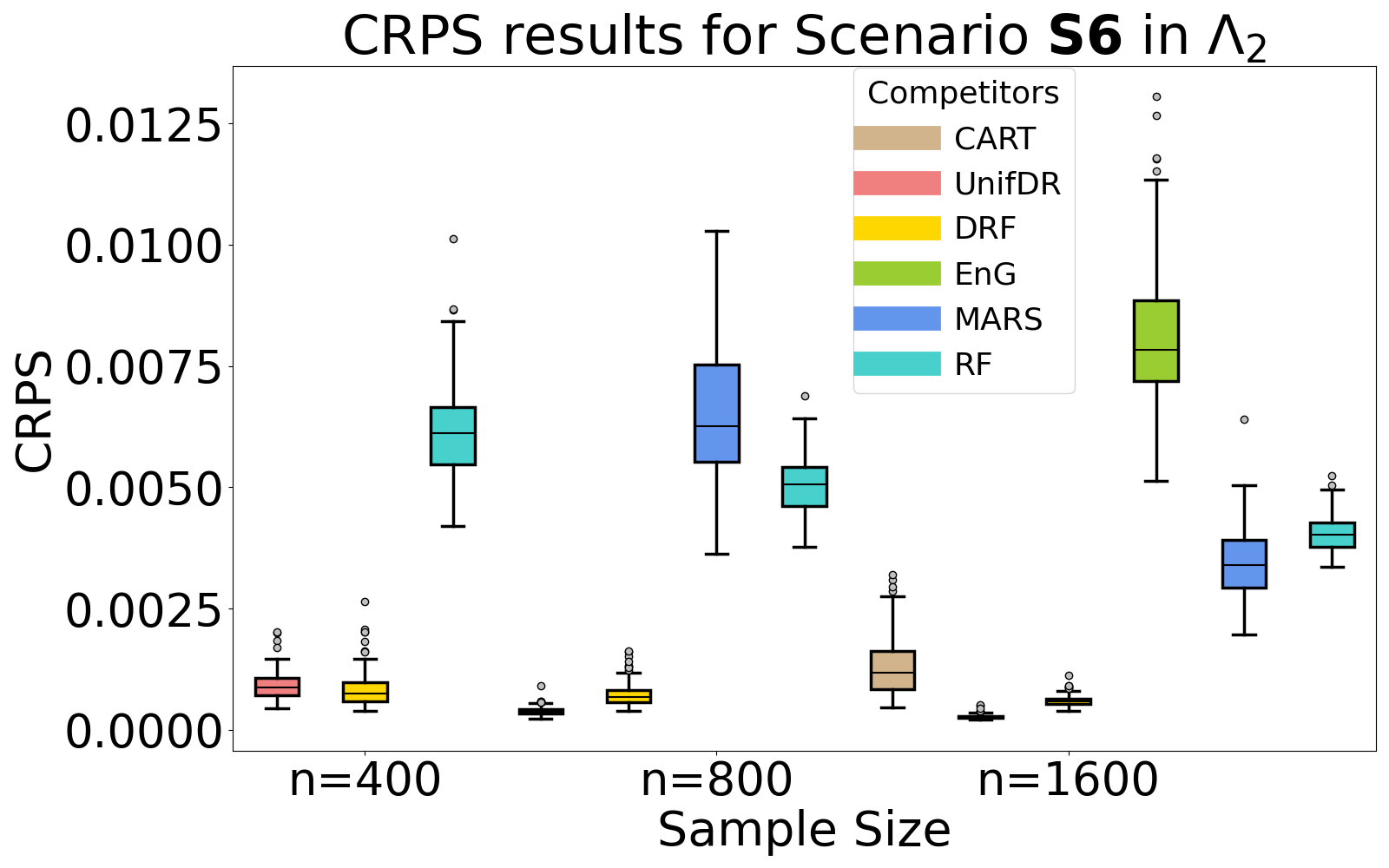}
    
    \includegraphics[width=0.47\columnwidth]{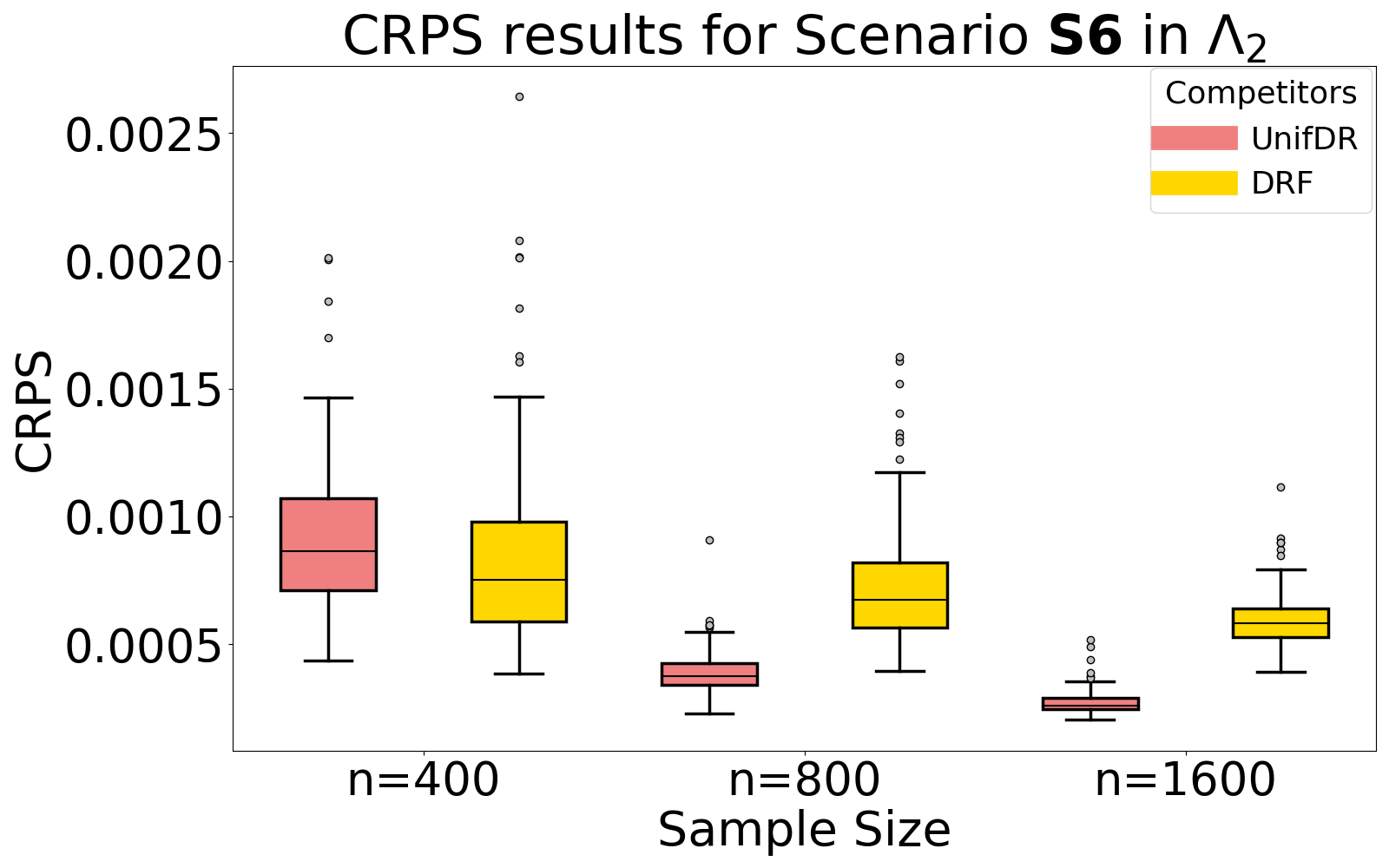}}
\end{tabular}
\caption{Box plots for CRPS results for {\bf S6} in $\Lambda_2$. The top row shows results for the all the competitors (left) and, competitors with median below $0.02$ (right). The bottom row displays results for competitors with median below \(0.01\) (left), and best two competitors (right).}
\label{fig:lambda2-S6-CRPS}
\end{figure}

\begin{figure}[]
\begin{center}
\centerline{
\includegraphics[width=0.5\textwidth]{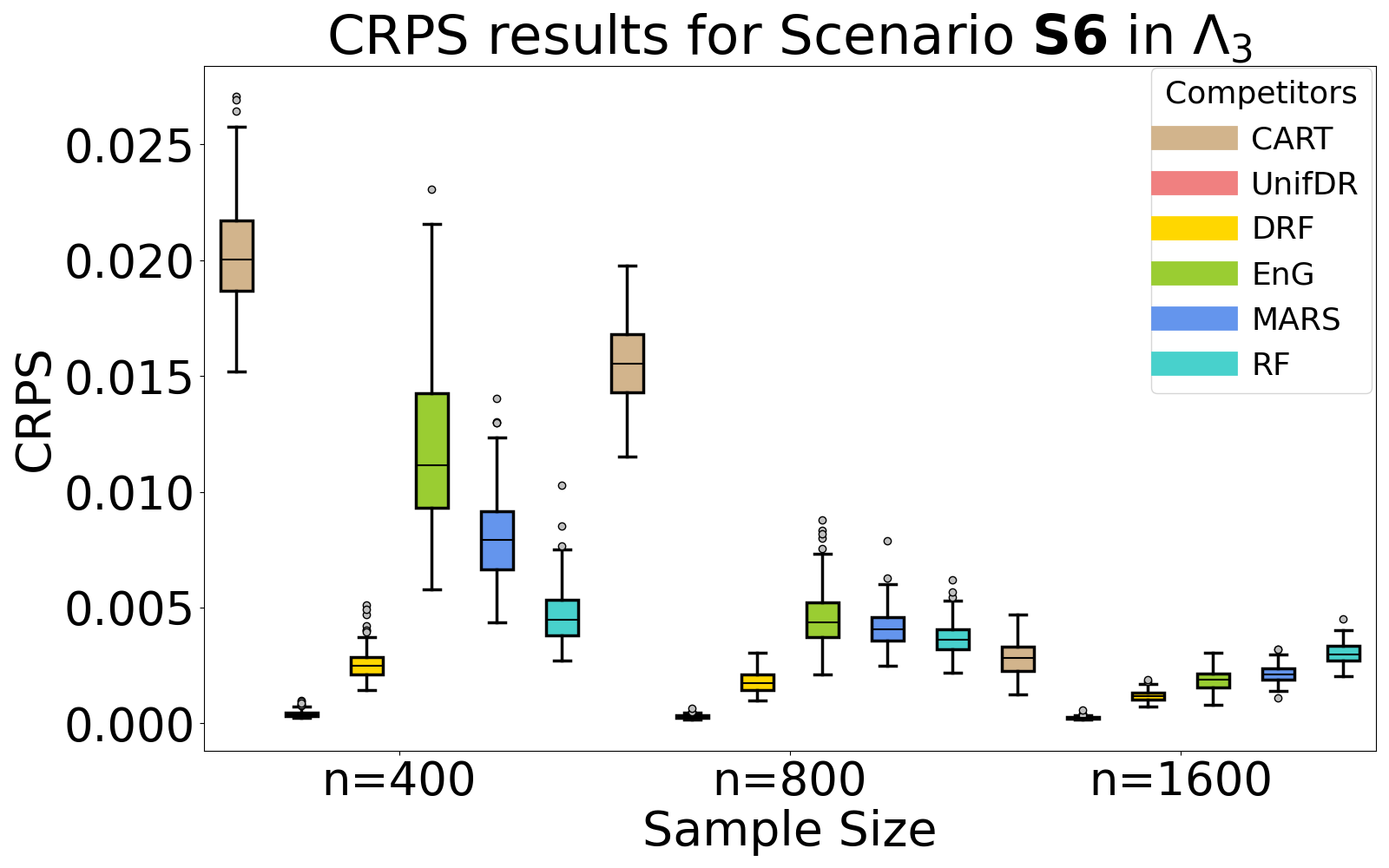}
\hspace{-0.5em}
\includegraphics[width=0.5\textwidth]{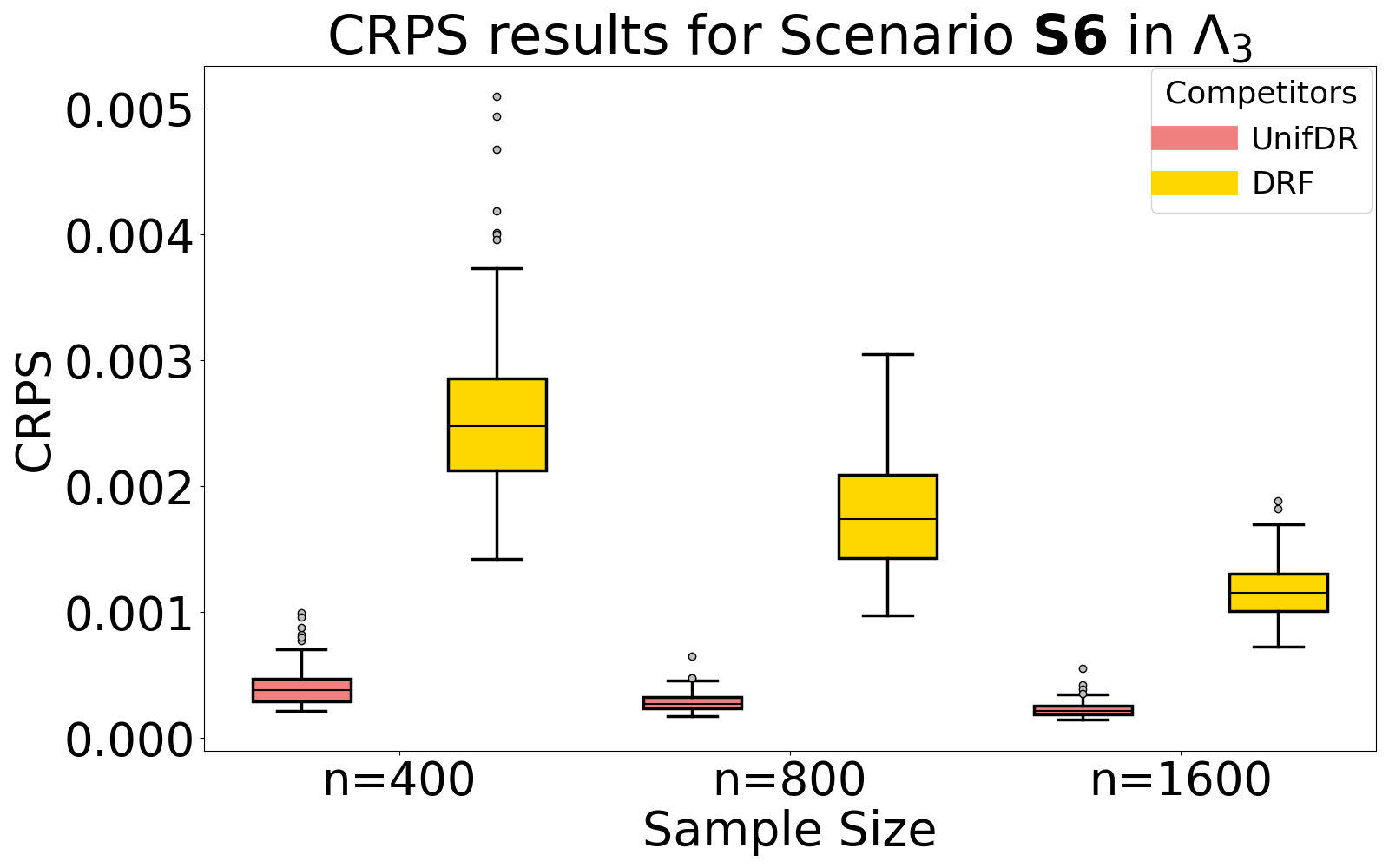}
}
\caption{Box plots for simulation results of {\bf S6} for the CRPS metric using evaluation set $\Lambda_3$. The left plot corresponds to all competitors performance, while the right plot corresponds to best two competitors.}
\label{fig:lambda3-S6-CRPS}
\end{center}
\vskip -0.2in
\end{figure}

\subsection{Additional Results for Maximum Squared Difference (MSD) Metric}
\label{msd-results}

This appendix extends the results presented in Section \ref{simu-data} and Appendix \ref{lambda1-lambda3} by providing evaluations of the Maximum Squared Difference (MSD) metric across all scenarios for the evaluation sets \(\Lambda_1\), \(\Lambda_2\), and \(\Lambda_3\). The MSD metric measures the worst-case discrepancy between the estimated and true cumulative distribution functions (CDFs), offering a stringent assessment of model accuracy and robustness.

Figure \ref{fig:composite} displays box plots of the MSD results for Scenarios \textbf{S1} and \textbf{S2} across all evaluation sets. As outlined in Section \ref{simu-data}, {\bf UnifDR} employs isotonic regression for these scenarios, which lacks direct competitors. The results indicate a decreasing trend in MSD values as the sample size increases, demonstrating the consistency and improved accuracy of isotonic regression with larger datasets. Moreover, variations in MSD across \(\Lambda_1\), \(\Lambda_2\), and \(\Lambda_3\) reflect natural differences in the underlying data distributions.

Figure \ref{fig:composite-S3-S} presents MSD results for Scenarios \textbf{S3} and \textbf{S4}, where {\bf UnifDR} employs trend filtering and is compared against the additive smoothing splines (AddSS) method. Across all evaluation sets and sample sizes, {\bf UnifDR} consistently outperforms AddSS, demonstrating its ability to adapt to complex structural variations. Notably, trend filtering exhibits particularly strong performance in regions with sparse or heavy-tailed data distributions.

Figures \ref{fig:composite-S5} to \ref{fig:lambda3-S6-MSD} summarize MSD results for Scenarios \textbf{S5} and \textbf{S6}, where {\bf UnifDR} is implemented via Dense ReLU Networks (Section \ref{DRN-section}). The results further confirm the superiority of {\bf UnifDR} over all competing methods, maintaining its advantage across different sample sizes and evaluation sets. As observed with the CRPS metric, {\bf UnifDR} effectively captures diverse structural patterns and remains robust in challenging scenarios involving data sparsity and heavy tails.

The inclusion of MSD results provides a comprehensive performance evaluation of {\bf UnifDR}, reinforcing its effectiveness and reliability across various experimental conditions.

\begin{figure}[]
\centering
\begin{tabular}{cc}   
\includegraphics[width=0.47\columnwidth]{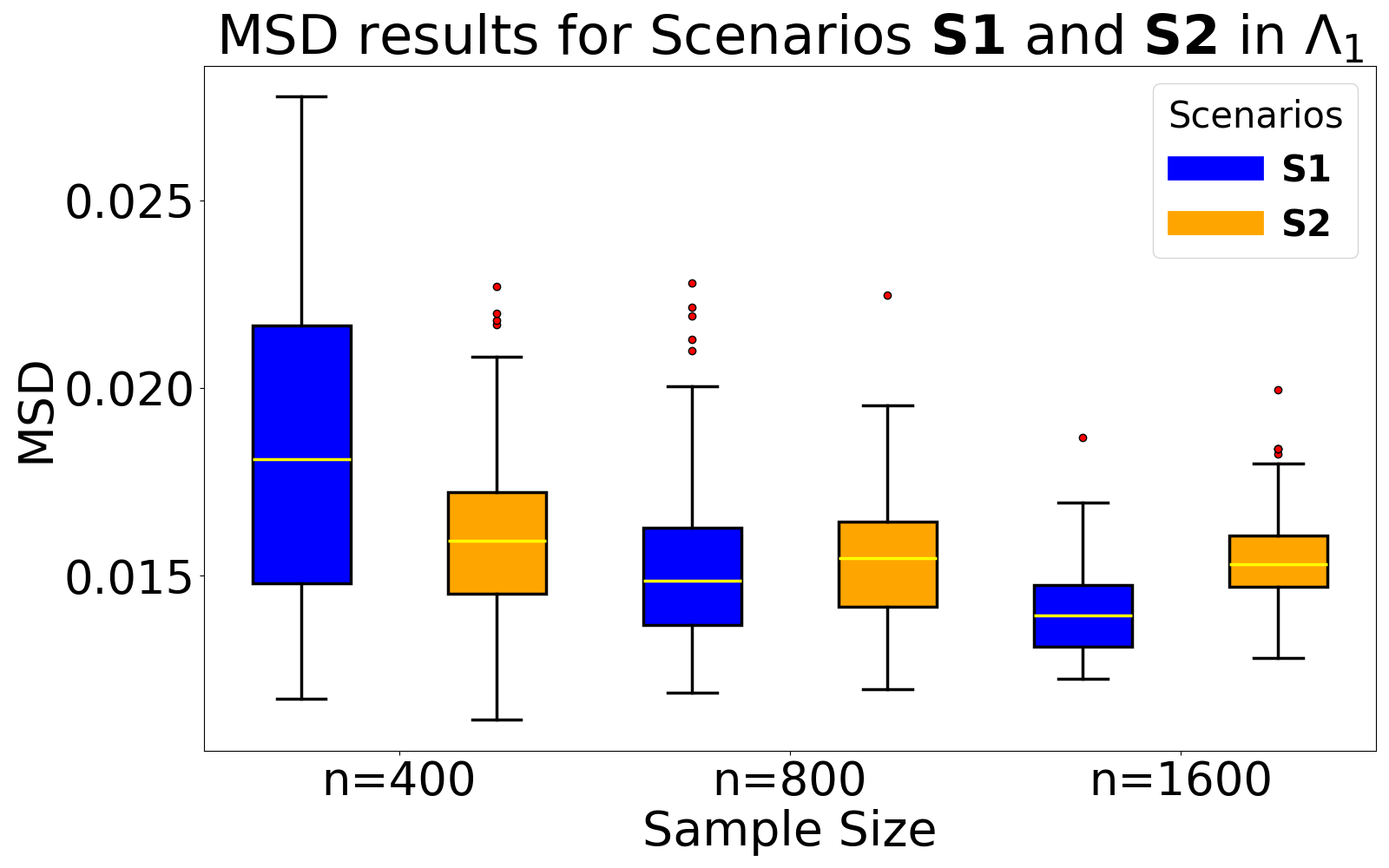} &
        \hspace{-1.1em}
\includegraphics[width=0.47\columnwidth]{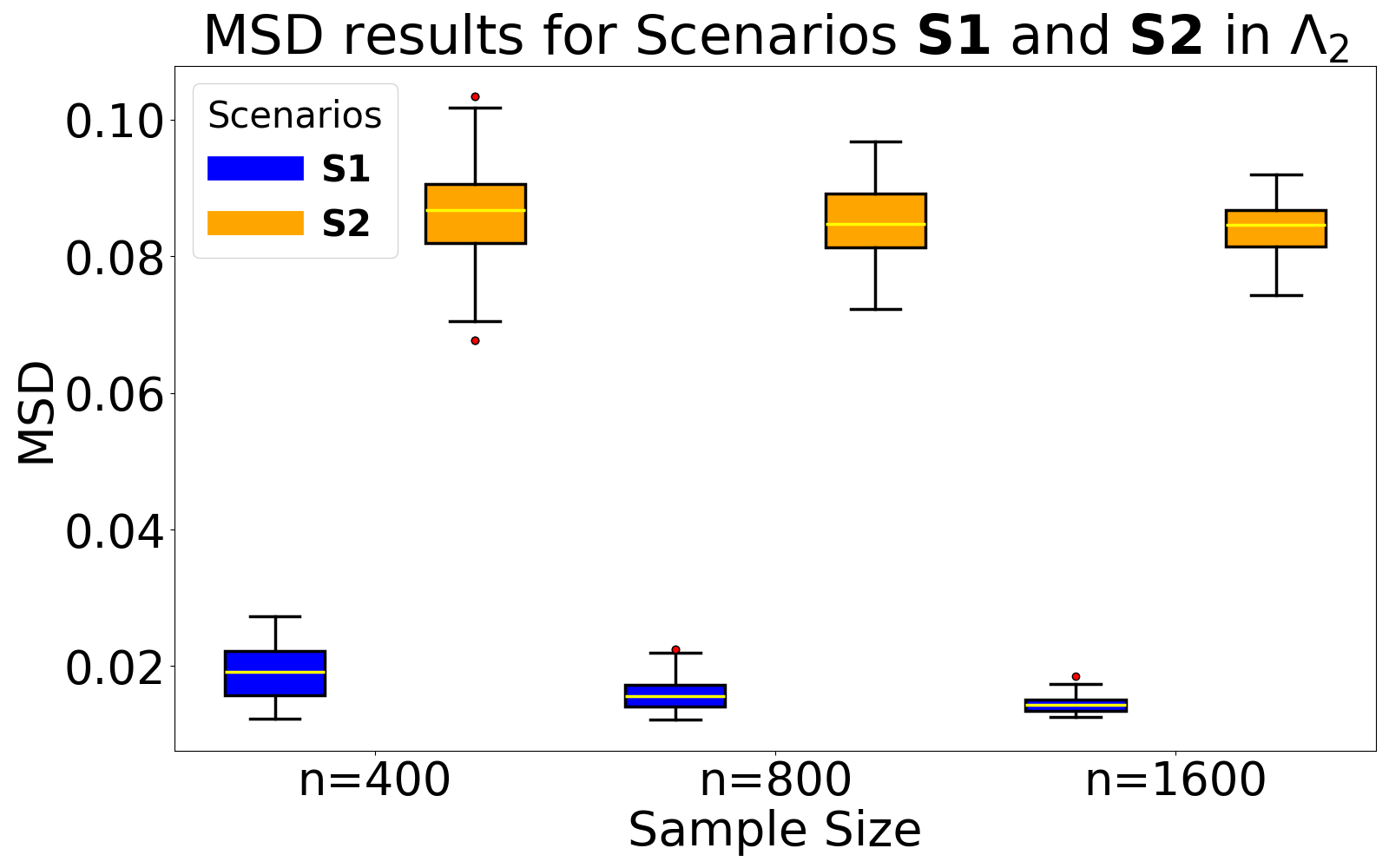} \\
    \multicolumn{2}{c}{\includegraphics[width=0.47\columnwidth]{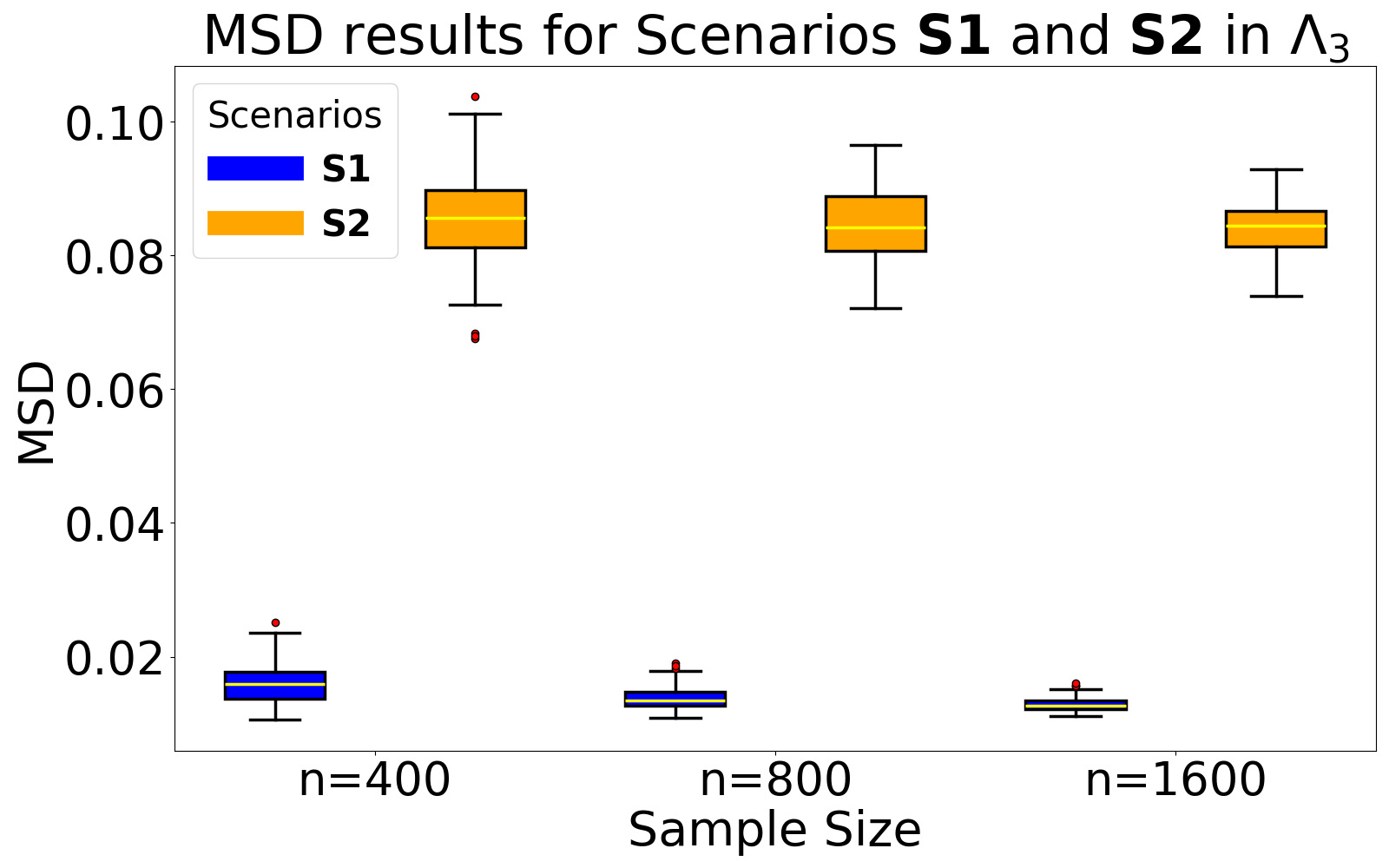}}
\end{tabular}
\caption{Box plots for MSD results in Scenarios \textbf{S1} and \textbf{S2}. The top row shows results for \(\Lambda_1\) (left) and \(\Lambda_2\) (right), while the bottom row displays results for \(\Lambda_3\).}
\label{fig:composite}
\end{figure}

\begin{figure}[]
\centering
\begin{tabular}{cc}    \includegraphics[width=0.45\columnwidth]{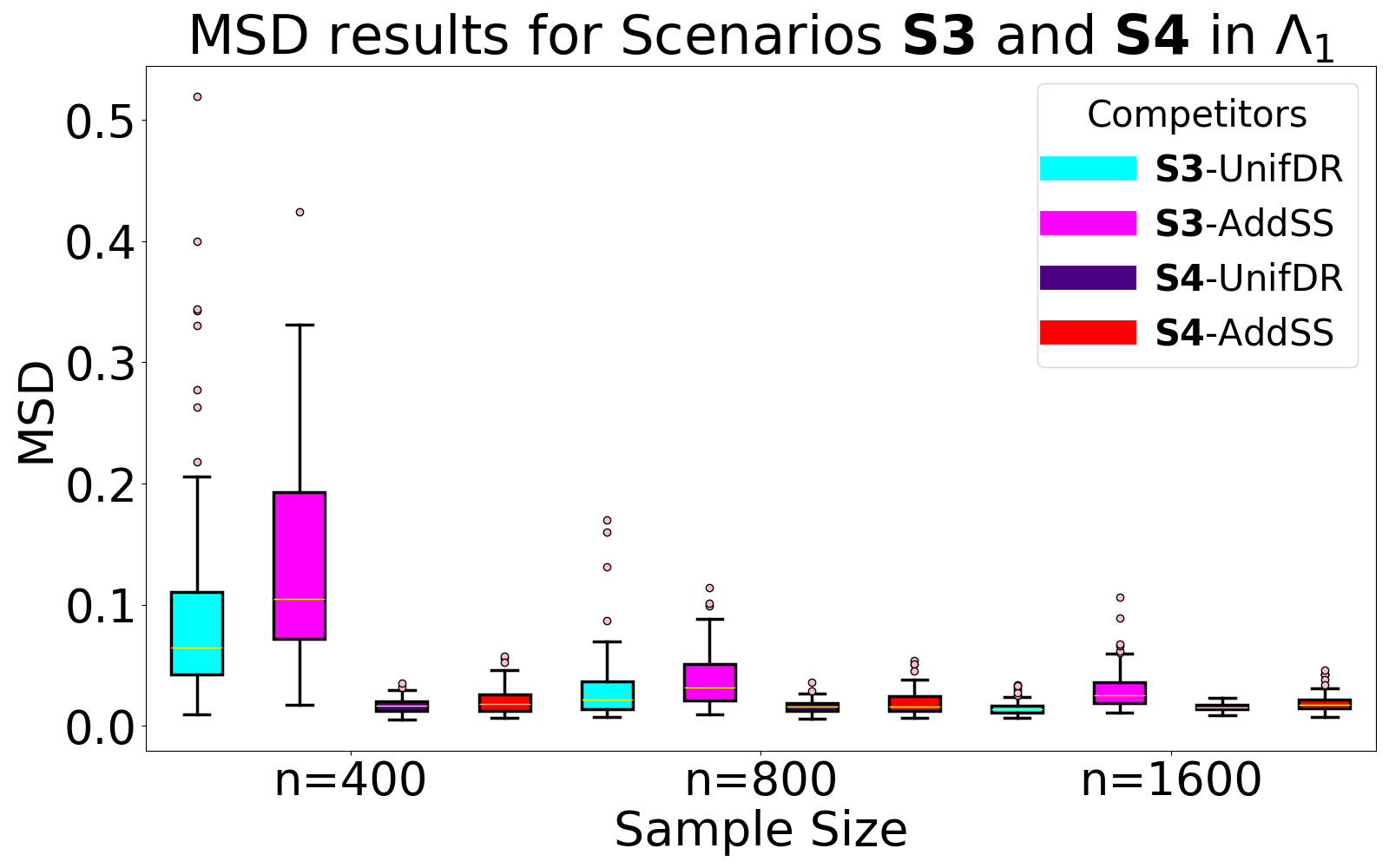} &
        \hspace{-1.1em}
\includegraphics[width=0.45\columnwidth]{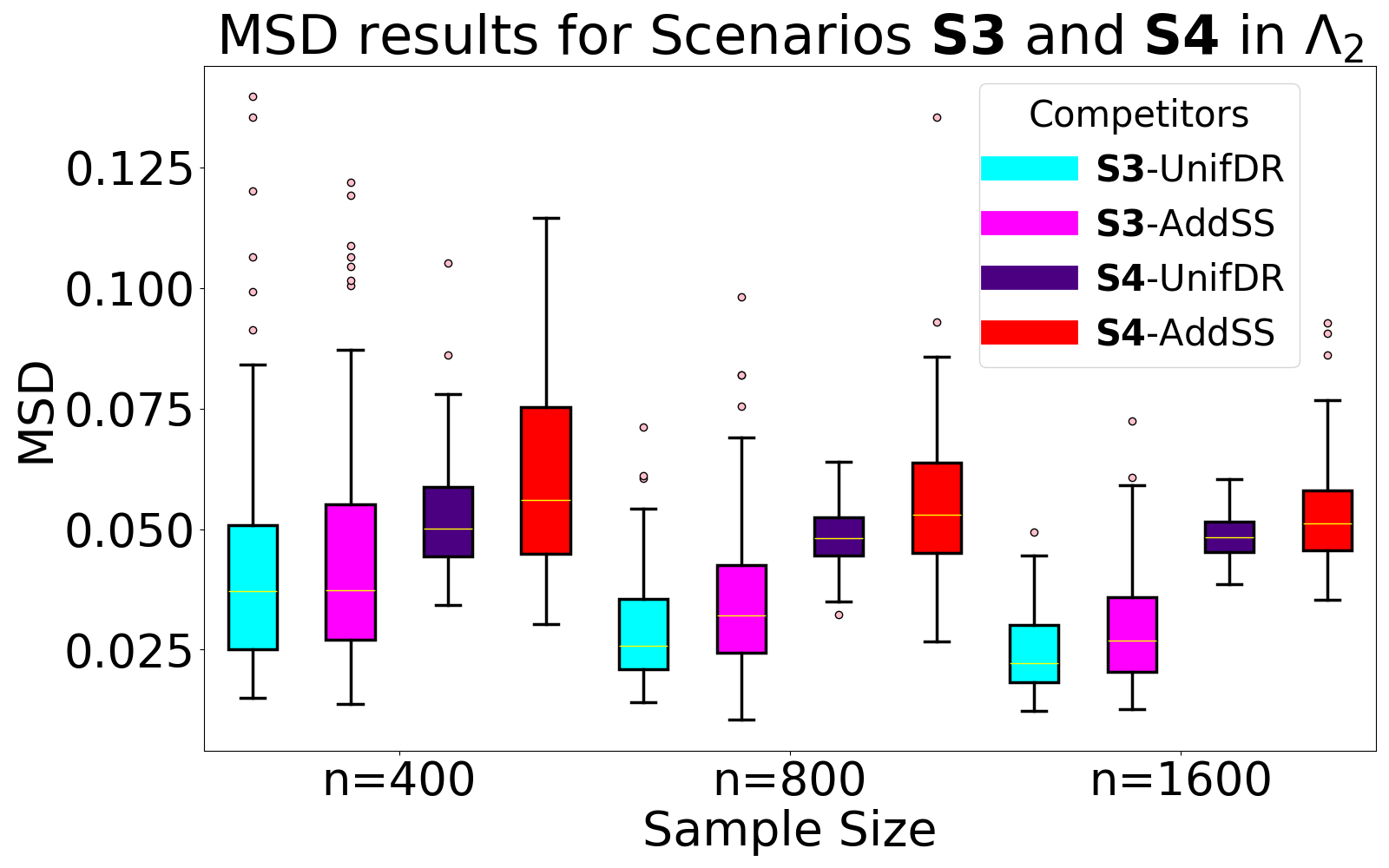} \\
    \multicolumn{2}{c}{\includegraphics[width=0.45\columnwidth]{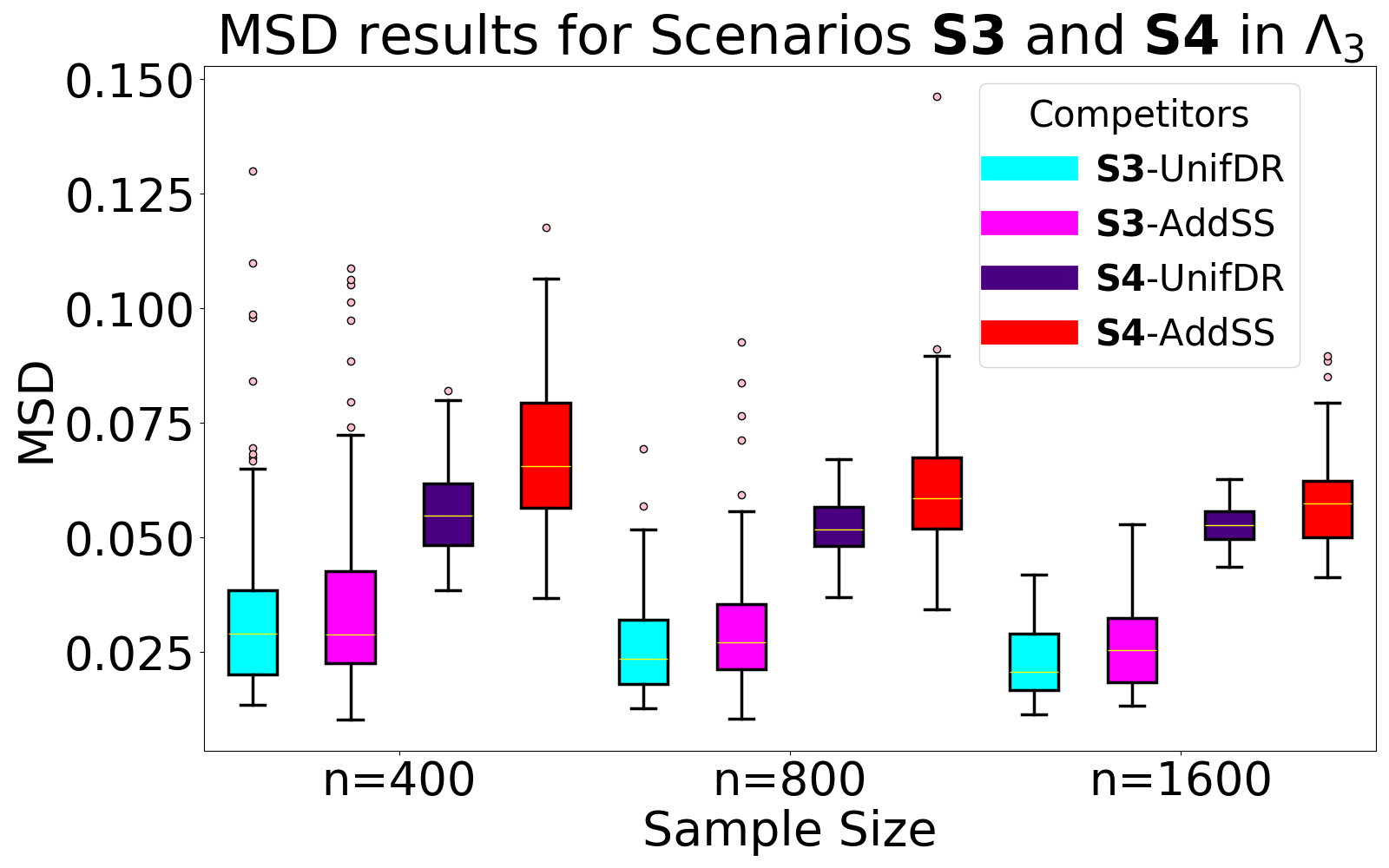}}
\end{tabular}
\caption{Box plots for MSD in \textbf{S3} and \textbf{S4}. The top row shows results for \(\Lambda_1\) (left) and \(\Lambda_2\) (right), while the bottom row displays results for \(\Lambda_3\).}
\label{fig:composite-S3-S}
\vskip -0.2in
\end{figure}

\begin{figure}[]
\centering
\begin{tabular}{cc}
    \includegraphics[width=0.47\columnwidth]{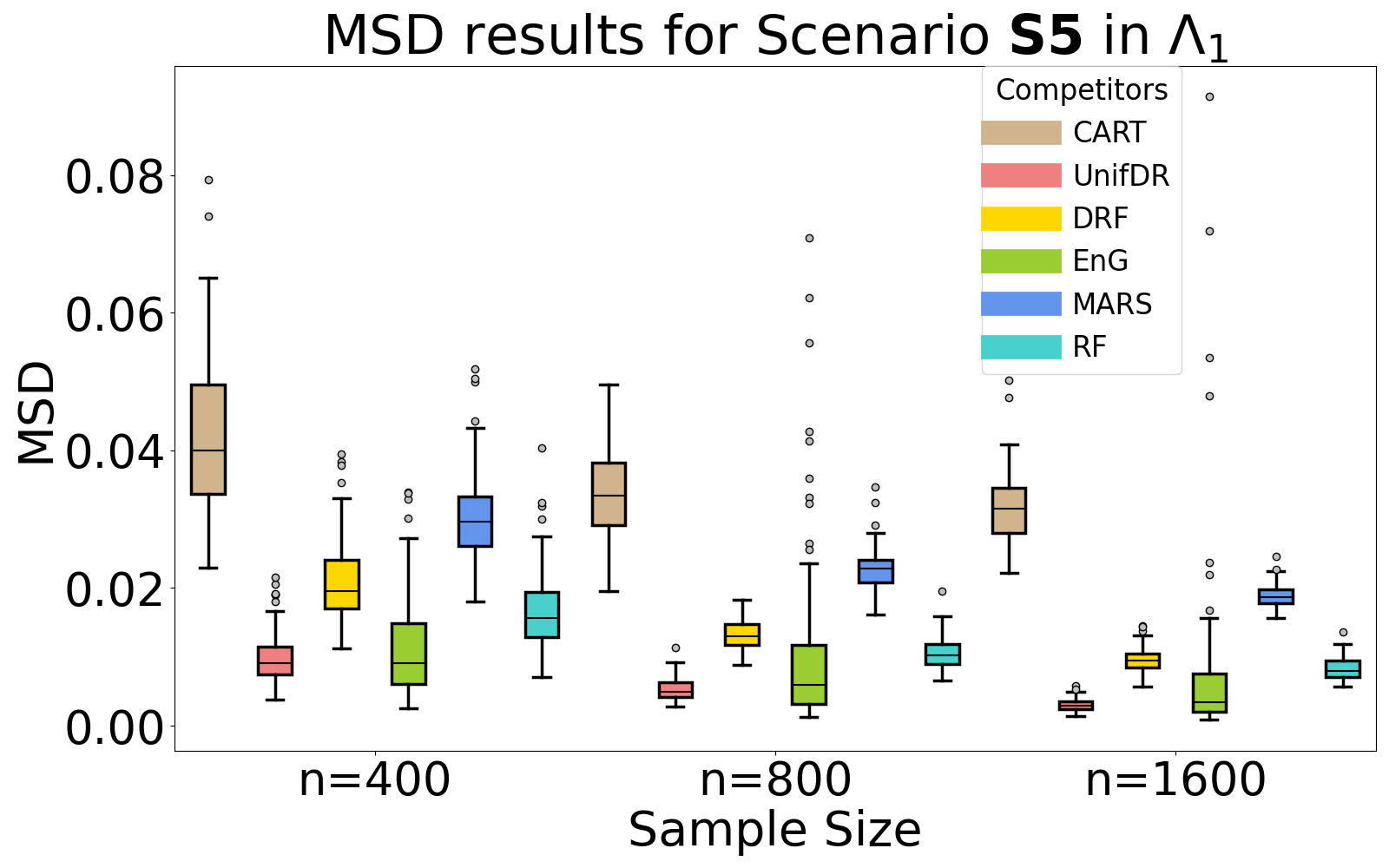} &
        \hspace{-1.1em}
\includegraphics[width=0.47\columnwidth]{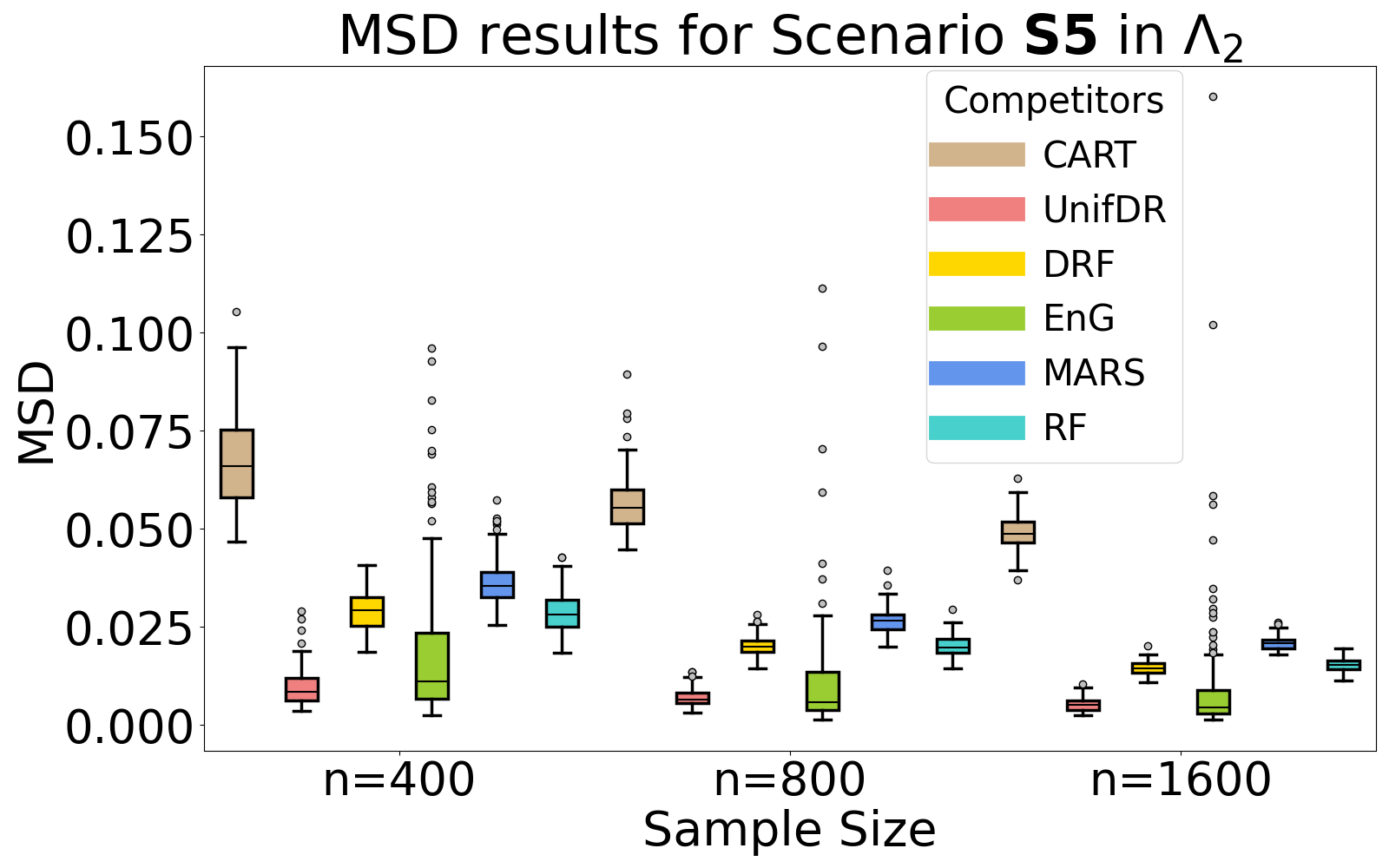} \\
    \multicolumn{2}{c}{\includegraphics[width=0.47\columnwidth]{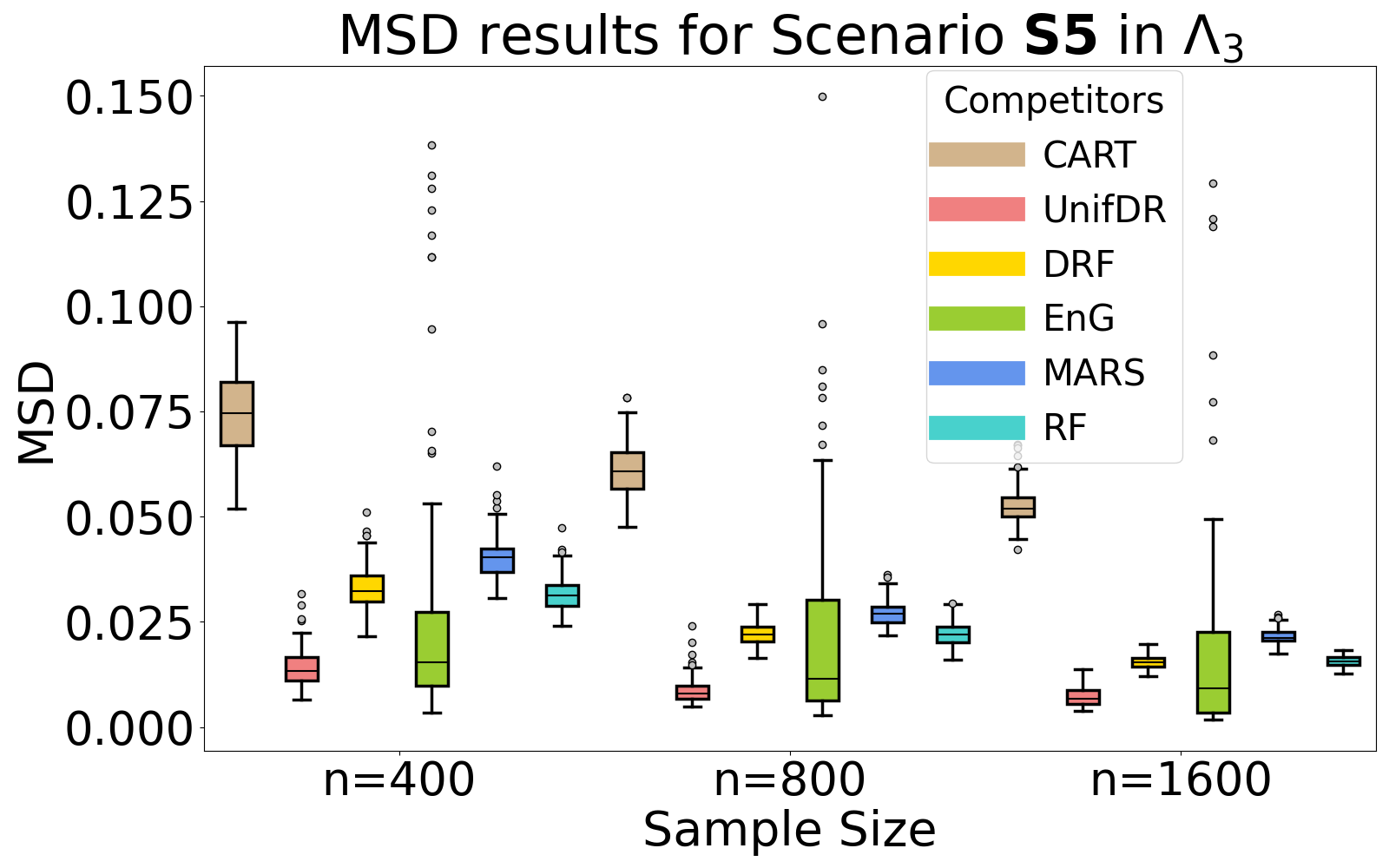}}
\end{tabular}
\caption{Box plots for MSD  in  \textbf{S5}. The top row shows results for \(\Lambda_1\) (left) and \(\Lambda_2\) (right), while the bottom row displays results for \(\Lambda_3\).}
\label{fig:composite-S5}
\vskip -0.3in
\end{figure}

\begin{figure}[]
\begin{center}
\centerline{
\includegraphics[width=0.5\textwidth]{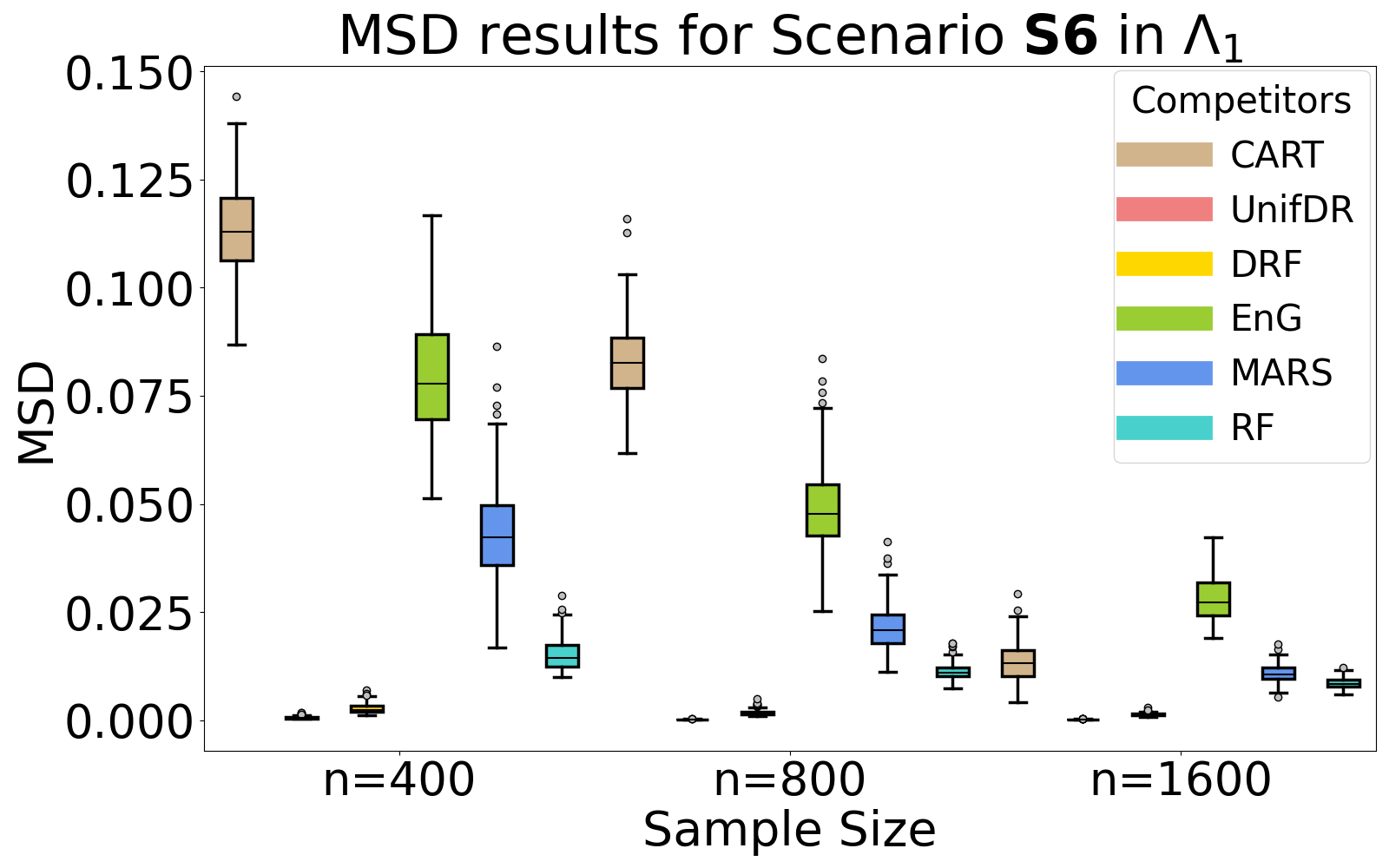}
\hspace{-0.5em}
\includegraphics[width=0.5\textwidth]{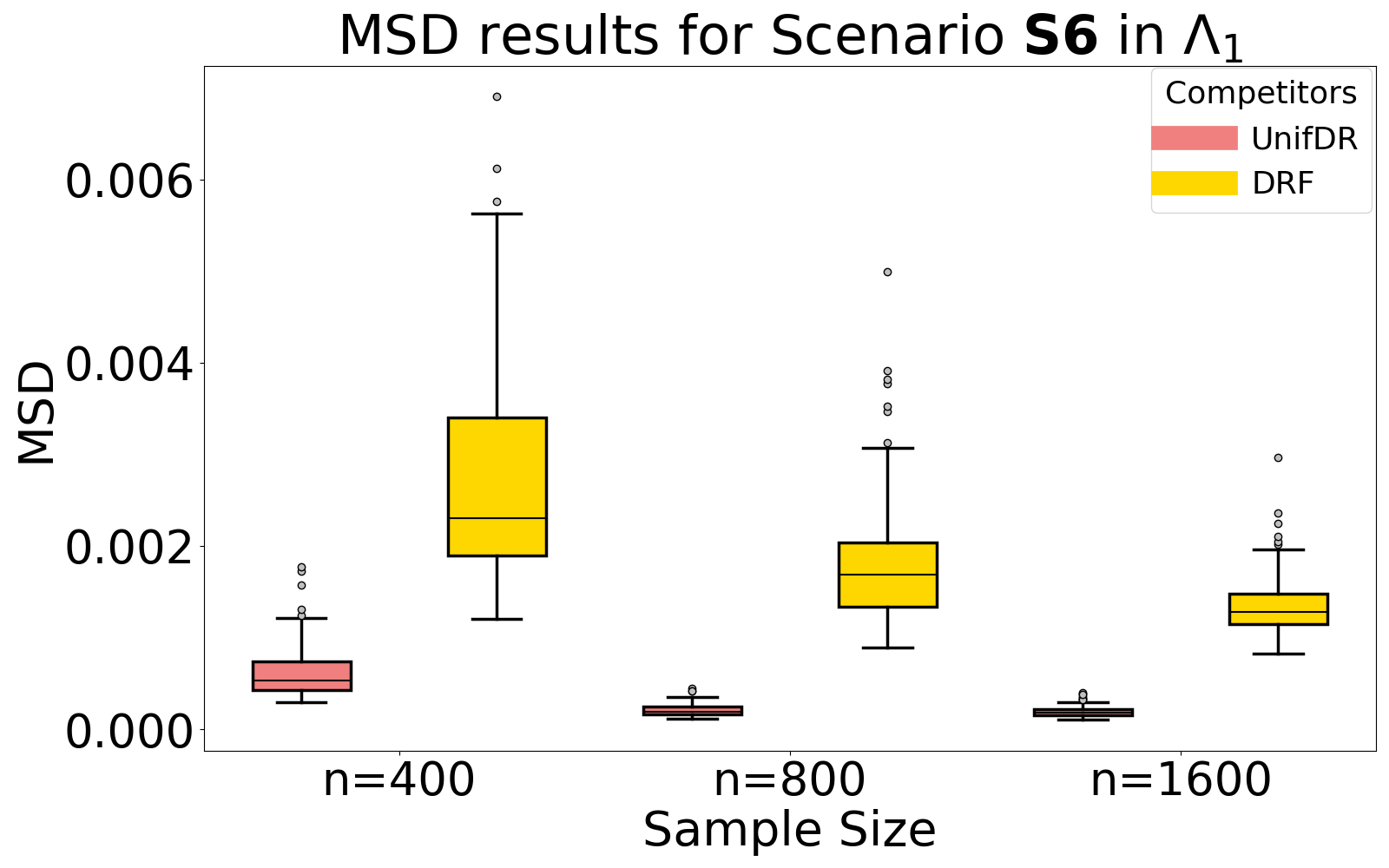}
}
\caption{Box plots for MSD in {\bf S6} using evaluation set $\Lambda_1$. The left plot corresponds to all competitors performance, while the right plot corresponds to best two competitors.}
\label{fig:lambda1-S6-MSD}
\end{center}
\vskip -0.2in
\end{figure}

\begin{figure}[]
\begin{center}
\centerline{
\includegraphics[width=0.5\textwidth]{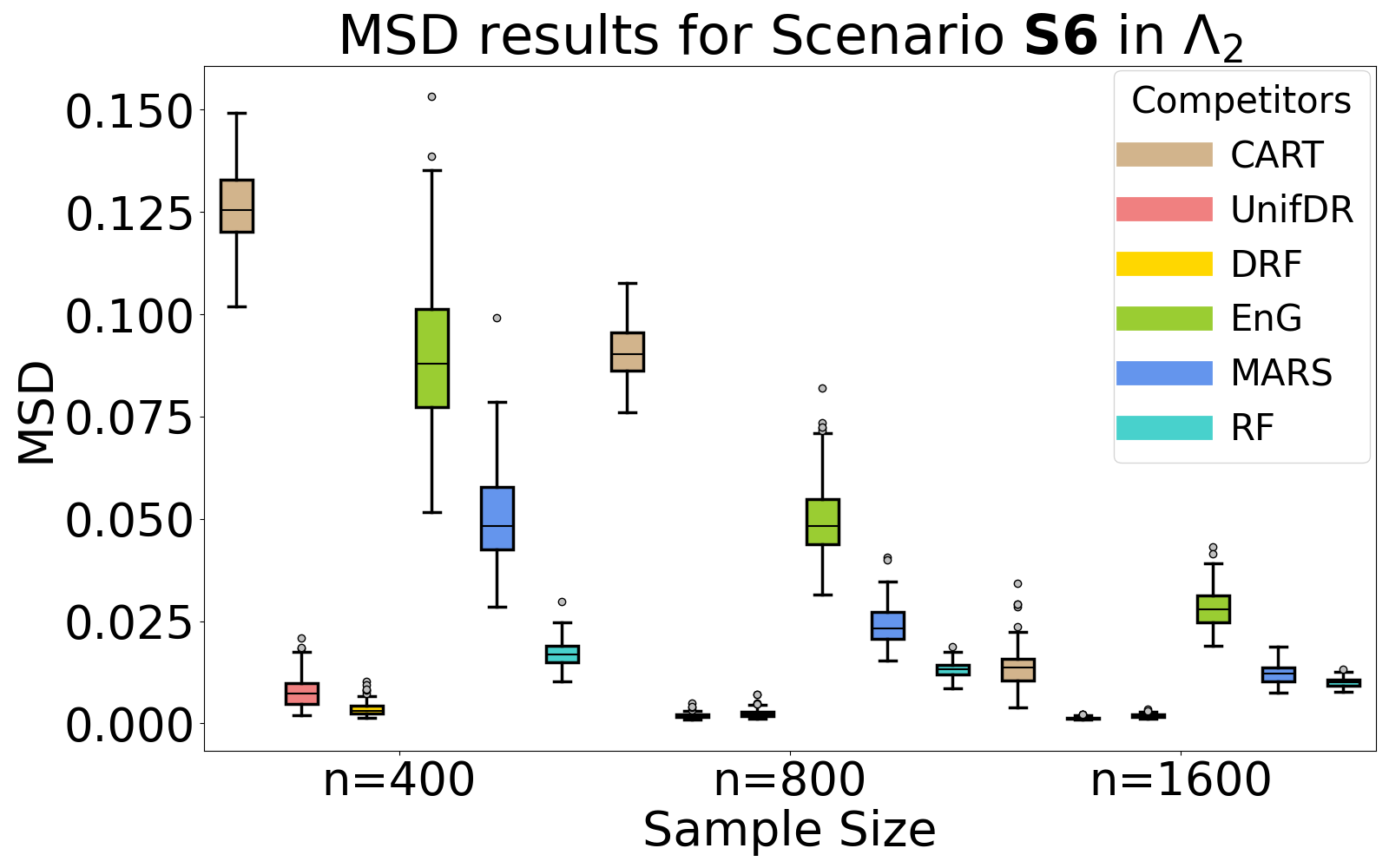}
\hspace{-0.5em}
\includegraphics[width=0.5\textwidth]{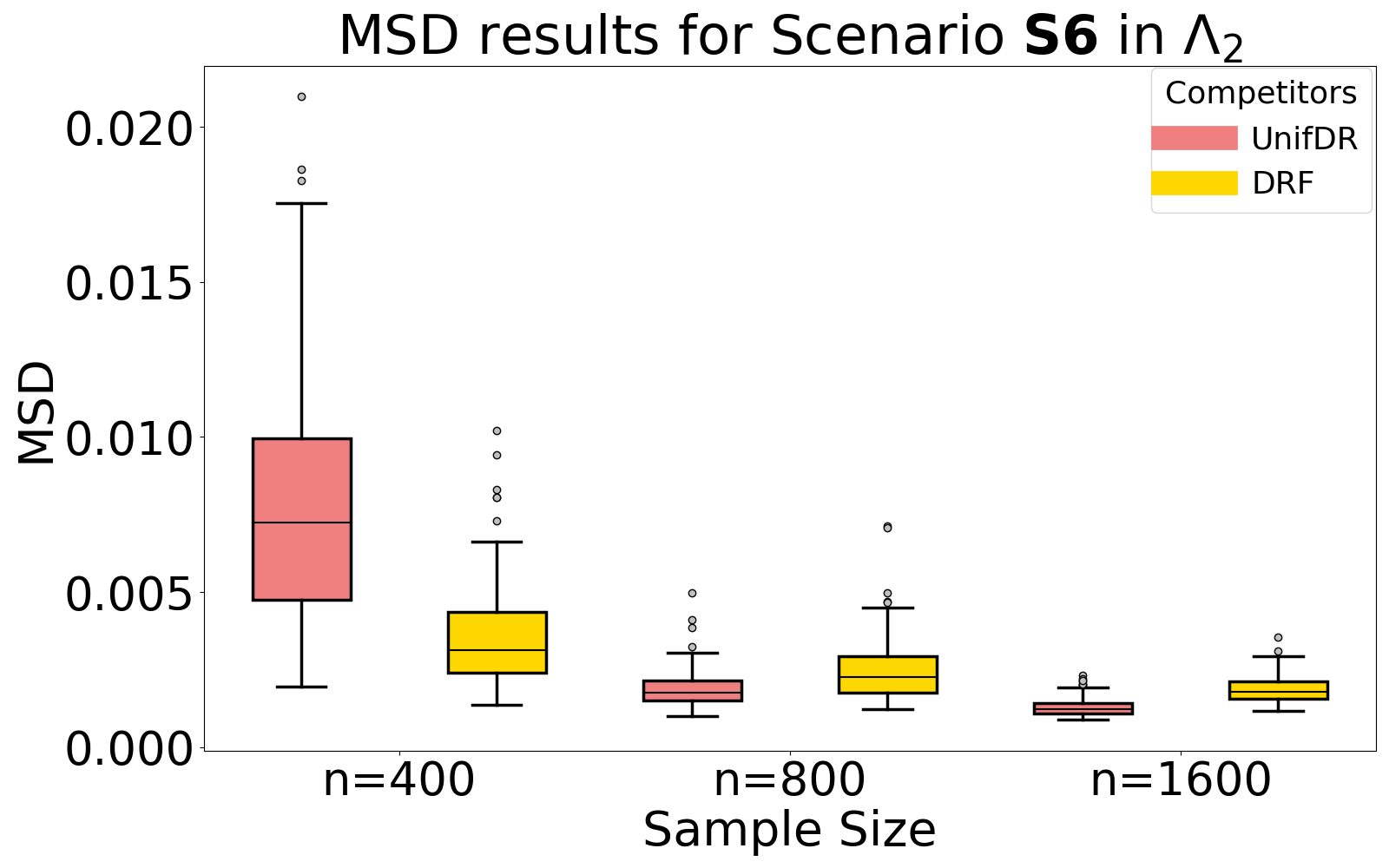}
}
\caption{Box plots for MSD in {\bf S6} using evaluation set $\Lambda_2$. The left plot corresponds to all competitors performance, while the right plot corresponds to best two competitors.}
\label{fig:lambda2-S6-MSD}
\end{center}
\vskip -0.3in
\end{figure}

\begin{figure}[]
\begin{center}
\centerline{
\includegraphics[width=0.5\textwidth]{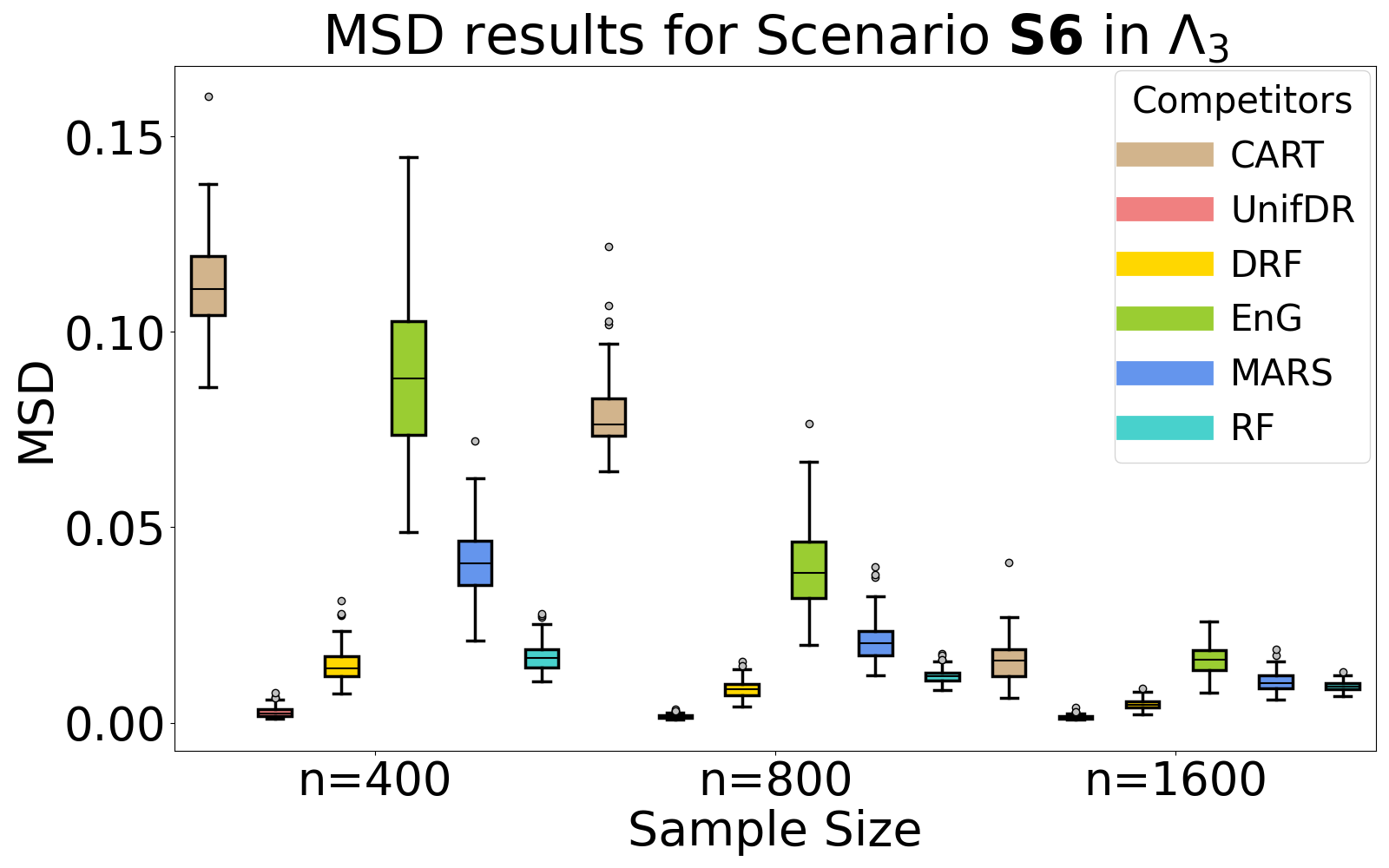}
\hspace{-0.5em}
\includegraphics[width=0.5\textwidth]{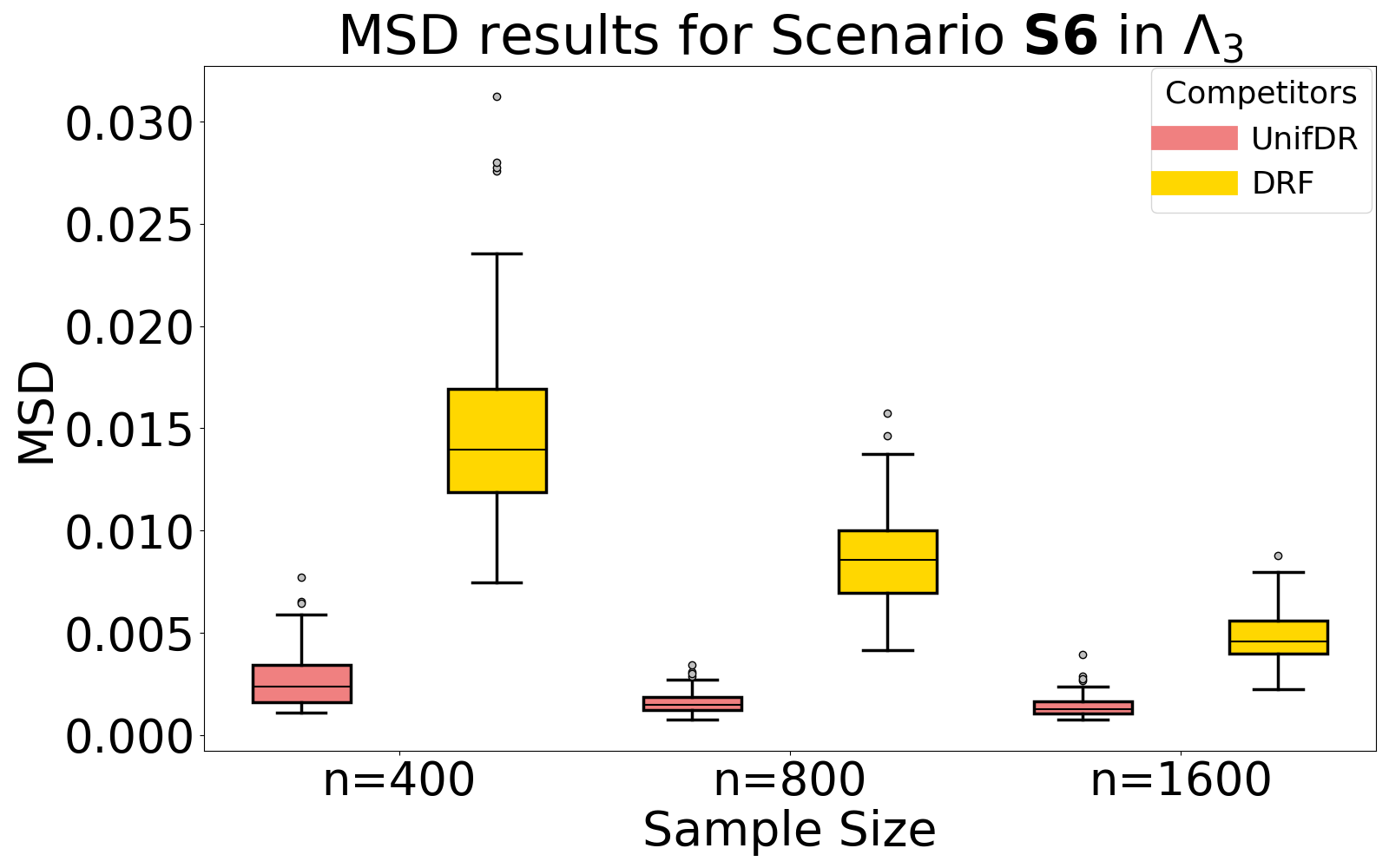}
}
\caption{Box plots for MSD in {\bf S6} using evaluation set $\Lambda_3$. The left plot corresponds to all competitors performance, while the right plot corresponds to best two competitors.}
\label{fig:lambda3-S6-MSD}
\end{center}
\vskip -0.2in
\end{figure}

\newpage

\section{Proofs}
\label{Appenproofs}

\subsection{Proof of Lemma \ref{lemma-1}}

\begin{proof}
	Let   $y_{(1)} \leq  y_{(2)}   \leq \ldots \leq  y_{(n)}$  the order statistics of  $y$.
	Notice that
	\[
	 \begin{array}{lll}
	 \displaystyle	L(F)   & := &  \displaystyle	\sum_{i=1}^{n}  \mathrm{CRPS}(F_i,  1\{y_i \leq  \cdot\} )  \\
	  &= & \displaystyle  \int    \sum_{i=1}^{n}  (F_i(t) - 1\{y_i \leq  t\} )^2dt  \\ 
	 & = & \displaystyle  \sum_{j=1}^{n+1} \int_{A_j} \sum_{i=1}^{n}  (F_i(t) - 1\{y_i \leq  t\} )^2dt   
  	 \end{array}
	\]
	where  $A_1 = (-\infty, y_{(1)}),\,  A_2 = [y_{(1)},y_{(2)}),\ldots, A_n = [y_{(n-1)},y_{(n)}), \, A_{n+1} =[y_{(n)},\infty)$.   However, for every $j \in \{1,\ldots, n+1\}$ and  $t, t^{\prime} \in A_j$ we have that 
	\[
	 \begin{array}{l}
	        \underset{F(t) \in K}{\min} \sum_{i=1}^{n}  (F_i(t) - 1\{y_i \leq  t\} )^2 
            \,= \,      \underset{F(t^{\prime}) \in K}{\min} \sum_{l=1}^{n}  (F_{i}(t^{\prime}) - 1\{y_i \leq  t^{\prime}\}   )^2.
	 \end{array}
	\]
	Hence,  letting  $t_j $ be an element of $A_j$, we obtain that minimizing  $L(F)$ with the constraints $F(t) \in K$  for all $t$ is equivalent to solving the independent problems 
	\[
	   \underset{ F(t_j)  \in  K }{\min}\,   	\sum_{i=1}^{n}  (F_i(t_j) - 1\{y_i \leq  t_j\} )^2,
	 \]
	 and the claim follows.
\end{proof}

\subsection{Proof of Theorem \ref{thm2}}

\begin{theorem}
	\label{thm1}
	\textbf{[Theorem A.1 in   \cite{guntuboyina2020adaptive}].} There exists a universal positive constant  $C >0$ such that for  every  $t$, 
	\[
	\begin{array}{l}
	\displaystyle    \frac{1}{n}\mathbb{E}\left(   \sum_{i=1}^{n}   \left(  \widehat{F}_i(t) -  F_i^*(t)  \right)^2    \right) 
     \,\leq  \,\frac{   C  \max\{\eta^2,\underset{i=1,\ldots,n}{\max }\,  F_i^*(t)(1-F_i^*(t)) \}    }{n}, 
	\end{array}
	\]
	for every  $\eta>1$ satisfying 
    \begin{equation}
    	\label{eqn:guntuboyina_cond}
    	 \mathbb{E}\left[  \underset{ \theta \in K_t   \,:\,  \| \theta  - F^*(t)\|\leq  \eta }{\sup}   \,\epsilon(t)^{\top} (\theta - F^*(t))     \right]\leq  \frac{\eta^2}{2}
    \end{equation}
	where  $\epsilon(t) =  w(t) -  F^*(t) $.
\end{theorem}

In the followig we first present the proof of Theorem \ref{thm1}.
\begin{proof}
Let $t\in\mathbb{R}.$ Define $\sigma^2=\underset{i=1,\ldots,n}{\max }\,  F_i^*(t)(1-F_i^*(t))$. We consider the following two cases separately based on the value of \(\sigma\):
\begin{enumerate}
    \item \(\sigma = 0\),
    \item \(\sigma \neq 0\).
\end{enumerate}
{\bf{Case 1: \(\sigma = 0\)}}.

By definition, 
$
\sigma^2 = \max_{i=1,\ldots,n} F_i^*(t)(1 - F_i^*(t)) = 0.
$
This implies that for all \(i = 1, \ldots, n\),
$
F_i^*(t)(1 - F_i^*(t)) = 0.
$
Since \(F_i^*(t)(1 - F_i^*(t)) = 0\), it follows that either \(F_i^*(t) = 0\) or \(F_i^*(t) = 1\) for each \(i\). Now observe that for each \(i \in \{1, \ldots, n\}\), either \(w_i(t) = 0\) or \(w_i(t) = 1\). 
Given that \(\mathbb{E}(w_i(t)) = F_i^*(t)\), it follows that \(w_i(t) = F_i^*(t)\). Therefore, by the definition of \(\widehat{F}\) in Equation \ref{eqn:estimator}, we have \(\widehat{F} = F^*\). 
In this case, it holds that
\[
\mathbb{E}\left(\|\widehat{F}(t) - F^*(t)\|_2^2\right) = 0,
\]
and the result is obtained trivially.

{\bf{Case 2: \(\sigma \neq 0\)}}.

Denote by $\Theta_{F^*(t)}(\eta):=\{\theta-F^*(t) \in \Theta_{F^*(t)}   \,:\,  \| \theta  - F^*(t)\|_2\leq  \eta \}$, where $\Theta_{F^*(t)}=\{\theta-F^*(t)\,:\, \theta\in K_t\}$.
 Notice that the function $L:\mathbb{R}^n\rightarrow \mathbb{R}$ given by
$$
\epsilon \rightarrow \sup _{\theta \in \Theta_{F^*(t)}(\eta)}\left|\epsilon^{\top}( \theta  - F^*(t))\right|
$$
is $\eta$-Lipschitz. Moreover observe that $L$ is separately convex. In fact, for any $k\in\{1,...n\}$ we have that
\begin{align*}
    &(\epsilon_1,...,\epsilon_{k-1},t\epsilon_k^{(1)}+(1-t)\epsilon_k^{(2)},\epsilon_{k+1},...,\epsilon_n)\\
    =&(t\epsilon_1,...,t\epsilon_{k-1},t\epsilon_k^{(1)},t\epsilon_{k+1},...,t\epsilon_n)
    \\
    +&((1-t)\epsilon_1,...,(1-t)\epsilon_{k-1},(1-t)\epsilon_k^{(2)},(1-t)\epsilon_{k+1},...,(1-t)\epsilon_n)
    \\
    =&t\epsilon_{k,1}+(1-t)\epsilon_{k,2}.
\end{align*}
Therefore,
\begin{align*}
    &L\left[(\epsilon_1,...,\epsilon_{k-1},t\epsilon_k^{(1)}+(1-t)\epsilon_k^{(2)},\epsilon_{k+1},...,\epsilon_n)\right]\\
    =&\sup _{\theta \in \Theta_{F^*(t)}(\eta)}\left|(t\epsilon_{k,1}+(1-t)\epsilon_{k,2})^{\top}( \theta  - F^*(t))\right|
    \\
    =&\sup _{\theta \in \Theta_{F^*(t)}(\eta)}\left|t(\epsilon_{k,1})^{\top}( \theta  - F^*(t))+(1-t)(\epsilon_{k,2})^{\top}( \theta  - F^*(t))\right|
    \\
    \le&\sup _{\theta \in \Theta_{F^*(t)}(\eta)}\left[t\left|(\epsilon_{k,1})^{\top}( \theta  - F^*(t))\right|+(1-t)\left|(\epsilon_{k,2})^{\top}( \theta  - F^*(t))\right|\right],
\end{align*}
where the inequality is followed by triangle inequality.
Using the fact that the supremum of the sum is bounded by the sum of the supremums, it follows that
\begin{align*}
    &L\left[(\epsilon_1,...,\epsilon_{k-1},t\epsilon_k^{(1)}+(1-t)\epsilon_k^{(2)},\epsilon_{k+1},...,\epsilon_n)\right]
    \\
    \le&t\sup _{\theta \in \Theta_{F^*(t)}(\eta)}\left|(\epsilon_{k,1})^{\top}( \theta  - F^*(t))\right|+(1-t)\sup _{\theta \in \Theta_{F^*(t)}(\eta)}\left|(\epsilon_{k,2})^{\top}( \theta  - F^*(t))\right|
    \\
    =&tL\left[(\epsilon_1,...,\epsilon_{k-1},\epsilon_k^{(1)},\epsilon_{k+1},...,\epsilon_n)\right]+(1-t)L\left[(\epsilon_1,...,\epsilon_{k-1},\epsilon_k^{(2)},\epsilon_{k+1},\epsilon_n)\right].
\end{align*}
Thus, $L$ is separately convex. Notice that $\epsilon_i(t)=\left(1\{y_i\leq t\}-F_i^*(t)\right)$ satisfies $\vert \epsilon(t)\vert\le 1.$
Hence, by Theorem 3.4 in \citet{wainwright2019high}, for any $\delta>0$
\begin{align}
\label{gunt-new-0}
\sup _{\theta \in \Theta_{F^*(t)}(\eta)}\left|\epsilon(t)^{\top}( \theta  - F^*(t))\right| \leq \mathbb{E}\left(\sup _{\theta \in \Theta_{F^*(t)}(\eta)}\left|\epsilon(t)^{\top}( \theta  - F^*(t))\right|\right)+\sigma\eta\delta
\end{align}
with probability at least $1-e^{-\delta^2\sigma^2 /16}$.
Next, we have that
\begin{align}
\label{gunt-new-1}
    \left|\epsilon(t)^{\top}( \theta  - F^*(t))\right|\le& \max\{\frac{\vert\vert  \theta  - F^*(t)\vert\vert_2}{\eta},1\}\widetilde{L}(\eta),
\end{align}
for any $\theta\in K_t,$ where $\widetilde{L}(\eta)=\sup _{\theta \in \Theta_{F^*(t)}(\eta)}\left|\epsilon(t)^{\top}( \theta  - F^*(t))\right|$. This conclusion is derived based on the subsequent line of reasoning.  If $\theta \in K_t$ and $\|\theta-F^*(t)\|_2 \leq \eta$, then $\theta-F^*(t) \in \Theta_{F^*(t)}(\eta)$ and inconsequence  $|\epsilon(t)^{\top} (\theta-F^*(t))| \leq \widetilde{L}(\eta)$, by definition of $\widetilde{L}(\eta)$. If $\theta \in K_t$ and $\|\theta-F^*(t)\|_2>\eta$, then $\frac{\theta-F^*(t)}{\|\theta-F^*(t)\|_2} \cdot \eta\in \Theta_{F^*(t)}(\eta)$ because $\Theta_{F^*(t)}$ is star-shaped, given that $K_t$ is convex. Also $\left\|\frac{\theta-F^*(t)}{\|\theta-F^*(t)\|}\eta\right\|_2=\eta$. Hence,
$$
\left| \epsilon(t)^{\top}\left( \frac{\theta-F^*(t)}{\|\theta-F^*(t)\|}\eta\right)\right| \leq \widetilde{L}(\eta),
$$
which implies,
$$
\vert \epsilon(t)^{\top} (\theta-F^*(t))\vert \leq \frac{\|\theta-F^*(t)\|_2}{\eta} \cdot \widetilde{L}(\eta),
$$
for any $\theta\in K_n.$ Then we observe that by the basic inequality
\begin{align*}
    \vert\vert w(t)-\widehat{F}(t)\vert\vert_2^2 \le &\vert\vert w(t)-F^*(t)\vert\vert_2^2.
\end{align*}
This implies that 
\begin{align*}
    \frac{1}{2}\vert\vert \widehat{F}(t)-F^*(t)\vert\vert_2^2 \le & \epsilon(t)^{\top}(\widehat{F}(t)-F^*(t)).
\end{align*}
Given that $\widehat{F}(t)\in K_t$
it follows from inequality (\ref{gunt-new-1}) that
\begin{align*}
     \frac{1}{2}\vert\vert \widehat{F}(t)-F^*(t)\vert\vert_2^2 \le& \max\{\frac{\vert\vert \widehat{F}(t)  - F^*(t)\vert\vert_2}{\eta},1\}\widetilde{L}(\eta),
\end{align*}
and by inequality (\ref{gunt-new-0}),
\begin{align*}
     \frac{1}{2}\vert\vert \widehat{F}(t)-F^*(t)\vert\vert_2^2 \le& \max\{\frac{\vert\vert \widehat{F}(t)  - F^*(t)\vert\vert_2}{\eta},1\}\left(\mathbb{E}\left(\sup _{\theta \in \Theta_{F^*(t)}(\eta)}\left|\epsilon(t)^{\top}( \theta  - F^*(t))\right|\right)+\sigma\eta\delta\right),
\end{align*}
for any $\delta>0$, with probability at least $1-e^{-\sigma^2\delta^2 /16}$.
Thus, for any $\delta>0$ 
\begin{align*}
     \vert\vert \widehat{F}(t)-F^*(t)\vert\vert_2 \le& \max\left\{\frac{2G(\eta,\delta)}{\eta},\sqrt{2G(\eta,\delta)}\right\},
\end{align*}
with probability at least $1-e^{-\sigma^2\delta^2 /16}$,
where $G(\eta,\delta)=\mathbb{E}\left(\sup _{\theta \in \Theta_{F^*(t)}(\eta)}\left|\epsilon(t)^{\top}( \theta  - F^*(t))\right|\right)+\sigma\eta\delta.$
Next, by inequality (\ref{eqn:guntuboyina_cond}),
\begin{align*}
    \max\left\{\frac{2G(\eta,\delta)}{\eta}\sqrt{2G(\eta,\delta)}\right\}\le& \max\left\{\frac{\eta^2+2\sigma\eta\delta}{\eta},\ \sqrt{\eta^2+2\sigma\eta\delta} \right\}=\max\left\{\eta+2\sigma\delta,\ \sqrt{\eta(\eta+2\sigma\delta)} \right\}\le \eta+2\sigma\delta.
\end{align*}
In consequence for any $\delta>0$,
\begin{align}
\label{gunt-new-2}
     \vert\vert \widehat{F}(t)-F^*(t)\vert\vert_2 \le& \eta+2\sigma\delta,
\end{align}
with probability at least $1-e^{-\sigma^2\delta^2 /16}$. Finally, we observe that
\begin{align*}
    \mathbb{E}(\vert\vert \widehat{F}(t)-F^*(t)\vert\vert_2^2)=&\int_{0}^{\infty} \mathbb{P}\left(\vert\vert \widehat{F}(t)-F^*(t)\vert\vert_2^2>s\right)ds
    \\
    =&\int_{0}^{(\eta+2\sigma)^2} \mathbb{P}\left(\vert\vert \widehat{F}(t)-F^*(t)\vert\vert_2^2>s\right)ds+\int_{(\eta+2\sigma)^2}^{\infty} \mathbb{P}\left(\vert\vert \widehat{F}(t)-F^*(t)\vert\vert_2^2>s\right)ds
    \\
    =&I_1+I_2.
\end{align*}
To analyze the term $I_1$ we observe that $\mathbb{P}\left(\vert\vert \widehat{F}(t)-F^*(t)\vert\vert_2^2>s\right)\le1$, and therefore 
\begin{align}
\label{I1-bound}
    I_1\le& (\eta+2\sigma)^2.
\end{align}
For the term $I_2$, we perform a change of variables $s=(\eta+2\sigma\delta)^2$ to obtain
\begin{align*}
    I_2\le\int_{1}^\infty \mathbb{P}\left(\vert\vert \widehat{F}(t)-F^*(t)\vert\vert_2^2>(\eta+2\sigma\delta)^2\right) 4\sigma(\eta+2\sigma\delta)d \delta,
\end{align*}
and by inequality (\ref{gunt-new-2}),
\begin{align}
    \label{I2-bound}
    I_2\le 4\sigma\int_{1}^\infty e^{-\frac{\sigma^2\delta^2}{16}} (\eta+2\sigma\delta)d \delta.
\end{align}
Moreover,
\begin{align*}
  4\sigma\int_{1}^\infty e^{-\frac{\sigma^2\delta^2}{16}} (\eta+2\sigma\delta)d \delta
  =&
   \, 4\eta\sigma\int_{1}^\infty e^{-\frac{\sigma^2\delta^2}{16}} d \delta+8\sigma^2\int_{1}^\infty e^{-\frac{\sigma^2\delta^2}{16}} \delta d \delta 
   \\
   =& 2\eta\frac{\sqrt{\pi} \sigma \operatorname{erfc}(\sigma)}{\sigma}+4e^{-\sigma^2}
   \\
   \le&
   2\eta\sqrt{\pi}  \operatorname{erfc}(\sigma)+4e^{-\sigma^2}
   .
\end{align*}
where $\operatorname{erfc}(z)  =\frac{2}{\sqrt{\pi}} \int_z^{\infty} e^{-\delta^2} d \delta $.
From inequality (\ref{I1-bound}) (\ref{I2-bound}), and the fact that $\eta>1$, we conclude that
\begin{align*}
     \mathbb{E}(\vert\vert \widehat{F}(t)-F^*(t)\vert\vert_2^2)\le& C_1 (\eta^2+\sigma^2),
\end{align*}
for an absolute positive constant $C_1.$
Finally observe that,
\begin{align*}
       \eta^2+\sigma^2 \leq& 2\max \left(\eta^2,\sigma^2\right) 
\end{align*}
 Taking $C=2C_1$ the result is achieved.
\end{proof}

Now we are ready to start the proof of Theorem \ref{thm2}. 

\begin{proof}
	Notice that for all $i$, it holds that  $\widehat{F}_i(t )   =F^*(t) = 0$  for all $t < \inf\{ a\,\,:\, a\in \Omega  \}$ and $\widehat{F}_i(t )   =F^*(t) = 1$  for all $t >\sup\{ a\,\,:\, a\in \Omega  \}$. Hence, 
	\[
	 \begin{array}{lll}
\displaystyle 	\mathbb{E}\left(  \frac{1}{n} \sum_{i=1}^{n} \mathrm{CRPS}( \widehat{F}_i,F_i^*  )   \right) 	&=& \displaystyle  \mathbb{E}\left(  \frac{1}{n} \sum_{i=1}^{n} \int_{-\infty}^{\infty}( \widehat{F}_i(t)\,-\, F_i^*(t)  )^2 dt   \right)  \\
&=& \displaystyle  \mathbb{E}\left(  \frac{1}{n} \sum_{i=1}^{n} \int_{\Omega }( \widehat{F}_i(t)\,-\, F_i^*(t)  )^2 dt   \right)  \\
&=&\displaystyle  \int_{ \Omega}  \mathbb{E}\left(  \frac{1}{n}\sum_{i=1}^{n} (\hat{F}_i(t) -  F_i^*(t)  )^2  \right) dt
	 	 \\
	 	  &\leq&  \displaystyle \int_{\Omega}  \frac{   C  \max\{1,\eta^2\}     }{n}dt\\
	 	&   = &\displaystyle  \frac{C  \max\{1,\eta^2\} }{n} \,\,\,\,\int_{\Omega} dt
	 \end{array}
	\]
    where the inequality follows from Theorem \ref{thm1}, by noticing that  (\ref{eqn:lc_gaussian}) and  Lemma \ref{lem4}   imply (\ref{eqn:guntuboyina_cond}).
\end{proof}

\subsection{Proof of Corollary \ref{cor1}}

\begin{proof}
Throughout we use the notation from Definitions \ref{def6} and \ref{def7}.

First, notice that $\widehat{F}_i(t) = \widetilde{F}_i(t) =0 $ for all $t 
< y_{(1)}$ and for all $i$. Similarly,  $\widehat{F}_i(t) = \widetilde{F}_i(t) =1$ for all $t
\geq y_{(n)}$ and for all $i$. Therefore,
\begin{equation}
    \label{eqn:e60}
     \int_{-\infty }^{y_{(1)}}  ( \widehat{F}_i^+(t) - F_i^*(t)  )^2  \,+\,   \int_{y_{(n)} }^{ \infty }  ( \widehat{F}_i^+(t) - F_i^*(t)  )^2 \,=\, \int_{-\infty }^{y_{(1)}}  ( \widetilde{F}_i(t) - F_i^*(t)  )^2  \,+\,   \int_{y_{(n)} }^{ \infty }  ( \widetilde{F}_i(t) - F_i^*(t)  )^2.
\end{equation}
Next, define
\[
  \widehat{G}_i(t)\,:=\, \begin{cases}
      \widehat{F}_i^+( (1-t) (y_{(n)} -  y_{(1)} ) + y_{(1)}  ) & \text{for} \,\, t\in [0,1),\\
  0  &  \text{otherwise}.
  \end{cases} 
\]
Clearly, $\widehat{G}_i(0) = \widehat{F}_i^+(y_{(n)})$ and  $\widehat{G}_i(1) = \widehat{F}_i^+(y_{(1)})$.  Moreover, recalling that for $t \in [y_{(1)}, y_{(n)}  )$, we can write 
\[
 \widehat{F}_i^+(t) \,=\, \sum_{k=1  }^{n-1} a_{i,j_k }1_{  [ y_{( j_k) }, y_{ ( j_k+1)  })}(t), 
\]
then for $t\in [0,1)$, it holds that 
\begin{equation}
    \label{eqn:e57}
    \widehat{G}_i(t)\,:=\,\sum_{k =1  }^{n-1} a_{i,j_k }1_{  [u_{j_{k}+1}, u_{j_k} ) }(t) 
\end{equation}
where 
\[
u_{l} \,:=\,1\,-\, \frac{  y_{(l)} - y_{(1)}  }{y_{(n)}- y_{(1)} }
\]
for $l\in \{1,\ldots,n\}$.

Furthermore,  let 
\[
  G_i^*(t)\,:=\, \begin{cases}
      F_i^*( (1-t) (y_{(n)} -  y_{(1)} ) + y_{(1)}  ) & \text{if} \,\,t  \in [0,1),\\
      0 &\text{otherwise.}
  \end{cases}
\]

Now, we observe that 
\begin{align}
   \label{eqn:e50}
       & \displaystyle  \int_{0}^{1}(\widehat{G}_i(t) -G_i^*(t)  )^2 dt \nonumber
       \\
          = &   \displaystyle  \int_{0}^{1}(\widehat{F}_i^{+}( (1-t) (y_{(n)} -  y_{(1)} ) + y_{(1)}  ) -F_i^*( (1-t) (y_{(n)} -  y_{(1)} ) + y_{(1)}  )  )^2 dt \nonumber
       \\
=&          \displaystyle \frac{1}{  y_{(n)}-y_{(1)}  } \int_{y_{(1)}}^{ y_{(n)}  }(\widehat{F}_i^+( s ) -F^*_i(s)  )^2 ds 
\end{align}
by making the change of variable $s=   (1-t) (y_{(n)} -  y_{(1)} ) + y_{(1)}  $.

Furthermore, by Lemma \ref{lem21}, 
\begin{align}
    \label{eqn:e51}
& \int_{0}^{1}  \vert D(G_i^*)(t)\vert^2dt  \,=\,   \int_{0}^{\infty}  \vert D(G_i^*)(t)\vert^2dt \nonumber \\
 &\,=\, \int_{0}^{\infty}  \vert G_i^*(t)\vert^2dt  \,=\,  \int_{0}^{1}  \vert G_i^*(t)\vert^2dt \,=\,  \frac{1}{ y_{(n)} -y_{(1)}  }      \int_{y_{(1)}}^{ y_{(n)}  } \vert F_i^*(s)\vert^2 ds,
\end{align}
and 
\begin{align}
    \label{eqn:e52}
 &\int_{0}^{1}  \vert D(\widehat{G}_i)(t)\vert^2dt  \,=\,   \int_{0}^{\infty}  \vert D(\widehat{G}_i)(t)\vert^2dt   \,=\, \int_{0}^{\infty}  \vert \widehat{G}_i(t)\vert^2dt \nonumber
 \\
 &\,=\,  \int_{0}^{1}  \vert \widehat{G}_i(t)\vert^2dt \,=\,  \frac{1}{ y_{(n)} -y_{(1)}  }      \int_{y_{(1)}}^{ y_{(n)}  } \vert \widehat{F}_i^+(s)\vert^2 ds,
\end{align}
Also, by Lemma \ref{lem21},
\begin{align}
    \label{eqn:e53-new}
  & -  \int_{0}^{1}   D(\widehat{G}_i)(t) \cdot D(G_i^*)(t) dt\,= \, -  \int_{0}^{\infty}   D(\widehat{G}_i)(t) \cdot D(G_i^*)(t) dt \, \nonumber
   \\
   \leq\,& -  \int_{0}^{\infty}  \widehat{G}_i(t) \cdot G_i^*(t) dt\,=\,-  \int_{0}^{1}  \widehat{G}_i(t) \cdot G_i^*(t) dt
\end{align}
which implies
\begin{equation}
    \label{eqn:e54}
   -  \int_{0}^{1}   D(\widehat{G}_i)(t) \cdot D(G_i^*)(t) dt\,\leq \,  \frac{1}{ y_{(n)} -y_{(1)}  }      \int_{y_{(1)}}^{ y_{(n)}  }   \widehat{F}_i^+(t)\cdot F^*_i(t) dt.
\end{equation}

Combining (\ref{eqn:e51}), (\ref{eqn:e52}) and (\ref{eqn:e53-new}), we obtain that  
\begin{equation}
    \label{eqn:e55}
    \int_{0}^{1} (D(\widehat{G}_i)(t)    \,-\,  D(G_i^*)(t) )^2 dt\,\leq\,  \frac{1}{ y_{(n)} -y_{(1)}  }      \int_{y_{(1)}}^{ y_{(n)}  } (\widehat{F}_i^+(t)- F^*_i(t) )^2 dt.
\end{equation}

However, since $G_i^*$ is decreasing and continuous in $[0,1)$, then by Lemma \ref{lem21}, it holds that $D(G_i^*)(t) = G_i^*(t)$ for all $t\geq 0$. Thus, from (\ref{eqn:e55}), 
\begin{equation}
    \label{eqn:e56}
    \int_{0}^{1} (D(\widehat{G}_i)(t)    \,-\,  G_i^*(t) )^2 dt\,\leq\,  \frac{1}{ y_{(n)} -y_{(1)}  }      \int_{y_{(1)}}^{ y_{(n)}  } (\widehat{F}_i^+(t)- F^*_i(t) )^2 dt.
\end{equation}
Now, by Lemma \ref{lem20} and  (\ref{eqn:e57}), we obtain that
\begin{equation}
    \label{eqn:e58}
  D(\widehat{G}_i)(t) \,:=\,\sum_{l=1}^{n-1} a_{i,j_l, }1_{  [ m_{l-1}, m_{l}) }(t) 
\end{equation}
    where 
    \[
    m_l \,:=\, \sum_{k=1}^l \frac{ y_{ (j_k+1)  }\,-\,y_{(j_k)}  }{y_{(n)}- y_{(1)}} , \,\,\,l=1,\ldots,n-1,
    \]
    and with $m_0 =0$.

With (\ref{eqn:e58}) in hand, we let $H_i(t)= D(\widehat{G}_i)( 1-  (t-y_{(1)} )/(  y_{(n)}\,-\,y_{(1)}  )    )$ for all $t \in (y_{(1)}, y_{(n)}]$. Thus,  can write
\[
H_i(t) \,=\,\sum_{l=1}^{n-1} a_{i,j_l, }1_{  [  v_{l}, v_{l-1}) }(t) 
\]
for all $t \in (y_{(1)}, y_{(n)})$, where $v_0$ and  $v_{l} = y_{(n)}  - \sum_{k=1}^l (y_{(j_k+1) }\,-\,y_{(j_k)} ) $ for all $l=1,\ldots,n-1$. Also, $D(\widehat{G}_i)(t)\,:=\, H((1-t) (  y_{(n)}\,-\,y_{(1)}  )  + y_{(1)}  )   $.
Hence,    from (\ref{eqn:e56}) we obtain that 
\[
   \begin{array}{lll}
 \displaystyle  \frac{1}{ y_{(n)} -y_{(1)}  }      \int_{y_{(1)}}^{ y_{(n)}  } (H_i(t)- F^*_i(t) )^2 dt      & = &   \displaystyle \int_{0}^{1} (D(\widehat{G}_i)(t)    \,-\,  G_i^*(t) )^2 dt\\
        & \leq& \displaystyle  \frac{1}{ y_{(n)} -y_{(1)}  }      \int_{y_{(1)}}^{ y_{(n)}  } (\widehat{F}_i^+(t)- F^*_i(t) )^2 dt\\
   \end{array}
\]
and as a result, 
\begin{equation}
    \label{eqn:e59}
    \int_{y_{(1)}}^{ y_{(n)}  } (\widetilde{F}_i(t)- F^*_i(t) )^2 dt \,=\,  \int_{y_{(1)}}^{ y_{(n)}  } (H_i(t)- F^*_i(t) )^2 dt  \,\leq\,  \int_{y_{(1)}}^{ y_{(n)}  } (\widehat{F}_i^{+}(t)- F^*_i(t) )^2dt, 
\end{equation}
since $H_i(t) = \widetilde{F}_i(t)$ for all $t \in (y_{(1)}, y_{(n)})$. 

Combining (\ref{eqn:e60}) and (\ref{eqn:e59}), we obtain, 
\[
\int_{\mathbb{R} } (\widetilde{F}_i(t)- F^*_i(t) )^2 dt\,\leq\,  \int_{\mathbb{R}}(\widehat{F}_i^+(t)- F^*_i(t) )^2 dt \,\leq\,  \int_{\mathbb{R}}(\widehat{F}_i(t)- F^*_i(t) )^2 dt.
\]
The claim then follows. 
\end{proof}

\subsection{Proof of Theorem \ref{thm3}}

\begin{proof}
	First, by the basic inequality we have that 
	\[
	        \frac{1}{2} \sum_{i=1}^{n}   ( F_i(t) -  F_i^*(t)  )^2  
            \,\leq\,  \sum_{i=1}^{n}    (  F_i^*(t) -  1_{  \{y_i \leq t\} }  ) (   F_i^*(t) -   F_i(t) )
	\]
	for all $F  \in  \{  \kappa F ^*+(1-\kappa) \widehat{F}\,:\, \kappa \in [0,1]  \}$. Hence, for $\xi_1,\ldots,\xi_n$ are independent Rademacher random variables, we have that 
	\[
	\begin{array}{lll}
		\displaystyle  \mathbb{P}\left(   \underset{t \in \mathbb{R}}{\sup} \,    \sum_{i=1}^{n}   \left(  \widehat{F}_i(t)    - F^*_i(t)  \right)^2   > 2\eta^2   \right)   
      &   \leq&    	  \displaystyle  \mathbb{P}\left(   \underset{t \in \mathbb{R}}{\sup}\, \, \underset{  \theta \in K_t\,:\,    \|\theta - F^*(t)\|\leq \eta  }{\sup} \,   \sum_{i=1}^{n}  (F_i^*(t)   -   1_{\{ y_i\leq t\}} )(  \theta_i - F_i^*(t) )  > \eta^2\right)\\
		& \leq &  	  \displaystyle  \mathbb{P}\left(   \underset{t \in \mathbb{R}}{\sup}\, \, \underset{  \theta \in K\,:\,    \|\theta - F^*(t)\|\leq \eta }{\sup} \,   \sum_{i=1}^{n}  (F_i^*(t)   -   1_{\{ y_i\leq t\}} )(  \theta_i - F_i^*(t) )  > \eta ^2\right)\\
	&	\leq &  \displaystyle \frac{1}{\eta^2}\mathbb{E}\bigg(\underset{t \in \mathbb{R}}{\sup}\, \, \underset{  \theta \in K\,:\,    \|\theta - F^*(t)\|\leq \eta }{\sup} \,   \sum_{i=1}^{n}  (F_i^*(t)   -   1_{\{ y_i\leq t\}} )(  \theta_i - F_i^*(t) )  \bigg)\\
	& \leq& \displaystyle \frac{1}{\eta^2}\mathbb{E}\bigg(\underset{t \in \mathbb{R}}{\sup}\, \, \underset{  \theta \in K-K\,:\,    \|\theta \|\leq \eta }{\sup} \,   \sum_{i=1}^{n}  (F_i^*(t)   -   1_{\{ y_i\leq t\}} ) \theta_i  \bigg)\\
	&	\leq& \displaystyle \frac{2}{\eta^2}\mathbb{E}\bigg(\underset{t \in \mathbb{R}}{\sup}\, \, \underset{  \theta \in K-K\,:\,    \|\theta \|\leq \eta }{\sup} \,   \sum_{i=1}^{n}  \xi_i  1_{\{ y_i\leq t\}} \theta_i  \bigg)\\
	\end{array}
	\]
	\[
	\begin{array}{lll}
		& \leq & \displaystyle \frac{2}{\eta^2}\mathbb{E}\bigg( \mathbb{E}\bigg(\underset{t \in \mathbb{R}}{\sup} \, \,\underset{  \theta \in K-K\,:\,    \|\theta \|\leq \eta  }{\sup}\,\,  \sum_{i=1}^{n}  \xi_i  1_{\{ y_i\leq t\}} \theta_i   |  y\bigg)  \bigg)\\
		& =&    	  \displaystyle \frac{2}{\eta^2}\mathbb{E}\bigg( \mathbb{E}\bigg(   \,\underset{t \in   \{  y_1,\ldots,y_n   \}  }{\max} \,\,   \underset{  \theta \in K-K\,:\,    \|\theta \|\leq \eta  }{\sup}\,  \sum_{i=1}^{n}  \xi_i  1_{\{ y_i\leq t\}} \theta_i   \bigg|  y\bigg)  \bigg)\\   
	\end{array}
	\]
	where the second inequality follows from Markov's inequality,  the fourth by simmetrization.  Next, notice that for a fixed $t$ and $y$, the random variables $\{  \xi_i  1_{\{ y_i\leq t\}}\}_{i=1}^n$ are subGaussian($1$). Hence, by Lemma \ref{lem3},
	\[
	\begin{array}{lll}
		\displaystyle  \mathbb{E}\bigg(   \,\underset{t \in   \{  y_1,\ldots,y_n   \}  }{\max} \,\,   \underset{  \theta \in K-K\,:\,    \|\theta \|\leq \eta  }{\sup}\,  \sum_{i=1}^{n}  \xi_i  1_{\{ y_i\leq t\}} \theta_i   \bigg|  y\bigg) &\leq & \displaystyle C \int_{0}^{\eta/4}    \sqrt{\log N(\varepsilon,(K-K)\cap   B_{\varepsilon}(0),\|\cdot\|   )} d\varepsilon  \\
		& & \displaystyle \,+ \,  C  \eta \sqrt{\log n},
	\end{array}
	\]
	for some constant $C>0$. The claim then follows.
	
\end{proof}

\subsection{Proof of Corollary \ref{thm_isotonic}}

\begin{proof}
	Following the proof of Theorem 2.2  in \citet{chatterjee2014new}, we obtain that for any positive integer $l$ it holds, for $g \sim  N(0,  I_n)$, that 
	\begin{equation}
		\label{eqn:lc_gaussian_iso}
	\begin{array}{lll}
	\displaystyle 		\mathbb{E}\left[  \underset{ \theta \in K_t   \,:\,  \| \theta  - F^*(t)\|\leq  \eta }{\sup}   \, g^{\top} (\theta - F^*(t))     \right] &=&	\displaystyle \mathbb{E}\left[  \underset{ \theta \in K   \,:\,  \| \theta  - F^*(t)\|\leq  \eta }{\sup}   \, g^{\top} (\theta - F^*(t))    \right]  \\
	 &\leq & \displaystyle  C_1 \left[   2\sqrt{2^l   \eta  } n^{1/4} \,+ \, \frac{\eta^2}{2^{l-1}  }   \right]
	\end{array}
	\end{equation}
	for a positive constant $C_1$. Next, let $L$ the constant in (\ref{eqn:lc_gaussian}).  We now choose $l$ large enough such that 
	\[
 \frac{ C_1 \eta^2}{2^{l-1}  }  	\,\leq\, \frac{\eta^2}{2L}.
	\]
	Furthermore,  for this choice of $l$, we can choose $\eta$ as $\eta \asymp n^{1/6}$ such that 
	\[
	C_1  2\sqrt{2^l   \eta  } n^{1/4}  \,\leq\,  \frac{\eta^2}{2L}.
	\]
	Thus,  for a choice of $\eta $  satisfying  $\eta \asymp n^{1/6}$, we obtain 
	\[
		\underset{t \in \mathbb{R}   }{\sup}\,	\mathbb{E}\left[  \underset{ \theta \in K_t   \,:\,  \| \theta  - F^*(t)\|\leq  \eta }{\sup}   \, g^{\top} (\theta - F^*(t))     \right]  \,\leq\, \frac{\eta^2 }{L},
	\]
    and so (\ref{eqn:e11}) follows from Theorem \ref{thm2}.    Furthermore, the corresponding conclusion for $\{\widetilde{F}(t)\}_{t\in \mathbb{R}}$ follows from Corollary \ref{cor1}.

	Finally, we notice that for some positive constant $C_2$
	\[
     \begin{array}{lll}
\displaystyle      	\int_{0}^{\eta/4}    \sqrt{\log N(\varepsilon,(K\cap [a,b] -K \cap [a,b])\cap   B_{\varepsilon}(0),\|\cdot\|   )} d\varepsilon   & \leq & \displaystyle 2	\int_{0}^{\eta/4}    \sqrt{\log N(\varepsilon/2,(K\cap [a,b] \cap   B_{\varepsilon}(0),\|\cdot\|   )} d\varepsilon   \\
 & \leq& \displaystyle 2	\int_{0}^{\eta/4}    \sqrt{\log N(\varepsilon/2,(K\cap [a,b] \cap   B_{\varepsilon}(0),\|\cdot\|   )} d\varepsilon   \\
  & \leq& \displaystyle 2	\int_{0}^{\eta/4}    \sqrt{  \frac{  2C_2 \sqrt{n}   (b-a) }{ \varepsilon }   }   d\varepsilon   \\
   & \leq &\displaystyle 2 \sqrt{2C_2  (b-a) } n^{1/4} \eta^{1/2}
      \end{array}
	\]
	where the third inequality follows from Lemma 4.20 in \citet{chatterjee2014new}. Therefore, the claim in (\ref{eqn:e12}) follows form Theorem \ref{thm3}  by taking $\eta$ satisfying $\eta \asymp (b-a)^{1/3}n^{1/6}  \,+\, \sqrt{\log n}$.
\end{proof}

\newpage

\subsection{Proof of Theorem \ref{thm6} }

\begin{proof}
First  we observe that  by the basic inequality, for all $t  \in \mathbb{R}$,
\begin{equation}
\label{eqn:e4}
\frac{ \|F(t)-F^*(t)\|^2   }{2}\,\leq\, a(t)^{\top} (   F(t) -  F^*(t)    )  +        \lambda_t[\mathrm{pen}_t (    F^*(t) ) \,- \, \mathrm{pen}_t (    F(t) ) ]
\end{equation}
where  $a(t) =  w(t) - F^*(t)$  for  all $t \in \mathbb{R}$ and  $i \in \{1,\ldots,n\}$, and where the inequality holds for all
\[
   F(t)   \,\in \, \Lambda(t)  \,:=\, \left\{  s \widehat{F}(t) +    (1-s) F^*(t) \,:\,  s\in [0,1]      \right\} \subset \mathbb{R}^n.   
\]
Therefore,
\[
\begin{array}{lll}
\displaystyle 	\mathrm{pen}_t(   F(t) )       & \leq&  \displaystyle  \mathrm{pen}_t(   F(t) )   \,+\,     \frac{ \|F(t)-F^*(t)\|^2   }{2\lambda_t} \\
 & \leq &   \displaystyle \frac{a(t)^{\top} (   F(t) -  F^*(t)    ) }{\lambda_t}   +       \mathrm{pen}_t(    F^*(t) )
\end{array}
\]
for all $F(t) \in  \Lambda(t)$ and  $t \in \mathbb{R}$.
Hence, by the properties of $\mathrm{pen}_t(\cdot)$, we have that 
\begin{equation}
\label{eqn:e2}
\begin{array}{lll}
	\displaystyle 	\mathrm{pen}_t(   F(t)  - F^*(t)  )  & \leq&\displaystyle \mathrm{pen}_t(   F(t))   \,+\,    \mathrm{pen}_t( F^*(t)  )   \\
	& \leq&    \displaystyle \frac{a(t)^{\top} (   F(t) -  F^*(t)    ) }{\lambda_t}   +      2  \mathrm{pen}_t(    F^*(t) )\\
\end{array}
\end{equation}
for all $F(t) \in  \Lambda(t)$ and  $t \in \mathbb{R}$.

Now  suppose that there exists $F(t) \in \Lambda(t)$   such that 
\[
      \| F(t) -  F^*(t) \|     \,\leq \, \eta^2
\]
and  $\mathrm{pen}_t(    F(t) )   \geq  5\mathrm{pen}_t(    F^*(t) )$. Then 
\[
      \begin{array}{lll}
        \mathrm{pen}_t(   F(t) -  F^*(t) )    & \geq&              \mathrm{pen}_t(   F(t) )\,-\,         \mathrm{pen}_t( F^*(t) )  \\ 
        & \geq& 4\mathrm{pen}_t( F^*(t) ) .
          \end{array}
\]
Hence, we let  
\[
s   \,:=\  \frac{     4\, \mathrm{pen}_T(  F^*(t) )  }{   \mathrm{pen}_t (    F(t)-F^*(t)  )  }   \in [0,1].
\]
Then we set 
\[
   \widetilde{F}(t) \,:=   s F(t)  \,+\, (1-s) F^*(t)        \,\in  \,  \Lambda(t).
 \]
 As a result,
 \[
        \|   F^*(t) - \widetilde{F}(t) \|^2 \,\leq\,\|   F^*(t) -  F(t) \|^2   \,\leq \eta^2.
 \]
 Also,
 \[
    \begin{array}{lll}
    	\mathrm{pen}_t(  \widetilde{F}(t) - F^*(t)   )  & = &    \mathrm{pen}_t(s F(t)  \,+\, (1-s) F^*(t)       -    F^*(t)   ) \\
    	& = &    \mathrm{pen}_t(s F(t)  \,-\, s F^*(t)      ) \\
    		& = &    s\mathrm{pen}_t( F(t)  \,-\, F^*(t)      ) \\
    		  & = & 4\mathrm{pen}_t(   F^*(t) ).
    \end{array}
 \]
 Therefore, 
 \[
 \begin{array}{lll}
 4\mathrm{pen}_t(   F^*(t) )    & = &\mathrm{pen}_t(  \widetilde{F}(t) - F^*(t)   )  \\
  & \leq&      \displaystyle \frac{a(t)^{\top} (     \widetilde{F}(t)  - F^*(t)  ) }{\lambda_t}   +      2  \mathrm{pen}_t(    F^*(t) ),
 \end{array}
 \]
 where the inequality follows from (\ref{eqn:e2}). This implies that 
 \begin{equation}
 \label{eqn:e3}
   2  \mathrm{pen}(      F^*(t)  ) \,\leq \, \frac{a(t)^{\top} (  \widetilde{F}(t)    -  F^*(t)  ) }{\lambda_t}.
 \end{equation}
 Hence, we  let
 \[
       \lambda _t\,=\,   \frac{\eta^2}{  4   \mathrm{pen}(  F^*(t))  }.
 \]
 Then from (\ref{eqn:e3})  we obtain that
 \[
        \frac{\eta^2}{2} \,\leq \,      a(t)^{\top} (    \widetilde{F}(t)   - F^*(t)     ) .
 \]
 It follows that the events
 \[
  \Omega_1 \,:=\, \underset{t \in \mathbb{R} }{\bigcup}\left\{\,\, \underset{F(t) \in  \Lambda(t) \,:\,     \| F(t) - F^*(t) \| \leq \eta }{\sup}   \, \mathrm{pen}_t(   F(t)  ) \,\geq  \,   5    \mathrm{pen}_t(   F^*(t)  ) \,\,\, \right\}
 \]
 and
  \[
 \Omega_2 \,:=\, \left\{\underset{t \in \mathbb{R}}{\sup}\,\, \underset{F(t) \in  \Lambda(t) \,:\,     \| F(t) - F^*(t) \| \leq \eta,\,\,     \mathrm{pen}( F^*(t)- F(t)  ) \leq  4 \mathrm{pen}( F^*(t)  ) }{\sup}   \,a(t)^{\top} (   F(t) - F^*(t)    ) \,\geq  \,   \frac{\eta^2}{2}  \right\}
 \]
 satisfy that $\Omega_1  \subset \Omega_2$. Hence, $\mathbb{P}(\Omega_1) \leq \mathbb{P}(\Omega_2)$.

Next, we observe that if   $\|  \hat{F}(t) -  F(t) \|  \geq  \eta$, then there exists $F(t) \in \Lambda(t)$  such that $\| F(t)  - F^*(t ) \|=\eta$. This implies, by (\ref{eqn:e4}), that 
\[
   \frac{\eta^2}{2}\,\leq \,      a(t)^{\top} (F(t) -F^*(t) )  \,+\,  \lambda_t\left[  \mathrm{pen}(F^*(t) ) - \mathrm{pen}(F(t))  \right]   
\]
and so from our choice of $\lambda_t$
\begin{equation}
	\label{eqn:e5}
	 \frac{\eta^2}{4}\,\leq\,   a(t)^{\top} (F(t)-F^*(t) )  .
\end{equation}
Thus, (\ref{eqn:e5})  holds for some $F(t) \in \Lambda(t)$  with $\|F(t) -F^*(t) \|\leq \eta$ provided that  $\| \widehat{F}(t)-F^*(t)\| >\eta$. Therefore,
\begin{equation}
	\label{eqn:e6}
	\begin{array}{lll}
		\mathbb{P}\left(   \underset{t \in  \mathbb{R}}{\sup}\, \| \widehat{F}(t)-F^*(t)\| \,>\,\eta      \right)  & \leq &    \mathbb{P}\left(  \left\{ \underset{t \in  \mathbb{R}}{\sup}\, \| \widehat{F}(t)-F^*(t)\| \,>\,\eta   \right\} \cap  \Omega_1^c  \right) \,+\, \mathbb{P}(   \Omega_1 )\\
		& \leq&  \mathbb{P}\bigg(       \underset{t \in  \mathbb{R} }{\sup}\,\,\underset{F(t)  \in \Lambda(t) \,:\,  \|F^*(t) -F(t) \|\leq \eta,  \,\,\mathrm{pen}(  F(t))  \leq  5  \mathrm{pen}(  F^*(t))  }{\sup}   a(t)^{\top} (F(t)-F^*(t)  )   \\
		& &\displaystyle \,\,\,\,\,\,\,\,\,\,\,\,      \, \geq \,  \frac{\eta^2}{4}    \bigg)\,+\, \mathbb{P}(   \Omega_2 )\\
		& \leq&  2\,\mathbb{P}\bigg(       \underset{t \in  \mathbb{R} }{\sup}\,\,\underset{F(t)  \in \Lambda(t) \,:\,  \|F(t) -F^*(t) \|\leq \eta,  \,\,\mathrm{pen}(  F(t))  \leq  5  \mathrm{pen}(  F^*(t))  }{\sup}   a(t)^{\top} (F(t) -F^*(t) )   \\
		& &\displaystyle \,\,\,\,\,\,\,\,\,\,\,\,      \, \geq \,  \frac{\eta^2}{4}    \bigg)\\
		& \leq& \displaystyle  \frac{8}{\eta^2}\mathbb{E}\left(   \underset{t \in  \mathbb{R} }{\sup}\,\,\underset{F(t)  \,:\,  \|F^*(t) -F(t) \|\leq \eta,  \,\,\mathrm{pen}(  F(t))  \leq  5  \mathrm{pen}(  F^*(t))  }{\sup}   a(t)^{\top} (F(t)-F^*(t) ) \right)\\
		& \leq& \displaystyle  \frac{8}{\eta^2}\mathbb{E}\left(   \underset{t \in  \mathbb{R} }{\sup}\,\,\underset{\theta  \in K \,:\,  \|\theta\|\leq \eta }{\sup}   a(t)^{\top}\theta \right).
	\end{array}
\end{equation}
 Hence, for $\xi_1,\ldots,\xi_n$ independent Rademacher variables  independent of $y$, we have that
 \[
 \begin{array}{lll}
 	\mathbb{P}\left(   \underset{t \in  \mathbb{R}}{\sup}\, \| \widehat{F}(t)-F^*(t)\| \,>\,\eta      \right)  & \leq &  \displaystyle \frac{16}{\eta^2}  \mathbb{E}\bigg(\underset{t \in \mathbb{R}}{\sup}\, \, \underset{  \theta \in K\,:\,    \|\theta \|\leq \eta }{\sup} \,   \sum_{i=1}^{n}  \xi_i  1_{\{ y_i\leq t\}} \theta_i  \bigg)\\
 		& =& \displaystyle \frac{16}{\eta^2}\mathbb{E}\bigg( \mathbb{E}\bigg(\underset{t \in \mathbb{R}}{\sup} \,\,\underset{  \theta \in K\,:\,    \|\theta \|\leq \eta  }{\sup}\,   \sum_{i=1}^{n}  \xi_i  1_{\{ y_i\leq t\}} \theta_i    \bigg|   y\bigg)  \bigg)\\
 	& =&    	  \displaystyle \frac{16}{\eta^2}\mathbb{E}\bigg( \mathbb{E}\bigg(   \,\underset{t \in   \{  y_1,\ldots,y_n   \}  }{\max} \,\,   \underset{  \theta \in K\,:\,    \|\theta \|\leq \eta  }{\sup}\,  \sum_{i=1}^{n}  \xi_i  1_{\{ y_i\leq t\}} \theta_i   \bigg|  y\bigg)  \bigg)\\   
 \end{array}
 \]
 where the first inequality follows from (\ref{eqn:e6}) and symmetrization. However, by Lemma \ref{lem3},
 \[
 \begin{array}{lll}
 	\displaystyle  \mathbb{E}\bigg(   \,\underset{t \in   \{  y_1,\ldots,y_n   \}  }{\max} \,\,   \underset{  \theta \in K\,:\,    \|\theta \|\leq \eta  }{\sup}\,  \sum_{i=1}^{n}  \xi_i  1_{\{ y_i\leq t\}} \theta_i   \bigg|  y\bigg) &\leq & \displaystyle C \int_{0}^{\eta/4}    \sqrt{\log N(\varepsilon,K\cap   B_{\varepsilon}(0),\|\cdot\|   )} d\varepsilon  \\
 	& & \displaystyle \,+ \,  C  \eta \sqrt{\log n},
 \end{array}
 \]
 for some constant $C>0$. The claim then follows.

\end{proof}

\subsection{Proof of 
Corollary  \ref{thm5} }

\begin{proof}
Let $g \sim  N(0,  I_n)$ and set
\begin{equation}
	\label{eqn:tv_set2}
	K \,:=\,   \left\{  \theta   \in \mathbb{R}^n  \,:\,    	\mathrm{TV}^{(r)}(\theta) \leq V \right\}.
\end{equation}
Hence, 
	\begin{equation}
		\label{eqn:lc_gaussian_tv}
	\begin{array}{lll}
	\displaystyle 		\mathbb{E}\left[  \underset{ \theta \in K_t   \,:\,  \| \theta  - F^*(t)\|\leq  \eta }{\sup}   \, g^{\top} (\theta - F^*(t))     \right] &\leq&	\displaystyle \mathbb{E}\left[  \underset{ \theta \in K   \,:\,  \| \theta  - F^*(t)\|\leq  \eta }{\sup}   \, g^{\top} (\theta - F^*(t))    \right]  \\
	 &\leq & \displaystyle  C_r \left[  \eta \left( \frac{\sqrt{n}V }{\eta} \right)^{1/2r} \,+\,\eta \sqrt{\log n} \right]
	\end{array}
	\end{equation}
where $C_r>0$  is a constant the second inequality follows from Lemma B.1 in \citet{guntuboyina2020adaptive}. 

Next, notice that 
\[
C_r   \eta \left( \frac{\sqrt{n}V }{\eta} \right)^{1/2r}\,\leq\,\frac{\eta^2}{2L}
\]
holds if
\[
(2LC_r)^{2r/(2r+1)} n^{1/(4r+2)}  V^{1/(2r+1)} \,\leq\, \eta.
\]
Also, 
\[
C_r\eta \sqrt{\log n} \,\leq \, \frac{\eta^2}{2L }
\]
provided that   $2LC_r\sqrt{\log n} \leq \eta $. Hence, taking
\[
\eta =  \max\{ (2LC_r)^{2r/(2r+1)} n^{1/(4r+2)}  V^{1/(2r+1)} ,   2LC_r\sqrt{\log n} \}
\]
we obtain (\ref{eqn:e22}).

The claim for $\widetilde{F}(t)$ follows from Corollary \ref{cor1}.

To show (\ref{eqn:e23}), we observe that by the proof of Theorem B.1 in \citet{guntuboyina2020adaptive},  for $0<\eta <n$, we have 
	\[
      \arraycolsep=1.4pt\def\arraystretch{1.6}
     \begin{array}{lll}
\displaystyle      	\int_{0}^{\eta/4}    \sqrt{\log N(\varepsilon,(K -K )\cap   B_{\varepsilon}(0),\|\cdot\|   )} d\varepsilon   & \leq & \displaystyle 2	\int_{0}^{\eta/4}    \sqrt{\log N(\varepsilon/2,K\cap   B_{\varepsilon}(0),\|\cdot\|   )} d\varepsilon   \\
 & \leq& \tilde{C}_{r} \left[  \eta \left( \frac{\sqrt{n}V }{\eta} \right)^{1/2r} \,+\,\eta \sqrt{\log n} \right]\\
      \end{array}
	\]
    for some positive constant $\tilde{C}_r$. Hence, (\ref{eqn:e23}) follows with the same argument as above. 

\end{proof}
\subsection{Proof of Theorem \ref{thm9}}
\label{sec:proof_thm9}

\begin{proof}
Notice that for any 
$t \in \mathbb{R}$ we have that 
\[
 \|F(t)- H(t)\|^2 \,\leq  \,  2n B
\]
for all $F(t), H(t) \in K_t$.



Then 
\[
\begin{array}{lll}
\mathbb{P}\left(      \underset{t\in \mathbb{R}}{\sup}\,   \|  \widehat{F}(t) - G(t) \|   \,>\,\eta   \right)  & =& \displaystyle  \mathbb{P}\left(      \underset{t\in \mathbb{R}}{\sup}\,   \|  \widehat{F}(t) - G(t) \|^2   \,>\,\eta ^2 ,\,\,\underset{t\in \mathbb{R}}{\sup}\,   \|  \widehat{F}(t) - G(t) \|^2 \leq 2n B\right)  \\
 & \leq& \displaystyle \sum_{j=1}^J   \mathbb{P}\left(      \underset{t\in \mathbb{R}}{\sup}\,   \|  \widehat{F}(t) - G(t) \|^2   \,>\,2^{j-1}\eta ^2 ,\,\,\underset{t\in \mathbb{R}}{\sup}\,   \|  \widehat{F}(t) - G(t) \|^2 \leq 2^j \eta^2 \right)  \\
  & \leq&\displaystyle \sum_{j=1}^J   \mathbb{P}\bigg(      \underset{t\in \mathbb{R}}{\sup}\,  2\sum_{i=1}^{n}(\widetilde{F}_i(t)-   1_{  \{ y_i\leq t \} }  )(G_i(t)-\widehat{F}_i(t) )    \,>\,2^{j-1}\eta ^2 ,\,\,\\
   & &\,\,\,\,\,\,\,\,\,\,\,\,\,\,\,\, \underset{t\in \mathbb{R}}{\sup}\,   \|  \widehat{F}(t) - G(t) \|^2 \leq 2^j \eta^2 \bigg)  \\
    & \leq&\displaystyle \sum_{j=1}^J    \mathbb{P}\bigg(  \underset{t\in \mathbb{R}}{\sup}  \,\,\underset{F(t) \in K_t \,:\,  \|  F(t)- G(t)\|^2\leq  2^{j }\eta^2   }{\sup} \, \sum_{i=1}^{n} (G_i(t)-   1_{  \{ y_i\leq t \} }  )(G_i(t)-F_i(t) )    \\
     & &\,\,\,\,\,\,\,\,\,\,\,\,\,>\,2^{j-2}\eta ^2\bigg)\\
\end{array}
\]
where the first inequalit follows by union bound, the second by the basic inequality. However, 
\[
  \begin{array}{l}
\displaystyle  \mathbb{P}\bigg(  \underset{t\in \mathbb{R}}{\sup}  \,\,\underset{F(t) \in K_t \,:\,  \|  F(t)- G(t)\|^2\leq  2^{j }\eta^2   }{\sup} \, \sum_{i=1}^{n} (G_i(t)-   1_{  \{ y_i\leq t \} }  )(G_i(t)-F_i(t) )  \,>\,  2^{j-2} \eta^2\bigg)\\
  \leq  	\displaystyle     \frac{1}{2^{j-2} \eta^2}\mathbb{E}\left(  \underset{t\in \mathbb{R}}{\sup}  \,\,\underset{F(t) \in K_t \,:\,  \|  F(t)- G(t)\|^2\leq  2^{j }\eta^2   }{\sup} \, \sum_{i=1}^{n} (G_i(t)-   1_{  \{ y_i\leq t \} }  )(G_i(t)-F_i(t) )   \right)\\
		 \leq\displaystyle     \frac{1}{2^{j-2} \eta^2}\mathbb{E}\left(  \underset{t\in \mathbb{R}}{\sup}  \,\,\underset{F(t) \in K_t \,:\,  \|  F(t)- G(t)\|^2\leq  2^{j }\eta^2   }{\sup} \, \sum_{i=1}^{n} (F_i^*(t)-   1_{  \{ y_i\leq t \} }  )(G_i(t)-F_i(t) )   \right)\\
	\displaystyle     \,+\,  \frac{1}{2^{j-2} \eta^2} \underset{t\in \mathbb{R}}{\sup}  \,\,\underset{F(t) \in K_t \,:\,  \|  F(t)- G(t)\|^2\leq  2^{j }\eta^2   }{\sup} \, \sum_{i=1}^{n}  (  G_i(t)- F_i^*(t)  )(G_i(t)-F_i(t) )\\
	\leq\displaystyle     \frac{1}{2^{j-2} \eta^2}\mathbb{E}\left(  \underset{t\in \mathbb{R}}{\sup}  \,\,\underset{F(t) \in K_t \,:\,  \|  F(t)- G(t)\|^2\leq  2^{j }\eta^2   }{\sup} \, \sum_{i=1}^{n} (F_i^*(t)-   1_{  \{ y_i\leq t \} }  )(G_i(t)-F_i(t) )   \right)\\
	\displaystyle     \,+\,  \frac{4\sqrt{n}   }{2^{j/2} \eta} \,\, \underset{t\in \mathbb{R}}{\sup}\,  \| G(t)-F^*(t) \|_{\infty} 
\end{array}
\]
for some constant $C>0$, where the first inequality follows from Markov's inequality, and  the last inequality holds by Cauchy–Schwarz inequality.  Furthermore, 
\[
 \begin{array}{l}
 \displaystyle 	\mathbb{E}\left(  \underset{t\in \mathbb{R}}{\sup}  \,\,\underset{F(t) \in K_t \,:\,  \|  F(t)- G(t)\|^2\leq  2^{j }\eta^2   }{\sup} \, \sum_{i=1}^{n} (F_i^*(t)-   1_{  \{ y_i\leq t \} }  )(G_i(t)-F_i(t) )   \right)\\
  \displaystyle \leq 	\mathbb{E}\left(  \underset{t\in \mathbb{R}}{\sup}  \,\,\underset{F(t) \in K_t -K_t\,:\,  \|  F(t)\|^2\leq  2^{j }\eta^2   }{\sup} \, \sum_{i=1}^{n} (F_i^*(t)-   1_{  \{ y_i\leq t \} }  )F_i(t)    \right)\\
   \displaystyle \leq   C \int_{0}^{2^{j/2}\eta/4}    \sqrt{\log N(\varepsilon,(K-K)\cap   B_{\varepsilon}(0),\|\cdot\|   )} d\varepsilon  
 \,+ \,  C  2^{j/2}\eta \sqrt{\log n},
 \end{array}
\]
where the last inequality follows as in the proof of Theorem \ref{thm3}. Therefore,
\[
  \begin{array}{lll}
  	\mathbb{P}\left(      \underset{t\in \mathbb{R}}{\sup}\,   \|  \widehat{F}(t) - G(t) \|   \,>\,\eta   \right)    & \leq&\displaystyle \sum_{j=1}^{J}\frac{1}{2^{j-2} \eta^2}\bigg[ C \int_{0}^{2^{j/2}\eta/4}    \sqrt{\log N(\varepsilon,(K-K)\cap   B_{\varepsilon}(0),\|\cdot\|   )} d\varepsilon  
  	\\
  	 & &\displaystyle \,+ \,  C   2^{j/2}\eta \sqrt{\log n}\bigg]  \,+\, \sum_{j=1}^{J} \frac{4\sqrt{n}   }{2^{j/2} \eta} \,\, \underset{t\in \mathbb{R}}{\sup}\,  \| G(t)-F^*(t) \|_{\infty} \\
  	  & \leq&\displaystyle   \frac{C  }{\eta^2}\sum_{j=1}^{J}    \frac{1}{2^{j-2}} \int_{0}^{  2^{j/2} \eta/4 }\sqrt{\log N(\varepsilon,(K-K)\cap   B_{\varepsilon}(0),\|\cdot\|   )} d\varepsilon  \,+\,\\
  	   & &\displaystyle \frac{4C \sqrt{\log n} }{\eta} \sum_{j=1}^{J}   \left( \frac{1}{2^{1/2}}\right)^j\,+\, \frac{4 \sqrt{n}}{\eta}\underset{t\in \mathbb{R}}{\sup}\,  \| G(t)-F^*(t) \|_{\infty}   \sum_{j=1}^{J} \left( \frac{1}{2^{1/2}}\right)^j\\
  	    & \leq&   \displaystyle   \frac{C  }{\eta^2}\sum_{j=1}^{J}    \frac{1}{2^{j-2}} \int_{0}^{  2^{j/2} \eta/4 }\sqrt{\log N(\varepsilon,(K-K)\cap   B_{\varepsilon}(0),\|\cdot\|   )} d\varepsilon  \,+\,\\
  	       & &\displaystyle \frac{4C \sqrt{\log n} }{\eta}  \frac{2^{-1/2}}{1- 2^{-1/2}}\,+\, \frac{4 \sqrt{n}}{\eta}\underset{t\in \mathbb{R}}{\sup}\,  \| G(t)-F^*(t) \|_{\infty} \frac{2^{-1/2}}{1- 2^{-1/2}}\\
  \end{array}
\]
and the claim follows  noticing that
\[
  \begin{array}{lll}
   	\mathbb{P}\left(      \underset{t\in \mathbb{R}}{\sup}\,   \|  \widehat{F}(t) - F^*(t) \|   \,>\,\eta   +    \sqrt{n}\| F^*(t)- G(t)\|_{\infty}  \right)  & \leq&	\mathbb{P}\left(      \underset{t\in \mathbb{R}}{\sup}\,   \|  \widehat{F}(t) - G(t) \|   \,>\,\eta   \right) .
  \end{array}
\]
\end{proof}

\newpage
\subsection{Proof  of Corollary \ref{thm11} }
\label{Proofofthm11}

We begin by restating the corollary to be proved. 

\begin{corollary}
    \label{thm11-aux}
    Let $\widehat{F}(t)$ be the estimator from (\ref{eqn:estimator}) with the set  $K_t$ as in (\ref{eqn:set_k}) for all $t\in \mathbb{R}$ with $F^*(t)$ not necessarily in $K_t$. Suppose that Assumption \ref{as1}, described in Appendix~\ref{Appendix-DRN}, holds. Let
    \[
    \phi_{n} = \max_{(p, M) \in \mathcal P } n^{\frac{-2p}{ (2p+M)}}. 
\]
    There exists  positive constants $c_1$ and $c_2$ such that if 
\begin{equation}
    \label{eqn:choice1}
       L  =\lceil c_1 \log n \rceil\,\,\,\,\,\,\text{and}\,\,\,\,\,\, \nu = \left\lceil c_2  \max_{(p, M) \in \mathcal P  } n^{\frac{M}{ 2(2p+M)}}\right\rceil
\end{equation}
or 
\begin{equation}
    \label{eqn:choice2}
    L  =\left\lceil c_1  \max_{(p, M) \in \mathcal P } n^{\frac{M}{ 2(2p+M)}} \log n\right\rceil\,\,\,\,\,\,\text{and}\,\,\,\,\,\, \nu = \left\lceil c_2  \right\rceil,
\end{equation}
then
    \begin{equation}
       	\label{eqn:e30-aux}
       	 \underset{t \in \mathbb{R}}{\sup} \,    \sum_{i=1}^{n}   \frac{1}{n}\left(  \widehat{F}_i(t)    - F^*_i(t)  \right)^2  = O_{\mathbb{P}}\left(  \frac{\log n}{n} \,+\, \phi_n \log^4 n   \right).
       \end{equation}
 \end{corollary}

Then the proof of such result is provided. 

\begin{proof}
Throughout the proof, we condition on the covariates $x_i$'s and omit this dependence from the notation for simplicity. We proceed by using Theorem \ref{thm9}. First, for a vector $v \in  \mathbb{R}^n$, we let $\|v\|_n := \|v\|/\sqrt{n}$. Then, by  Theorem 3 in \citet{kohler2021rate}, it holds that 
\begin{equation}
    \label{eqn:e31}
     \underset{t \in \mathbb{R} }{\sup}\, \|F^*(t) - G(t) \|_{\infty}\,\leq\, C_1\sqrt{\phi_n}. 
\end{equation}
for some positive constant $C_1$.

Furthermore, as in the proof of Theorem 2 in \citet{zhang2024dense}, see also Lemma 19 in \citet{kohler2021rate}, we have that 
\[
   \log N( \varepsilon,  \mathcal{F}(L,\nu ),\| \cdot\|_n ) \lesssim L^2 \nu^2 \log( L \nu ) \log( \varepsilon^{-1}   )
\]
for $\varepsilon \in (0,1)$. Therefore, for $\epsilon \in (0,2\sqrt{n})$, we have that 
\[
   \log N( \varepsilon,  \mathcal{F}(L,\nu ),\| \cdot\|) \lesssim L^2 \nu^2 \log( L \nu ) \log( 2\varepsilon^{-1} \sqrt{n}  ).
\]
Therefore, for $\eta <\sqrt{n}$, with 
\[
J \,=\,  \left\lceil \frac{\log(2n/\eta^2)}{\log 2} \right\rceil,
\]
it holds that 
\begin{equation*}
    \begin{array}{l}
\displaystyle     \frac{C  }{\eta^2}\sum_{j=1}^{J}    \frac{1}{2^{j-2}} \int_{0}^{  2^{j/2} \eta/4 }\sqrt{\log N(\varepsilon,K({\varepsilon}),\|\cdot\|   )} d\varepsilon  \,+\,  \frac{C \sqrt{\log n} }{\eta} \,+\, \frac{C\sqrt{n}}{\eta}\underset{t\in \mathbb{R}}{\sup}\,  \| G(t)-F^*(t) \|_{\infty} \\
\leq  \displaystyle     \frac{C  }{\eta^2}\sum_{j=1}^{J}    \frac{1}{2^{j-2}} \int_{0}^{  2^{j/2} \eta/4 }\sqrt{\log N(\varepsilon,K -K,\|\cdot\|   )} d\varepsilon  \,+\,  \frac{C \sqrt{\log n} }{\eta} \,+\, \frac{C\sqrt{n}}{\eta}\underset{t\in \mathbb{R}}{\sup}\,  \| G(t)-F^*(t) \|_{\infty} \\
\leq  \displaystyle     \frac{C  }{\eta^2}\sum_{j=1}^{J}    \frac{1}{2^{j-2}} \int_{0}^{  2^{j/2} \eta/4 }\sqrt{\log N(\varepsilon,K -K,\|\cdot\|   )} d\varepsilon  \,+\,  \frac{C C_1 \sqrt{\log n} }{\eta} \,+\, \frac{C\sqrt{n \phi_n }}{\eta}.   \\
\lesssim  \displaystyle     \frac{1 }{\eta^2}\sum_{j=1}^{J}    \frac{1}{2^{j-2}} \int_{0}^{  2^{j/2} \eta/4 }\sqrt{\log N(\varepsilon/2,K,\|\cdot\|   )} d\varepsilon  \,+\,  \frac{\sqrt{\log n} }{\eta} \,+\, \frac{\sqrt{n \phi_n }}{\eta}   \\
\lesssim  \displaystyle     \frac{1 }{\eta^2}\sum_{j=1}^{J}    \frac{1}{2^{j-2}} \int_{0}^{  2^{j/2} \eta/4 }\sqrt{L^2 \nu^2 \log( L \nu ) \log( 4\varepsilon^{-1} \sqrt{n}  )} d\varepsilon  \,+\,  \frac{\sqrt{\log n} }{\eta} \,+\, \frac{\sqrt{n \phi_n }}{\eta} \\
\leq \displaystyle     \frac{L \nu \sqrt{\log(L \nu)} }{\eta^2}\sum_{j=1}^{J}    \frac{1}{2^{j-2}} \int_{0}^{  2^{j/2} \eta/4 }\sqrt{\log(n ) +  4\epsilon^{-1}  } d\varepsilon  \,+\,  \frac{\sqrt{\log n} }{\eta} \,+\, \frac{\sqrt{n \phi_n }}{\eta}  \\
    \end{array}
\end{equation*}
Moreover,
\begin{equation}
    \label{eqn:e32}
    \begin{array}{l}
 \displaystyle     \frac{L \nu \sqrt{\log(L \nu)} }{\eta^2}\sum_{j=1}^{J}    \frac{1}{2^{j-2}} \int_{0}^{  2^{j/2} \eta/4 }\sqrt{\log(n ) +  4\epsilon^{-1}  } d\varepsilon  \,+\,  \frac{\sqrt{\log n} }{\eta} \,+\, \frac{\sqrt{n \phi_n }}{\eta}   \\
\leq\displaystyle     \frac{L \nu \sqrt{\log(L \nu)} }{\eta^2}\sum_{j=1}^{J}    \frac{1}{2^{j-2}} \int_{0}^{  2^{j/2} \eta/4 } (\sqrt{\log(n )} +  2\epsilon^{-1/2})   d\varepsilon  \,+\,  \frac{\sqrt{\log n} }{\eta} \,+\, \frac{\sqrt{n \phi_n }}{\eta}   \\
\leq\displaystyle  \frac{L \nu \sqrt{\log(L \nu)} }{\eta^2}\sum_{j=1}^{J}    \frac{4\eta  \sqrt{\log n} }{2^{j/2}}     \,+\, \frac{L \nu \sqrt{\log(L \nu)} }{\eta^2}\sum_{j=1}^{J}    \frac{1}{2^{j-2}} \int_{0}^{  2^{j/2} \eta/4 }   2\epsilon^{-1/2}   d\varepsilon  \,+\,  \frac{\sqrt{\log n} }{\eta} \,+\, \frac{\sqrt{n \phi_n }}{\eta}   \\
\lesssim \displaystyle  \frac{L \nu \sqrt{\log(L \nu)  \log n } }{\eta}    \,+\, \frac{L \nu \sqrt{\log(L \nu)} }{\eta^2}\sum_{j=1}^{J}    \frac{1}{2^{j-2}} \int_{0}^{  2^{j/2} \eta/4 }   2\epsilon^{-1/2}   d\varepsilon  \,+\,  \frac{\sqrt{\log n} }{\eta} \,+\, \frac{\sqrt{n \phi_n }}{\eta}   \\
\lesssim \displaystyle  \frac{L \nu \sqrt{\log(L \nu)  \log n } }{\eta}    \,+\, \frac{L \nu \sqrt{\log(L \nu)} }{\eta^2}\sum_{j=1}^{J}    \frac{ \sqrt{\eta} }{2^{3j/4-2}}  \,+\,  \frac{\sqrt{\log n} }{\eta} \,+\, \frac{\sqrt{n \phi_n }}{\eta}.   \\
\lesssim \displaystyle  \frac{L \nu \sqrt{\log(L \nu)  \log n } }{\eta}   \,+\, \frac{L \nu \sqrt{\log(L \nu)} }{\eta^{3/2}  }  \,+\,  \frac{\sqrt{\log n} }{\eta} \,+\, \frac{\sqrt{n \phi_n }}{\eta}.   \\
    \end{array}
\end{equation}
Moreover, by our choice of $L$ and $\nu$,
\begin{equation}
    \label{eqn:e33}
    L\nu \,\asymp \,(\log n) \cdot \sqrt{n \phi_n}.
\end{equation}
Therefore, from (\ref{eqn:e32}) and (\ref{eqn:e33}), 
\begin{equation}
    \label{eqn:e34}
    \begin{array}{l}
\displaystyle     \frac{C  }{\eta^2}\sum_{j=1}^{J}    \frac{1}{2^{j-2}} \int_{0}^{  2^{j/2} \eta/4 }\sqrt{\log N(\varepsilon,K({\varepsilon}),\|\cdot\|   )} d\varepsilon  \,+\,  \frac{C \sqrt{\log n} }{\eta} \,+\, \frac{C\sqrt{n}}{\eta}\underset{t\in \mathbb{R}}{\sup}\,  \| G(t)-F^*(t) \|_{\infty} \\
\lesssim \displaystyle  \frac{\sqrt{n \phi_n} \log^2 n }{\eta}    \,+\, \frac{ \sqrt{. n \phi_n \log(n)} }{\eta^{3/2}  }  \,+\,  \frac{\sqrt{\log n} }{\eta} \,+\, \frac{\sqrt{n \phi_n }}{\eta}.   \\
    \end{array}
\end{equation}
Hence, the claim follows by taking 
\[
    \eta\,\asymp\,    \sqrt{\log n} \,+\,\sqrt{n  \phi_n} \log^2 n.
\]

\end{proof}
\newpage

\subsection{Proof of Corollary \ref{thm_isotonic-fr}}
\label{proofthm_isotonic-fr}
\begin{proof}

By Theorem \ref{thm2}, it is enough to bound 
$$ \underset{t \in \mathbb{R}}{\sup}\,    \mathbb{E}\left[  \underset{ \theta \in K_t   \,:\,  \| \theta  - F^*(t)\|\leq  \eta }{\sup}   \, g^{\top} (\theta - F^*(t))     \right],$$
where where  $g \sim  N(0,  I_n)$. Let, $t\in \mathbb{R}$ and \( g \sim \mathcal{N}(0, I_n) \) be a standard Gaussian vector in \( \mathbb{R}^n \), and let \( K_t \subset \mathbb{R}^n \) be a constraint set. Now we analyze the quantity
\[
\mathbb{E}\left[ \sup_{\theta \in K_t : \| \theta - F^*(t) \| \leq \eta} \langle g, \theta - F^*(t) \rangle \right].
\]
Since \( K_t \) is a convex cone it is closed under positive scalar multiplication; that is, \( \frac{K_t}{\eta} = K_t \) for any \( \eta > 0 \). Using this, we have
\[
\left( K_t - F^*(t) \right) \cap B(0, \eta) = \eta \cdot \left( \left( K_t - \frac{F^*(t)}{\eta} \right) \cap B(0, 1) \right).
\]
In consequence,
\begin{align}
\label{iso-fast-rate-eq2}
    \mathbb{E}\left[ \sup_{\theta \in K_t : \| \theta - F^*(t) \| \leq \eta} \langle g, \theta - F^*(t) \rangle \right]=\eta \cdot \mathbb{E}\left[ \sup_{\theta \in \left( K_t - \frac{F^*(t)}{\eta} \right) \cap B(0, 1) } \langle g, \theta\rangle \right].
\end{align}
Now, we consider the tangent cone to a convex set $K_t$ at a point $F^*(t) \in K_t$. This is defined as,
$$
T_{K_t}\left(\frac{F^*(t)}{\eta}\right):=\operatorname{cl}\left\{h\left(\theta-\frac{F^*(t)}{\eta}\right): \theta \in K_t, h>0\right\}.
$$
Now, observe that \( K_t := \left\{ \theta \in \mathbb{R}^n : \theta_1 \leq \cdots \leq \theta_n \right\} \) is a closed convex cone, and  the set \( K_t - \frac{F^*(t)}{\eta} \) is contained in the \emph{tangent cone} of \( K_t \) at \( \frac{F^*(t)}{\eta} \), that is,
\[
K_t - \frac{F^*(t)}{\eta} \subset T_{K_t}\left(\frac{F^*(t)}{\eta}\right).
\]
Therefore,
\begin{align}
\label{iso-fast-rate-eq3}
   \mathbb{E}\left[ \sup_{\theta \in \left( K_t - \frac{F^*(t)}{\eta} \right) \cap B(0, 1) } \langle g, \theta\rangle \right]\le \mathbb{E}\left[ \sup_{\theta \in  T_{K_t}\left( \frac{F^*(t)}{\eta} \right)\cap B(0, 1) } \langle g, \theta\rangle \right].
\end{align}
Furthermore,
$ \sup_{\theta \in  T_{K_t}\left( \frac{F^*(t)}{\eta} \right)\cap B(0, 1) } \langle g, \theta\rangle =  \sup_{\theta \in  T_{K_t}\left( \frac{F^*(t)}{\eta}\right)  \cap S^{n-1}} \langle g, \theta\rangle $,
  where \( S^{n-1} \) denotes the unit Euclidean sphere.
Then, by Equation (\ref{iso-fast-rate-eq3})
\begin{align}
\label{iso-fast-rate-eq5}
    \mathbb{E}\left[ \sup_{\theta \in  \left( K_t - \frac{F^*(t)}{\eta} \right) \cap B(0, 1) } \langle g, \theta\rangle \right]\le \mathbb{E}\left[ \sup_{\theta \in  T_{K_t}\left( \frac{F^*(t)}{\eta}\right)  \cap S^{n-1}} \langle g, \theta\rangle \right].
\end{align}
Now, we define the Gaussian width of subset $\mathcal{K}$ as
$$
w(\mathcal{K}):=\mathbb{E}\left[\sup _{v \in \mathcal{K} \cap S^{n-1}}\langle g, v\rangle\right].
$$
Therefore, from Equations (\ref{iso-fast-rate-eq2}) and (\ref{iso-fast-rate-eq5})
\begin{equation}
\label{Fast-rate-proof-iso-1}
\mathbb{E}\left[\sup _{\theta \in K_t:\left\|\theta-F^*(t)\right\| \leq \eta}\left\langle g, \theta-F^*(t)\right\rangle\right] \leq \eta \cdot w\left(T_{K_t}\left(\frac{F^*(t)}{\eta}\right)\right) .
\end{equation}
Let $\delta(\mathcal{K})$ denote the statistical dimension of a closed convex cone $\mathcal{K} \subset \mathbb{R}^n$, defined as
$$
\delta(\mathcal{K}):=\mathbb{E}[D(g ; \mathcal{K})], \quad \text { where } D(y ; \mathcal{K}):=\sum_{i=1}^n \frac{\partial}{\partial y_i} \hat{\theta}_i(y ; \mathcal{K})
$$
and $g \sim \mathcal{N}\left(0, I_n\right)$.
Here, $\hat{\theta}(y ; \mathcal{K})$ denotes the Euclidean projection of $y \in \mathbb{R}^n$ onto $\mathcal{K}$, formally defined as
$$
\hat{\theta}(y ; \mathcal{K}):=\arg \min _{\theta \in \mathcal{K}}\|\theta-y\|^2.
$$ Using Proposition 10.1 in \cite{amelunxen2013living}, the Gaussian width is controlled by the statistical dimension:
$$
w\left(T_{K_t}\left(\frac{F^*(t)}{\eta}\right)\right) \leq \sqrt{\delta\left(T_{K_t}\left(\frac{F^*(t)}{\eta}\right)\right)}.
$$
We now analyze the statistical dimension of this tangent cone. Let
\[
k(t) := \left| \left\{ i \in \{1, \dots, n - 1\} : F_i^*(t) < F_{i+1}^*(t) \right\} \right|
\]
be the number of strict increases in \( F^*(t) \). Let \( 0 = i_0 < i_1 < \cdots < i_{k(t)} < i_{k(t)+1} = n \) denote the jump points such that \( F^*(t) \) is constant over each segment \( \{ i_{\ell-1} + 1, \dots, i_\ell \} \). Then, the tangent cone admits the decomposition
\[
T_{K_t}\left(\frac{F^*(t)}{\eta}\right) = \bigoplus_{\ell = 1}^{k(t)+1} \mathcal{M}^{n_\ell}, \quad \text{where } n_\ell := i_\ell - i_{\ell-1},
\]
with each \( \mathcal{M}^{n_\ell} := \{ \theta \in \mathbb{R}^{n_\ell} : \theta_1 \leq \cdots \leq \theta_{n_\ell} \} \) denoting the isotonic cone in \( \mathbb{R}^{n_\ell} \). This decomposition follows from Remark~2.1 in \citet{chatterjee2015risk}. Therefore, the statistical dimension satisfies
\[
\delta\left(T_{K_t}\left(\frac{F^*(t)}{\eta}\right)\right) = \sum_{\ell = 1}^{k(t)+1} \delta(\mathcal{M}^{n_\ell}).
\]
From Example 2.2 in \cite{chatterjee2015risk}, we have that \( \delta(\mathcal{M}^{m}) \leq \log(em) \). Moreover, applying Jensen’s inequality, we conclude
\[
\delta\left(T_{K_t}\left(\frac{F^*(t)}{\eta}\right)\right) \leq (1 + k(t)) \cdot \log \left( \frac{e n}{1 + k(t)} \right).
\]
Hence, the Gaussian width is bounded as
$$
w\left(T_{K_t}\left(\frac{F^*(t)}{\eta}\right)\right) \leq \sqrt{(1+k(t)) \cdot \log \left(\frac{e n}{1+k(t)}\right)}
$$
Therefor, taking $\eta=\underset{t \in \mathbb{R}}{\sup}\sqrt{(1+k(t)) \cdot \log \left(\frac{e n}{1+k(t)}\right)},$ from Theorem \ref{thm2} and Equation \ref{Fast-rate-proof-iso-1} we conclude that
\[
\mathbb{E} \left( \frac{1}{n} \sum_{i=1}^n \mathrm{CRPS}(\widehat{F}_i, F_i^*) \right)
\leq C \underset{t \in \mathbb{R}}{\sup} \cdot \frac{1 + k(t)}{n} \log \left( \frac{e n}{1 + k(t)} \right).
\]

\end{proof}
\newpage

\subsection{Proof of Corollary \ref{Cor-fast-rat-trend}}
\label{Corfastrattrendproof}

\begin{proof}
By Theorem \ref{thm2}, it is enough to bound 
$$ \underset{t \in \mathbb{R}}{\sup}\,    \mathbb{E}\left[  \underset{ \theta \in K_t   \,:\,  \| \theta  - F^*(t)\|\leq  \eta }{\sup}   \, g^{\top} (\theta - F^*(t))     \right],$$
where where  $g \sim  N(0,  I_n)$. Let, $t\in \mathbb{R}$ and \( g \sim \mathcal{N}(0, I_n) \) be a standard Gaussian vector in \( \mathbb{R}^n \), and let \( K_t \subset \mathbb{R}^n \) be a constraint set. Now we analyze the quantity
\[
\mathbb{E}\left[ \sup_{\theta \in K_t : \| \theta - F^*(t) \| \leq \eta} \langle g, \theta - F^*(t) \rangle \right].
\]
To that end, remember that
$$
K_t=\left\{\theta \in R^n:\left\|D^{(r)} \theta\right\|_1 \leq \frac{V_t^*}{n^{r-1}}\right\} .
$$
and using Definition \ref{def1} we have that $\mathbb{E}\left[ \sup_{\theta \in K_t : \| \theta - F^*(t) \| \leq \eta} \langle g, \theta - F^*(t) \rangle \right]$ is the same as the Gaussian complexity of $\left\{\theta \in K_t-F^*(t): \| \theta\| \leq \eta\right\}$, this is
$$
\mathcal{R}\left(\left\{\theta \in K_t-F^*(t):  \| \theta\| \leq \eta\right\}\right),
$$
see Equation (\ref{Aux-gaus-comp}) for such notation.
Then, following the same line of arguments as in the proof of Theorem 2 in \cite{madrid2022risk} it follows that
$$
\begin{aligned}
\mathcal{R}\left(\left\{\theta \in K_t-F^*(t):  \| \theta\| \leq \eta\right\}\right) \leq & C \left[\max \left\{\frac{V^*}{n^{r-1}}, 1\right\} (s+1) \log \left(\frac{e n}{s+1}\right)  \right]
\end{aligned}
$$
and the desired result is obtained.

\end{proof}
\newpage

\subsection{Auxiliary Lemmas}

\begin{definition}
    \label{def_1}
Let $K \subset \mathbb{R}^n$ and $t>0$. A subset $P $ of $K$ is a packing of $K$ if  $P\subset K $ and  the set $\{ B_t(x)\}_{x \in P}$ is pairwise disjoint. Then, the $t$-packing number of $K$, denoted as $M(t,K,\|\cdot\|)$, is defined as the cardinality of any maximum $t$-packing.

\end{definition}

\begin{lemma}
	\label{lem2}
\textbf{[Variant of Dudley's inequality.]}	Let  $S \subset \mathbb{R}^n$  be a finite set and  $\epsilon^{(j)} \in \mathbb{R}^n$ be a vector of mean zero independent SubGaussian(1) random variables,  for  $j=1 \ldots,m$. Suppose that  $0 \in S$ and  $v\in S$  implies that $\|v\|\leq D_n/2$ for  some $D_n>0$. Then there exists a constant $C>0$  such that 
	\[
	      \mathbb{E}\left(    \underset{j=1,\ldots,m}{\max} \,\underset{v \in S}{\max}   \,v^{\top }\epsilon^{(j)}   \right) \,\leq \, C\left(   D_n\sqrt{\log m}  +        \int_{0}^{D_n/4}     \sqrt{  \log    M\left(\delta,S, \|\cdot\|\right)  }\,d\delta \right) 
	\]
	where $M(\delta, S, \| \cdot\|)$ is the packing number of $S$ with respect to the metric $\|\cdot\|$.
\end{lemma}

\begin{proof}
	For  $C\geq  1$, let  $S_n$ be a maximal  $D_n 2^{-l}$-separated subset of $S$, i.e.,
	\[
	      \underset{v,u \in S}{\min}\,  \|v-u\| \,>\,   D_n 2^{-l}.  
	 \]
	 By  construction, $\vert  S_l \vert  =   M(D_n 2^{-l},  S,\|\cdot\|)$. Clearly,  because of the maximality, 
	 \[
	     \underset{v\in S}{\max }\,  \underset{u \in  S_l}{\min} \,\|v- u\| \leq  D_n 2^{-l}.
	 \]
	 Furthermore,  $S_l = S$ for large enough $l$. Hence, we let 
	 \[
	  N\,=\,\min\left\{     l \geq  1\,:\,    S_l = S  \right\}.  
	 \]
	 Also, for $l\geq 1$, let   $\pi_l$ be the function that assigns $v \in S$ to the point in $S_l$ closest to $v$. By definition,
	 \[
	     \| \pi_l(v) - v\|  \,\leq \, D 2^{-l} 
	 \]
	 for  all $v \in S$ and $l \in \mathbb{N}$. We also write  $S_0 = \{0\}$ and so $\pi_0(v) =  0$ for all $v \in S$. Next, we observe that
	 \[
	    v^{\top} \epsilon^{(j)}  \,=\, \sum_{l=1}^{N}      ( \pi_l(v) - \pi_{l-1}(v)   )^{\top}   \epsilon^{(j)} 
	 \]
	 for all $v \in S$. It follows that
	 \[
	 \begin{array}{lll}
	 	\displaystyle \underset{j = 1,\ldots,m}{\max}\, \underset{v \in S}{\max}  \,\, v^{\top} \epsilon^{(j)}  & \leq & \displaystyle 	\underset{j = 1,\ldots,m}{\max}\, \underset{v \in S}{\max}  \,\, \sum_{l=1}^{N}      ( \pi_l(v) - \pi_{l-1}(v)   )^{\top}   \epsilon^{(j)} \\
	 	 & \leq& \displaystyle \sum_{l=1}^N  	\underset{j = 1,\ldots,m}{\max}\, \underset{v \in S}{\max}  \,\,  ( \pi_l(v) - \pi_{l-1}(v)   )^{\top}   \epsilon^{(j)} \\
	 \end{array}
	 \]
	 and so 
	 \[
	 	 \begin{array}{lll}
	 	\displaystyle   \mathbb{E}\left(\underset{j = 1,\ldots,m}{\max}\, \underset{v \in S}{\max}  \,\, v^{\top} \epsilon^{(j)}   \right)& \leq & \displaystyle \sum_{l=1}^N  	\mathbb{E}\left(\underset{j = 1,\ldots,m}{\max}\, \underset{v \in S}{\max}  \,\,  ( \pi_l(v) - \pi_{l-1}(v)   )^{\top}   \epsilon^{(j)}\right). \\ 
	 \end{array}
	 \]
	 However, notice that for all $u >0$,
	 \[
	    \mathbb{P}\left(     (  \pi_l(v) - \pi_{l-1}(v)  )^{\top} \epsilon^{(j)} \geq u\right)  \,\leq \,    2\exp\left(     \frac{-u^2}{2 \| \pi_l(v) -    \pi_{l-1}(v) \|^2  } \right)
	 \]
	 and 
	\[
	     \begin{array}{lll}
	      \| \pi_l(v) -\pi_{l-1}(v)\| & \leq&      \| \pi_l(v) -v\| +        \| \pi_{l-1}(v)-v\|\\ 
	      &\leq  &    D_n  2^{-l}   +  D_n  2^{l-1}\\
	       & \leq& 3D_n 2^{-l} 
 	     \end{array}
	 \]
	 which implies, by the subGaussian maximal inequality,  that 
	 \[
	 \begin{array}{lll}
	 \mathbb{E}\left(   \underset{j = 1,\ldots,m}{\max}\, \underset{v \in S}{\max}  \,\,  ( \pi_l(v) - \pi_{l-1}(v)   )^{\top}   \epsilon^{(j)}       \right) & \leq&\displaystyle  \frac{3C  D_n}{2^{l}}\sqrt{\log(  2m \vert S_l\vert  \vert S_{l-1}\vert   )     } \\
	 & \leq&\displaystyle  \frac{3C  D_n}{2^{l}}\sqrt{\log(  2m \vert S_l\vert^2)     } \\
	  & \leq&\displaystyle  \frac{3C  D_n}{2^{l}}\sqrt{\log(  2m  M(D 2^{-l},S,\|\cdot\|)  )     } \\
	 \end{array}
	 \]
	 for some constant $C>0$.  Therefore,
	 \[
	  \begin{array}{lll}
	  	 \mathbb{E}\left(   \underset{j = 1,\ldots,m}{\max}\, \underset{v \in S}{\max}  \,\,  v^{\top}  \epsilon^{(j)}       \right)  & \leq&    \displaystyle  3 C  \sum_{l=1}^{N}    \frac{D_n}{2^{l}} \sqrt{  \log m   \,+\,   \log (2 M(D_n 2^{l}, S, \|\cdot\| )  )   }\\
	  	  &\leq &     \displaystyle 3  C  \sqrt{\log m} \sum_{l=1}^{N} \frac{D_n}{2^{l}} \,+\,       3 C  \sum_{l=1}^{N}   \frac{D_n}{2^l} \sqrt{\log \left(   2M(D2^{-l},S,\|\cdot\|)   \right)     }\\
	  	   &\leq &     \displaystyle 3  C  D_n \sqrt{\log m}  \,+\,       3 C  \sum_{l=1}^{N}   \frac{D_n}{2^l} \sqrt{\log \left(   2M(D_n2^{-l},S,\|\cdot\|)   \right)     }\\	  	   
	  	   	  	   &\leq &     \displaystyle 3  C  D_n \sqrt{\log m}  \,+\,       6 C  \sum_{l=1}^{N}   \int_{ D_n /2^{l+1} }^{  D/2^l }   \sqrt{\log \left(   2M(r,S,\|\cdot\|)   \right)        }  dr \\
	  	   	  	   &= &     \displaystyle 3  C  D_n \sqrt{\log m}  \,+\,       6 C     \int_{ D_n /2^{N+1} }^{  D/2 }   \sqrt{\log \left(   2M(r,S,\|\cdot\|)   \right)        }  dr \\
	  \end{array}
	 \]
	  \[
	      \begin{array}{lll}
	      		  	   	  	    & \leq&    \displaystyle 3  C  D_n \sqrt{\log m}  \,+\,       6 C     \int_{ 0 }^{  D_n/2 }   \sqrt{\log \left(   2M(r,S,\|\cdot\|)   \right)        }  dr \\
	      	& \leq&    \displaystyle 3  C  D_n \sqrt{\log m}  \,+\,       6 C     \int_{ 0 }^{  D_n/4 }   \sqrt{\log \left(   2M(r,S,\|\cdot\|)   \right)        }  dr  \\
	      	& & \displaystyle \,+\,  6 C     \int_{ 0 }^{  D_n/4 }   \sqrt{\log \left(   2M(r  +  D_n/4  ,S,\|\cdot\|)   \right)        }  dr\\
	      	& \leq&    \displaystyle 3  C  D_n \sqrt{\log m}  \,+\,       12C     \int_{ 0 }^{  D_n/4 }   \sqrt{\log \left(   2M(r,S,\|\cdot\|)   \right)        }  dr  \\
	      	& \leq&    \displaystyle 3  C  D_n \sqrt{\log m}  \,+\,       24C     \int_{ 0 }^{  D_n/4 }   \sqrt{\log \left(   M(r,S,\|\cdot\|)   \right)        }  dr  \\
	      \end{array}
	  \]
	 and the claim follows.
	 
\end{proof}

\begin{lemma}
	\label{lem3}
	With the notation and conditions of  Lemma \ref{lem2}, if  $S\subset \mathbb{R}$ arbitrary (not necessarily finite), then 
	\[
	\mathbb{E}\left(    \underset{j=1,\ldots,m}{\max} \,\underset{v \in S}{\sup}   \,v^{\top }\epsilon^{(j)}   \right) \,\leq \, C\left(   D_n\sqrt{\log m}  +        \int_{0}^{D_n/4}     \sqrt{  \log    M\left(\delta,S, \|\cdot\|\right)  } d\delta \right).
	\]
\end{lemma}

\begin{proof}
	Let  $\tilde{S} \subset S$ be a countable set such that 
	\[
	 \underset{j =1,\ldots,m}{\max}\, \underset{v \in  S}{\sup}\, \,v^{\top }\epsilon^{(j)} \,=\,	 \underset{j =1,\ldots,m}{\max}\, \underset{v \in  \tilde{S}}{\sup}\, \,v^{\top }\epsilon^{(j)}.
	\]
	Let  $\tilde{S}_l$  the set consisting of the first $l$ elements of $\tilde{S}$. Without loss of generality let's assume that $0 \in \tilde{S}_l$ for all $l$. Then by Lemma \ref{lem2}, we have that 
	\begin{equation}
	\label{eqn:e1}
    \begin{array}{lll}
    			\mathbb{E}\left(    \underset{j=1,\ldots,m}{\max} \,\underset{v \in \tilde{S}_l }{\sup}   \,v^{\top }\epsilon^{(j)}   \right) &\leq & C\left(   D_n\sqrt{\log m}  +        \int_{0}^{D_n/4}     \sqrt{  \log    M\left(\delta,\tilde{S}_l, \|\cdot\|\right)  }   d\delta\right)\\
    	& \leq& \displaystyle  C\left(   D_n\sqrt{\log m}  +        \int_{0}^{D_n/4}     \sqrt{  \log    M\left(\delta,S, \|\cdot\|\right)  }d\delta \right)\\
    \end{array}	 
	\end{equation}
	for all $l \geq 1$. Hence, the claim follows by  letting $l \rightarrow \infty$ in  (\ref{eqn:e1}) and applying the Monotone Convergence Theorem.
\end{proof}

\begin{definition}
	\label{def1}
For $t \in \mathbb{R}$ let  $\epsilon(t) = w(t) - F^*(t) \in \mathbb{R}^n$. Then for a set $\mathcal{V} \subset \mathbb{R}^n$  define 
\[
     R_t( \mathcal{V})\,=\,   \mathbb{E}\left(  \underset{v\in \mathcal{V}}{\sup} \,    \sum_{i=1}^n v_i \epsilon_i (t)    \right).
\]
Furthermore,  define the Gaussian complexity of $\mathcal{V}$ as
\begin{equation}
\label{Aux-gaus-comp}
\mathcal{R}( \mathcal{V})\,=\,   \mathbb{E}\left(  \underset{v\in \mathcal{V}}{\sup} \,    \sum_{i=1}^n v_i g_i     \right),
\end{equation}
where $g \sim N(0,I)$. 

\end{definition}

\begin{lemma}
	\label{lem4}
	There exists a universal constant such that 
	for any set $\mathcal{V} \subset \mathbb{R}^n$ it holds that
	\[
	\underset{t \in \mathbb{R} }{\sup}\,     R_t( \mathcal{V})\,\leq\, L\mathcal{R}( \mathcal{V})  ,
	\]
	where $L>0 $ is a universal constant. 
\end{lemma}

\begin{proof}
Fix $t \in \mathbb{R}$. Then for  $v \in \mathcal{V}$, let $Y_v =   v^{\top }  \epsilon(t)$ and  $X_v =  v^{\top}  g $. Next, observe that 
\[
 \underset{u,v \in  \mathcal{V}}{\sup}\,  \vert Y_u -Y_v \vert \,=\,   \underset{u,v \in  \mathcal{V}}{\sup}\, (Y_u -Y_v )\,=\,   \underset{u \in \mathcal{V} }{\sup}  Y_u  \,+\,   \underset{v \in \mathcal{V} }{\sup}  -Y_v. 
\]
Hence, for any $v_0 \in \mathcal{V}$,
\begin{equation}
	\label{eqn:e10}
	\begin{array}{lll}
		\displaystyle   \mathbb{E}\left(\underset{u,v \in  \mathcal{V}}{\sup}\,  \vert Y_u -Y_v \vert \right)&=&  \displaystyle \mathbb{E}\left(  \underset{u \in \mathcal{V} }{\sup}  Y_u \right)\,+\, \mathbb{E}\left( \underset{v \in \mathcal{V} }{\sup}  -Y_v\right)\\
		&\geq & \displaystyle \mathbb{E}\left(  \underset{u \in \mathcal{V} }{\sup}  Y_u \right)\,+\, \mathbb{E}\left(  -Y_{v_0}\right)\\
		&\geq & \displaystyle \mathbb{E}\left(  \underset{u \in \mathcal{V} }{\sup}  Y_u \right)\\
		&  = & R_{t}(\mathcal{V}).
	\end{array}
\end{equation}
Now, we observe that $\epsilon_i(t) $ is sub-Gaussian(1). Hence, by Theorem 2.1.5 from \citet{talagrand2005generic}, we have that  
\[
  \mathbb{E}\left(\underset{u,v \in  \mathcal{V}}{\sup}\,  \vert Y_u -Y_v \vert \right)\,\leq\, L \mathbb{E}(   \underset{v\in \mathcal{V}}{\sup}   \, v^{\top} g  )\,=\, L \mathcal{R}( \mathcal{V}),
\]
for a universal constant that does not depend on $t$. Hence, the claim follows.

\end{proof}

\begin{definition}
\label{def6}
    For a measurable function $f  \,:\,\mathbb{R} \rightarrow \mathbb{R}$ we define its distribution with respect to the Lebesgue measure as the function $\mu_f  \,:\,\,[0,\infty) \,\rightarrow \mathbb{R}$ given as 
    \[
    \mu_f(\lambda)\,:=\, \mu( \{ x \,:\, \vert f(x) \vert >\lambda  \} )
    \]
    where $\mu$ is the Lebesgue measure. 
\end{definition}

\begin{definition}
\label{def7}
    For a measurable function $f  \,:\,\mathbb{R} \rightarrow \mathbb{R}$ we define its  decreasing rearrangement  $D(f)  \,:\,\,[0,\infty) \,\rightarrow \mathbb{R}$ given as 
    \[
    D(f)(t)\,:=\, \inf\{ \lambda \geq 0 \,:\, \mu_f(\lambda) \leq t \}
    \]
    where $\mu$ is the Lebesgue measure. 
\end{definition}

\begin{lemma}
    \label{lem20}
    Suppose that $f$ can be written as 
    \[
      f(t) \,=\, \sum_{ l=1}^{n-1} b_l 1_{ E_l }(t)
    \]
    for $b_1\geq \ldots \geq b_{n-1}\geq 0$ and measurable sets $E_l \subset \mathbb{R}$  that are pairwise disjoint.  Then the  decreasing rearrangement of $f$ is given by 
    \[
      D(f)(t) \,=\, \sum_{l=1}^{n-1}  b_l 1_{ [m_{l-1}, m_{l})  }(t)
    \]
    where 
    \[
    m_l \,:=\, \sum_{k=1}^l \mu(E_k), \,\,\,l=1,\ldots,n-1,
    \]
    and with $m_0 =0$.
\end{lemma}

\begin{proof}
    This is Example  1.6 in \citet{bennett1988interpolation}.
\end{proof}

\begin{lemma}
    \label{lem21}
    For any integrable functions $f $ and $g$  the following hold:
    \begin{enumerate}
        \item 
       \begin{equation}
           \label{eqn:e40}
            \int_{ \mathbb{R} } \vert f(t)\vert^2 dt\,=\, \int_{0}^{ \infty  } \vert D(f)(t)\vert^2 dt.
       \end{equation}
        \item \textbf{[G.H. Hardy and J.E. Littlewood].}
               \begin{equation}
           \label{eqn:e41}
              \int_{ \mathbb{R} }  \vert f(t) g(t)\vert dt \,\leq\,  \int_{0}^{\infty} D(f)(s) \cdot D(g)(s) ds. 
       \end{equation}
       \item 
       Suppose that $f$  is decreasing and continuous in $[0,a)$ for some $a>0$,  and $f(t) =0$ otherwise.
Then $f(t) = D(f)(t)$ for all $t \in [0,\infty)$.

    \end{enumerate}
\end{lemma}

\begin{proof}
    The claim in (\ref{eqn:e40}) follows from Proposition 1.8 in \citet{bennett1988interpolation}. The inequality in  (\ref{eqn:e41})  is the well-known G.H Hardy and J.E Littlewood inequality, see for instance Theorem 2.2 in \citet{bennett1988interpolation}. 

We now prove the final claim. 
Let $t\geq0$.  Suppose that $t \in [0,a)$. Then 
  \begin{equation}
      \label{eqn:e53}
          \begin{array}{lll}
          D(f)(t) &=&  \inf\{  \lambda \geq 0\,:\, \sup\{ x\,:\,  f(x) > \lambda  \} \leq t \}\\
          &=& \inf\{  \lambda \geq 0\,:\, \sup\{ x \geq 0\,:\,  f(x) > \lambda  \} \leq t \}.\\
    \end{array}
  \end{equation}
Hence, for $0\leq \lambda< f(t)$, the continuity of $f$ in $[0,a)$, implies that there exists $t^{\prime} \in (t,a)$ such that $\lambda< f(t^{\prime}) \leq f(t)$. Thus, 
$\sup\{ x \geq 0\,:\,  f(x) > \lambda  \}> t$. On the other hand, if $  f(0)\geq  \lambda\geq f(t)$, then, also by the continuity of $f$ in $[0,a)$, 
  \[
  \sup\{ x \geq 0\,:\,  f(x) > \lambda  \} \,=\, \inf\{ x\geq 0\,:\, f(x) =\lambda   \}.
  \]
  Hence, from (\ref{eqn:e53}), we obtain $f(t) =  D(f)(t)$ for $t \in [0,a)$.
  
  Suppose now that $t \in [a,\infty)$. Then 
  \[
  \mu_f(0) \,\leq \,  a \leq t.
  \]
  Hence, $D(f)(t) =0$ and the claim follows. 
\end{proof}

\end{document}